\documentclass{article}
\usepackage{iclr2026_conference,times}

\usepackage{amsmath,amssymb,amsthm}
\usepackage{aliascnt}
\usepackage{graphicx}
\usepackage{tikz}
\usetikzlibrary{arrows.meta}
\usepackage{booktabs}
\usepackage{enumitem}
\setlist{leftmargin=7mm}
\definecolor{cupViolet}{HTML}{7B61A8}
\definecolor{cupTerracotta}{HTML}{D66B55}
\definecolor{cupTeal}{HTML}{2A8D8C}
\definecolor{cupInk}{HTML}{263349}
\definecolor{cupMuted}{HTML}{7D8794}
\usepackage{tabularx}
\usepackage{colortbl}

\usepackage{xcolor}
\usepackage{hyperref}
\usepackage{url}
\usepackage{cleveref}
\usetikzlibrary{calc}
\usetikzlibrary{arrows.meta,calc,shapes.geometric}
\usepackage{wrapfig}
\usepackage{multirow}
\usepackage{xfrac}
\usetikzlibrary{arrows.meta}
\definecolor{srInk}{HTML}{17243A}
\definecolor{srBlue}{HTML}{176B87}
\definecolor{srGreen}{HTML}{168B62}
\definecolor{srGold}{HTML}{AA731E}
\definecolor{srPurple}{HTML}{7A5AA6}
\definecolor{srGray}{HTML}{788392}

\definecolor{ink}{HTML}{17243A}
\definecolor{acc}{HTML}{176B87}
\definecolor{good}{HTML}{168B62}
\definecolor{bad}{HTML}{CF3C4F}
\definecolor{mut}{HTML}{8A93A3}
\usetikzlibrary{arrows.meta,calc}

\usepackage{amsmath,amsfonts,bm}

\def\eqref#1{equation~\ref{#1}}
\def\Eqref#1{Equation~\ref{#1}}

\def\1{\bm{1}}

\def\vtheta{{\bm{\theta}}}
\def\va{{\bm{a}}}
\def\vb{{\bm{b}}}
\def\vc{{\bm{c}}}

\def\ve{{\bm{e}}}

\def\vh{{\bm{h}}}

\def\vk{{\bm{k}}}

\def\vo{{\bm{o}}}
\def\vp{{\bm{p}}}
\def\vq{{\bm{q}}}
\def\vr{{\bm{r}}}
\def\vs{{\bm{s}}}

\def\vu{{\bm{u}}}
\def\vv{{\bm{v}}}
\def\vw{{\bm{w}}}
\def\vx{{\bm{x}}}
\def\vy{{\bm{y}}}
\def\vz{{\bm{z}}}

\def\mA{{\bm{A}}}
\def\mB{{\bm{B}}}
\def\mC{{\bm{C}}}
\def\mD{{\bm{D}}}
\def\mE{{\bm{E}}}

\def\mG{{\bm{G}}}
\def\mH{{\bm{H}}}
\def\mI{{\bm{I}}}
\def\mJ{{\bm{J}}}
\def\mK{{\bm{K}}}

\def\mM{{\bm{M}}}

\def\mP{{\bm{P}}}
\def\mQ{{\bm{Q}}}
\def\mR{{\bm{R}}}
\def\mS{{\bm{S}}}
\def\mT{{\bm{T}}}
\def\mU{{\bm{U}}}
\def\mV{{\bm{V}}}
\def\mW{{\bm{W}}}
\def\mX{{\bm{X}}}
\def\mY{{\bm{Y}}}

\def\mLambda{{\bm{\Lambda}}}
\def\mSigma{{\bm{\Sigma}}}

\DeclareMathAlphabet{\mathsfit}{\encodingdefault}{\sfdefault}{m}{sl}
\SetMathAlphabet{\mathsfit}{bold}{\encodingdefault}{\sfdefault}{bx}{n}

\newcommand{\R}{\mathbb{R}}

\newcommand{\C}{\mathbb{C}}
\newcommand{\norm}[1]{\left\lVert#1\right\rVert}
\newcommand{\abs}[1]{\left\lvert#1\right\rvert}
\newcommand{\spec}{\sigma}
\def\valpha{{\bm{\alpha}}}
\def\vsigma{{\bm{\sigma}}}
\def\mLambda{{\bm{\Lambda}}}

\DeclareMathOperator{\rank}{rank}
\definecolor{gr}{HTML}{8A8A8A}

\theoremstyle{plain}
\newtheorem{theorem}{Theorem}
\newaliascnt{proposition}{theorem}
\newtheorem{proposition}[proposition]{Proposition}
\aliascntresetthe{proposition}
\newaliascnt{lemma}{theorem}
\newtheorem{lemma}[lemma]{Lemma}
\aliascntresetthe{lemma}
\newaliascnt{corollary}{theorem}
\newtheorem{corollary}[corollary]{Corollary}
\aliascntresetthe{corollary}
\theoremstyle{remark}
\newaliascnt{remark}{theorem}

\aliascntresetthe{remark}
\theoremstyle{definition}
\newtheorem{definition}{Definition}[section]
\newcommand{\SH}{\mathcal{SH}}

\iclrfinalcopy
\usepackage{fontawesome5}
\usepackage{adjustbox}

\newcommand{\Diag}{\operatorname{Diag}}

\title{Complex KDA: Understanding and Enhancing\\ the Expressivity of Kimi Delta Attention}
\author{Julien Siems$^{*\heartsuit}$,~
Riccardo Grazzi$^{*\diamondsuit}$,~Korbinian Pöppel$^{*\text{\faLightbulb[regular]}\bigcirc}$,~
Jaisidh Singh$^{\text{\faUniversity}\spadesuit\square}$,
\\
\textbf{Arber Zela}$^{\triangle}$,~
\textbf{Timur Carstensen}$^{\text{\faLightbulb[regular]}\heartsuit\bigcirc}$,~
\textbf{Jenia Jitsev}$^{\dagger\clubsuit\Psi\bigcirc}$,~
\textbf{Frank Hutter}$^{\heartsuit\text{\faLightbulb[regular]}\text{\faTable}}$,\\
\textbf{Volkan Cevher}$^{\triangle}$,~
\textbf{Antonio Orvieto}$^{\text{\faLightbulb[regular]}\square}$,~
\textbf{Aaron Klein}$^{\text{\faLightbulb[regular]}\bigcirc}$\\
\footnotesize $^*$Equal contribution,~
$^{\heartsuit}$University of Freiburg,~
$^{\diamondsuit}$Microsoft Research,~
$^{\text{\faUniversity}}$University of Tübingen,\\
\footnotesize $^{\spadesuit}$Zuse School ELIZA,~
$^{\square}$MPI-IS Tübingen,~
$^{\text{\faLightbulb[regular]}}$ELLIS Institute Tübingen,~
$^{\triangle}$EPFL,~$^{\bigcirc}$OpenEuroLLM,\\
\footnotesize $^{\dagger}$Jülich Supercomputing Center (JSC),~
$^{\clubsuit}$LAION,~
$^{\Psi}$Open-$\Psi$ (Open-Sci) Collective,~
$^{\text{\faTable}}$PriorLabs\\
{{\small \texttt{juliensiems@gmail.com}~~~~\small \texttt{riccardograzzi4@gmail.com}~~~~\small\texttt{korbi@korbi.ai}} \qquad}
}
\usepackage[disable]{todonotes}
\begin{document}
\maketitle

 \begin{abstract}
Linear RNNs based on the delta-rule enable efficient sequence modeling, but their linear updates with a low-rank correction constrain their expressivity. 
Prior work has shown that composing two delta-rule transitions in a single recurrent update can model a 2D rotation, but this increases the rank and the cost of the updates compared to a single transition.
We show that Kimi Delta Attention (KDA) can realize 2D rotations by combining a single delta-rule transformation with a second reflection supplied by its channel-wise gate. This requires extending the parameter ranges of KDA by 
combining two existing range extensions: 
allowing gates in $[-1,1]$ and the delta-rule coefficient $\beta$ in $[0,2]$. We call the resulting model Complex KDA (CKDA). It preserves KDA’s stability and efficiency, with transitions that remain diagonal-plus-rank-one and non-expansive, while reaching the state-tracking expressivity of DeltaProduct$_2$.
We characterize the expressivity of CKDA and prove that every orthogonal diagonal-plus-rank-one matrix is exactly a CKDA transition matrix. 
A single CKDA layer can track every finite group isomorphic to a subgroup of $\mathrm{SO}(3)$, and many state-tracking results use one fewer layer for CKDA compared to other diagonal-plus-rank-one Linear RNNs. Empirically, combining both extensions yields the strongest length extrapolation among tested KDA range settings on $S_3$, $S_4$, and periodic audio continuation. 
In language modeling, CKDA outperforms Transformers and other linear RNNs, obtains similar results to a KDA baseline, and shows promising scaling behavior. 
Our code is \href{https://github.com/OpenEuroLLM/ComplexKDA}{open-source} as are our \href{https://huggingface.co/collections/openeurollm/complexkda}{models}.
\end{abstract}

\begin{figure}[!t]
    \centering
\resizebox{0.95\linewidth}{!}{%
\begin{tikzpicture}[
 x=1cm,y=1cm,font=\fontsize{8}{9}\selectfont,
 text=srInk,>=Latex,line cap=round,line join=round,
 move/.style={-{Latex[length=1.3mm]},line width=.8pt},
 connector/.style={-{Latex[length=1.1mm]},line width=.65pt},
 labelstate/.style={inner sep=3pt},
 equation/.style={anchor=base}
]
\path[use as bounding box] (0,.87) rectangle (14,3.50);

\fill[srGreen!7,rounded corners=3pt]
 (10.50,.90) rectangle (13.30,3.45);

\node[equation,text=srBlue] at (1.35,1.10)
 {$\vp=(1,0)^\top$};
\node[equation,text=srGold] at (4.65,1.10)
 {$\mD=\Diag(-1,1)$};
\node[equation,text=srPurple] at (8.00,1.10)
 {$\mH=\mI-2\vk\vk^\top$};
\node[equation,text=srGreen] at (11.90,1.10)
 {$\mR_{\pi/2}=\mH\mD$};

\foreach \cx in {1.35,4.65,8.00,11.90} {
 \begin{scope}[shift={(\cx,2.05)}]
  \draw[srGray!42,line width=.3pt] (-.94,0)--(.94,0);
  \draw[srGray!42,line width=.3pt] (0,-.55)--(0,.94);
 \end{scope}
}

\begin{scope}[shift={(1.35,2.05)}]
 \fill[srBlue] (.68,0) circle (2pt);
 \node[above,labelstate,text=srBlue] at (.68,.06) {$\vp$};
\end{scope}

\begin{scope}[shift={(4.65,2.05)}]
 \draw[srGold,dashed,line width=.75pt] (0,-.55)--(0,.94);
 \node[text=srGold,anchor=west,inner sep=2pt]
  at (.06,.82) {$x=0$};

 \draw[move,srGold] (.59,0)--(-.59,0);
 \draw[srBlue,fill=white,line width=.8pt]
  (.68,0) circle (2pt);
 \fill[srGold] (-.68,0) circle (2pt);

 \node[below,labelstate,text=srBlue] at (.68,-.06) {$\vp$};
 \node[below,labelstate,text=srGold] at (-.68,-.06) {$\mD\vp$};
\end{scope}

\begin{scope}[shift={(8.00,2.05)}]
 \draw[srPurple,dashed,line width=.75pt]
  (-.77,.77)--(.55,-.55);
 \draw[move,srPurple] (-.616,.064)--(-.064,.616);

 \draw[srGold,fill=white,line width=.8pt]
  (-.68,0) circle (2pt);
 \fill[srPurple] (0,.68) circle (2pt);
 \node[below,labelstate,text=srGold]
  at (-.68,-.06) {$\mD\vp$};
 \node[right,labelstate,text=srPurple]
  at (.10,.68) {$\mH\mD\vp$};

 \draw[-{Latex[length=1.2mm]},srPurple,line width=.8pt]
  (0,0)--(.480833,.480833);
 \node[text=srPurple,anchor=north west,inner sep=2pt]
  at (.51,.43) {$\vk$};
 \draw[srPurple!65,line width=.4pt]
  (-.0725,.0725)--(0,.145)--(.0725,.0725);
\end{scope}

\draw[connector,srGray!80] (2.70,2.05)--(3.14,2.05);
\draw[connector,srGray!80] (5.97,2.05)--(6.41,2.05);
\node[text=srInk,font=\fontsize{14}{16}\selectfont]
 at (9.85,2.05) {$=$};

\begin{scope}[shift={(11.90,2.05)}]
 \draw[srBlue,fill=srGreen!7,line width=.8pt]
  (.68,0) circle (2pt);
 \fill[srGreen] (0,.68) circle (2.2pt);

 \node[below,labelstate,text=srBlue]
  at (.68,-.06) {$\vp$};
 \node[above,labelstate,text=srGreen]
  at (0,.74) {$\mR_{\pi/2}\vp$};

 \draw[-{Latex[length=1.5mm]},srGreen,line width=1pt]
  (8:.68) arc[start angle=8,end angle=82,radius=.68];
 \node[text=srGreen,anchor=west,inner sep=2pt]
  at (.59,.58) {$+90^\circ$};
\end{scope}

\end{tikzpicture}%
}\vspace{-4mm}
\caption{Visualization of CKDA applying a rotation to a vector $\vp$ within a single recurrent update $(\mI - \beta\vk\vk^\top)\Diag(-1,1)$ by combining a signed diagonal gate with a Householder reflection ($\beta=2$). $\vk=(1,1)^\top/\sqrt2$ allows a $90^\circ$ rotation, while varying $\vk$ allows any planar rotation angle.}    \label{fig:signed-householder-rotation}
\par\medskip
    \includegraphics[width=0.95\linewidth]{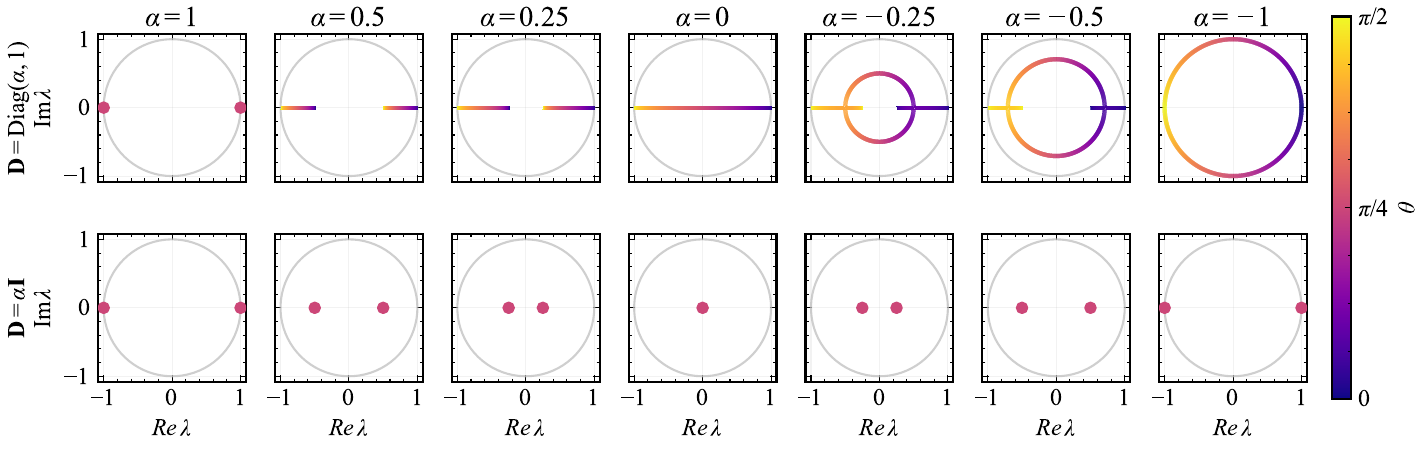}
    \vspace{-5mm}
    \caption{Spectral view of the two-dimensional construction in \Eqref{eq:kda_alpha_rotation}. For $\mA=(\mI-2\vk\vk^\top)\mD$ with $\vk=(\cos\theta,\sin\theta)^\top$ and $\theta\in[0,\pi/2]$. For the channel-wise gate $\mD=\Diag(\alpha,1)$, negative $\alpha$ allows a complex-conjugate pair, which moves toward the unit circle as $\alpha$ approaches $-1$ and recovers a rotation. For the scalar gate $\mD=\alpha\mI$, the spectrum remains real (see also~\Cref{fig:beta_alpha_sweep}).}
    \label{fig:alpha_sweep_2d}
\end{figure}

\section{Introduction}
Linear recurrent neural networks (RNNs) offer efficient sequence modeling with linear scaling in sequence length and a fixed-size recurrent state. Their efficiency and expressivity depend on the structure of their state-transition matrices: diagonal transitions, as used in Mamba-1/2~\citep{gu2023mamba,dao2024transformers}, GLA~\citep{gla}, and mLSTM~\citep{beck2024xlstm}, support fast computation, while non-diagonal transitions using the delta-rule~\citep{schlag2021linear} introduce a rank-one correction that mixes information across state coordinates~\citep{yang2024deltanet,peng2025rwkv7,hatamizadeh2026gdn2}. Gated delta-rule models have become recurrent backbones of large language models, with two variants differing in how they parameterize the gate in the state-transition $\mA_t=(\mI-\beta_t\vk_t\vk_t^\top)\Diag(\valpha_t)$. Gated DeltaNet (GDN)~\citep{yang2025gdn} uses a scalar gate, $\valpha_t=\alpha_t\mathbf{1}$, and is adopted in several recent language models~\citep{team2026qwen3,qwen3,merrill2026olmo}. Kimi Delta Attention (KDA)~\citep{kimi2025} allows a separate gate value per channel and has likewise been adopted in recent models~\citep{upstage2026,zai2026glm53flash,inclusionai2026ling3,team2026kimi}. Both retain diagonal-plus-rank-one transitions, motivating us to understand what the shift from a scalar to a full diagonal adds to their expressivity.

We study this expressivity question through state tracking, which requires composing input-dependent updates over time, as in parity, modular addition or general permutation composition. 
Equipped with increasingly complex orthogonal transitions, Linear RNNs can solve harder state-tracking problems with one layer: 1D reflections for parity, 2D rotations for modular addition and higher-dimensional orthogonal representations for permutation composition. DeltaProduct$_k$~\citep{siems2025} controls this complexity by composing $k$ delta-rule transitions per token, each an identity-plus-rank-one matrix. In particular, one layer can solve modular addition via 2D rotations when $k=2$, using two delta-rule updates per token but increasing the computational cost relative to DeltaProduct$_1$.

We show that KDA's channel-wise gate provides the structure needed to break the symmetry of the Householder--diagonal transition, making complex eigenvalues possible, whereas GDN's scalar gate cannot break this symmetry. However, KDA's standard nonnegative gate still restricts the transition to a real spectrum. We introduce \emph{Complex KDA (CKDA)} by combining two existing range extensions: gate entries in $[-1,1]$~\citep{sarrof2024} and $\beta_t\in[0,2]$~\citep{grazzi2025}.
As a consequence, CKDA can realize any 2D rotation while maintaining non-expansiveness and efficient recurrent computation. CKDA realizes planar rotations using gate entries of opposite sign to supply a coordinate reflection, which combines with the freely oriented Householder reflection at $\beta_t=2$ (Figures~\ref{fig:signed-householder-rotation} and~\ref{fig:alpha_sweep_2d}). Our main contributions are:
\vspace{-3mm}
\begin{itemize}[leftmargin=*]
    \item We characterize the spectrum of CKDA's transitions and prove that every orthogonal diagonal-plus-rank-one matrix is exactly a CKDA transition. We also study the expressivity of products of CKDA matrices (\Cref{sec:dplr}).
    \item One CKDA layer tracks every finite subgroup of $\mathrm{SO}(3)$, including $S_3$, $S_4$, and $A_5$ (\Cref{thm:main-positive}). Three layers solve arbitrary finite-group word problems and, with $\beta > 2$, simulate weighted finite automata in polynomial precision (\Cref{thm:multilayer}). These bounds match DeltaProduct$_2$ and use one layer fewer than (Gated) DeltaNet.
    \item We rule out one-layer $S_5$ tracking for CKDA and DeltaProduct$_k$ ($k\leq3$) with non-expansive transitions and finite reachability (\Cref{thm:s5-impossibility}). Finite reachability (the set of RNN states is finite) holds in most group-tracking constructions considered in the literature and simplifies the analysis.
    \item Combining both extensions yields the strongest length extrapolation among the tested KDA range settings on $S_3$, $S_4$, and periodic waveform continuation. At 1.3B parameters and 100B tokens in language modeling, CKDA performs on par with KDA, while outperforming Transformers and other linear RNN architectures. Our implementation is a minor modification of KDA recurrence kernels from FLA~\citep{yang2024fla} and retains $96$--$97\%$ of standard KDA's throughput.
\end{itemize}

\section{Background \& Related Work}\label{sec:bg}
\vspace{-3mm}
\textbf{Linear RNNs.} Linear recurrent neural networks process sequences through stacked layers with affine state updates. For one head, following~\citet{grazzi2025}, we write
\begin{equation}\label{eq:linrnn}
\mH_i = \mA(\vx_i)\mH_{i-1}+\mB(\vx_i),\qquad
\hat{\vy}_i = \mathrm{dec}(\mH_i,\vx_i), \qquad i=1,\dots,t.
\end{equation}
Here $\mH_i\in\R^{n\times d_v}$ is the recurrent state; the learned maps $\mA$, $\mB$, and $\mathrm{dec}$ specify the transition applied to its columns, additive update, and output. Architectures differ mainly in the structure imposed on $\mA$. GLA~\citep{gla} uses diagonal transitions, while Mamba-2~\citep{dao2024transformers} and mLSTM~\citep{beck2024xlstm} use scalar-times-identity transitions within each head. Diagonal-plus-low-rank (DPLR) transitions permit channel mixing. DeltaNet~\citep{schlag2021linear,yang2024deltanet} uses $\mI-\beta_i\vk_i\vk_i^\top$: for unit keys, $\beta_i=0,1,2$ give identity, projection, and Householder reflection~\citep{householder1958unitary,grazzi2025}, respectively. Gated DeltaNet~\citep{yang2025gdn} adds scalar decay, whereas KDA~\citep{kimi2025} uses a channel-wise positive gate, $\mA_i=(\mI-\beta_i\vk_i\vk_i^\top)\Diag(\valpha_i)$; both retain diagonal-plus-rank-one transitions and efficient WY-based chunk-wise implementations. Related delta architectures introduce channel-wise learning rates and separate erase/write gates~\citep{hatamizadeh2026gdn2}, or asymmetric rank-one corrections~\citep{peng2025rwkv7}. 

\textbf{Complex and rotation-based recurrences.} S4~\citep{gu2022efficiently} and LRU~\citep{orvieto2023} use complex DPLR and complex-diagonal representations, respectively. Earlier RNNs combine nonlinear activations with unitary~\citep{arjovsky2016unitary} or Householder-parameterized~\citep{mhammedi2017efficient} orthogonal recurrent matrices, unlike the affine updates we consider. RotRNN~\citep{biegun2024rotrnn} uses rotation-based linear recurrences, while \citet{orvieto2024universality} analyze the benefits of complex eigenvalues for finite-window memory reconstruction. DeltaProduct~\citep{siems2025} realizes planar rotations using two Householder factors per token, requiring an additional rank-one update. Separately, higher-order and block-diagonal LRUs~\citep{dubinin2026} enrich state mixing outside the delta-rule family. Selective RoPE~\citep{movahedi2025srope} adds input-dependent rotations to gated attention, Mamba-3~\citep{lahoti2026mamba} uses complex dynamics, and Adaptive Unitary SSMs~\citep{karuvally2025bridging} use input-dependent unitary transitions sharing a diagonalizing basis. MDN~\citep{huang2026} adds an auxiliary momentum state, yielding second-order dynamics with complex-conjugate eigenvalues. Semidirect Fourier Delta Attention~\citep{zhang2026sfda} combines explicit phase/decay gates with delta updates and chunk-WY algorithms; its phases admit real $2\times2$ rotation implementations and realize cyclic counters even without delta updates. CKDA instead retains KDA's first-order, real diagonal-plus-rank-one recurrence: signed gates $\valpha_i\in[-1,1]^n$~\citep{sarrof2024} and $\beta_i\in[0,2]$~\citep{grazzi2025} let coordinate and delta-rule reflections compose into noncommuting rotation families, without auxiliary recurrent states or explicit phase gates. See~\Cref{tab:recurrence-spectrum} for a comparison between CKDA, MDN, and SFDA.

\textbf{Formal languages and recurrent expressivity.} Finite monoids recognize exactly the regular languages; non-commutative group problems require order-sensitive composition, with $S_5$ linked to the computational complexity class $\mathsf{NC}^1$ through permutation branching programs~\citep{barrington1986bounded}. We consider every group element as an input, not only generators. Under the finite-precision assumptions of \citet{shakerinava2026the}, single-layer input-dependent complex-diagonal SSMs track exactly the finite abelian groups, excluding non-abelian tracking despite complex eigenvalues. Alternative transitions use bilinear interactions~\citep{ebrahimi2026revisiting}, selective SSM parameterizations for automaton emulation~\citep{terzic2025expressiveness}, or fixed-point iterations~\citep{movahedi2026fixed}. Our multilayer results adapt finite-state and rational weighted-automaton constructions~\citep{siems2025,peng2025rwkv7,merrill_why_2026}, replacing the two-layer DeltaNet clock with one CKDA layer. Circuit-complexity bounds depend on precision and evaluation assumptions~\citep{merrill2024,merrill_why_2026}; algebraic analyses show that discretization and evaluation order can change expressivity~\citep{nowak2026algebraic}. Recent work highlights a gap between the expressive capacity of linear RNNs and their robustness in long-horizon state tracking, emphasizing limitations in error correction that can allow perturbations to accumulate and compromise state representations~\citep{dankowiakowski2026metric,chung2026rethinking}.

\textbf{Beyond abstract group-word benchmarks.} \citet{siems2026learning,merrill2026olmo} study permutation tracking in next-token-style code traces; \citet{Shin2026CanVW} report improved extrapolation when permitting negative transition eigenvalues in their action-conditioned video Shell Game. 
\section{Motivation: From Symmetry to Rotation}
\label{sec:free}
We begin by isolating the mechanism through which channel-wise signed gating changes the transition geometry: We show that scalar gates commute with every matrix, whereas signed diagonal gates can combine with the Householder update to produce rotations.

\textbf{Symmetry.}
A KDA state-transition matrix $\mA=\mH_{\vk}\mD$, with $\mH_{\vk}:=\mI-\beta\vk\vk^\top$ and $\mD:=\Diag(\valpha)$, is the product of two symmetric matrices. $\mA$ is symmetric if and only if the factors commute, since $\mA^\top=\mD\mH_{\vk}$. Entrywise, the commutation condition $\mH_{\vk}\mD = \mD\mH_{\vk}$ becomes
\[
    (\mH_{\vk}\mD-\mD\mH_{\vk})_{ij}
    =
    \beta k_i k_j(\alpha_i-\alpha_j)
    =
    0
    \qquad \text{for all } i,j.
\]

\textbf{Scalar Gate.}
For Gated DeltaNet, $\alpha_i=\alpha$ for every coordinate, so the commutation condition above is always true. Hence $\mA=\alpha(\mI-\beta\vk\vk^\top)$ is symmetric and has a real spectrum. Moreover, over several recurrent steps the scalar gates factor entirely out: let $\mA_t=\alpha_t(\mI-\beta_t\vk_t\vk_t^\top)$, then
\[
    \mA_T\cdots\mA_1
    =
    \left(\prod_{t=1}^T\alpha_t\right)
    \bigl(\mI-\beta_T\vk_T\vk_T^\top\bigr)
    \cdots
    \bigl(\mI-\beta_1\vk_1\vk_1^\top\bigr).
\]
Thus extending the scalar gate to $[-1,1]$ can change the overall scale and introduce a global sign, but it cannot contribute an additional independently oriented transformation (\Cref{fig:alpha_sweep_2d} bottom row).

\textbf{Diagonal Gate.}
For a channel-wise gate, differences $\alpha_i-\alpha_j$ need not vanish. Whenever $\beta\ne0$ and the key has nonzero components on two coordinates with different gate values, the two symmetric factors do not commute and the resulting transition is nonsymmetric. Noncommutation alone, however, is not sufficient to produce a non-real spectrum. For strictly positive gates, $\mH_{\vk}\mD$ is similar to the symmetric matrix $\mD^{1/2}\mH_{\vk}\mD^{1/2}$, so all its eigenvalues are real. The same conclusion holds when some gates are zero, by continuity. Hence standard nonnegative KDA still has a real spectrum. As we show next, allowing the diagonal gate to change sign removes this restriction.

We see the resulting geometry in a simple 2D case. Consider the KDA transition with $\beta{=}2$ given by
\[
    \mA_{\alpha,\theta}
    =
    \mH_{\vk}\mD_\alpha,
    \qquad
    \mH_{\vk}:=\mI-2\vk\vk^\top,
    \qquad
    \vk=(\cos\theta,\sin\theta)^\top,
    \qquad
    \mD_\alpha:=\Diag(\alpha,1).
\]
Expanding the product gives
\begin{equation}
    \mA_{\alpha,\theta}
    =
    \begin{pmatrix}
        -\alpha\cos 2\theta & -\sin 2\theta\\
        -\alpha\sin 2\theta & \cos 2\theta
    \end{pmatrix}.
    \label{eq:kda_alpha_rotation}
\end{equation}
Since $\mA_{\alpha,\theta}\in\mathbb R^{2\times2}$, its eigenvalues are non-real precisely when its discriminant is negative,
\[
    \Delta
    :=
    \operatorname{tr}(\mA_{\alpha,\theta})^2
    -4\det(\mA_{\alpha,\theta})
    =
    (1-\alpha)^2\cos^2(2\theta)+4\alpha
    <0.
\]
Hence, for $\alpha\geq0$, the spectrum is necessarily real, whereas for $\alpha<0$, a complex-conjugate pair can exist. Whenever the eigenvalues are non-real, their product is $\det(\mA_{\alpha,\theta})=-\alpha$, and hence $|\lambda|=\sqrt{-\alpha}$. Thus, as $\alpha$ moves from $0$ toward $-1$, the complex pair moves outward toward the unit circle (see~\Cref{fig:alpha_sweep_2d}).
At the endpoint $\alpha=-1$, we have $\Diag(-1,1)=\mI-2\ve_1\ve_1^\top=\mH_{\ve_1}$, and hence $\mA_{-1,\theta}=\mH_{\vk}\mH_{\ve_1}$. The transition is therefore the composition of two reflections whose mirror lines differ by an angle $\theta$. Their composition is a rotation by $2\theta$. 

This perspective extends coordinate by coordinate. A diagonal matrix can be decomposed as
\begin{align}
    \Diag(\valpha)
    &=
    \bigl(\mI-(1-\alpha_1)\ve_1\ve_1^\top\bigr)
    \bigl(\mI-(1-\alpha_2)\ve_2\ve_2^\top\bigr)
    \cdots
    \bigl(\mI-(1-\alpha_n)\ve_n\ve_n^\top\bigr).
    \label{eq:diagfactor}
\end{align}
Thus each factor is an axis-aligned generalized Householder transformation with normal $\ve_i$ and learning rate $\beta_i=1-\alpha_i$. This can also be viewed as a special case of the generalized Householder composition studied by~\citet{siems2025}. In particular, when $\alpha_i=-1$, the corresponding factor is an exact coordinate reflection. Hence a KDA state transition $\mA_t=(\mI-\beta_t\vk_t\vk_t^\top)\Diag(\valpha_t)$ combines one freely oriented generalized Householder transformation with $n$ axis-aligned ones. 

\textbf{CKDA.}
We define CKDA as KDA with $\valpha_t\in[-1,1]^n$ and $\beta_t\in[0,2]$; our implementation uses $\beta_t=2\sigma(b_t)$ and a signed gate based on $r_{t,i}=2\sigma(a_{t,i})-1$, following~\citet{sarrof2024} and \citet{grazzi2025} with the magnitude floor and backward convention specified in Appendix~\ref{app:gauge}.

\begin{table}[t]
\centering
\begingroup
\hypersetup{hidelinks}
\caption{\textbf{Expressivity of linear recurrent models}}
\label{tab:theory-summary}
\begingroup
\fontsize{8.5}{9.5}\selectfont
\setlength{\tabcolsep}{1pt}
\renewcommand{\arraystretch}{1.04}
\renewcommand{\tabularxcolumn}[1]{m{#1}}
\newcommand{\yes}{\textcolor{good}{$\checkmark$}\,}
\newcommand{\no}{\textcolor{bad}{$\times$}\,}
\newcommand{\nl}{\newline}
\newcommand{\ours}[1]{\textcolor{acc}{\textbf{(Thm.~\ref{#1})}}}
\definecolor{expansive}{HTML}{A65E00}
\newcommand{\relaxed}[1]{\textcolor{expansive}{#1}}
\newcommand{\restricted}{\textsuperscript{\ensuremath{\dagger}}}
\newcommand{\unitgate}{\textsuperscript{\ensuremath{\ddagger}}}
\newcommand{\fpentry}[2]{\mbox{\yes$#1$#2\enspace\no$1$\,FP[U2]}}
\newcommand{\defaultparams}[1]{\mbox{\thickmuskip=2mu\medmuskip=2mu$#1$}}
\begin{tabularx}{\linewidth}{@{}>{\raggedright\arraybackslash}m{.16\linewidth}>{\hsize=.96\hsize\linewidth=\hsize\centering\arraybackslash}X>{\hsize=.90\hsize\linewidth=\hsize\centering\arraybackslash}X>{\hsize=1.20\hsize\linewidth=\hsize\centering\arraybackslash}X>{\hsize=.94\hsize\linewidth=\hsize\centering\arraybackslash}X@{}}
\toprule
& \textcolor{acc}{\textbf{CKDA}} & \textbf{GDN} &
\textbf{Gated DeltaProduct} & \textbf{RWKV-7 / GDN-2} \\
& $(\mI-\beta\vk\vk^\top)D(\valpha)$
& $\alpha(\mI-\beta\vk\vk^\top)$
& $\alpha\prod_{j=1}^{k}(\mI-\beta_j\vk_j\vk_j^\top)$
& $D(\valpha)-\vk(\vv\odot\vk)^\top$ \\
\midrule
Param. ranges
& \defaultparams{\beta\in[0,2],\,\valpha\in[-1,1]^d}
& \defaultparams{\alpha\in[0,1],\,\beta\in[0,2]}
& \defaultparams{\alpha\in[0,1],\,\beta_j\in[0,2]}
& \defaultparams{\valpha\in[0,1]^d,\,\relaxed{\vv\in[0,2]^d}} \\
\midrule
Complex\nl pairs
& \cellcolor{acc!7}$\boldsymbol{\leq1}$ \ours{thm:complex_eigenvalues}\nl
$0$ if $\beta\leq1$ or $\valpha\geq0$
& $0$
& $\leq\lfloor k/2\rfloor$
& $0$ \\
\midrule
$S_2$ (parity)
& \yes $1$ via [U3]\restricted
& \yes $1$ via [U3]
& \yes $1$ ($k\geq1$) via [U3]
& \yes $1$ [RL1]\unitgate \\
\addlinespace[2pt]
$\mathbb Z_n,D_n$ ($n>2$)
& \cellcolor{acc!7}\yes $\mathbf1$ \ours{thm:main-positive}\restricted
& \fpentry{2}{[D7]}
& \yes $1$ ($k\geq2$) [D7]
& \fpentry{2}{[D7]\unitgate} \\
\addlinespace[2pt]
$S_4,A_5$
& \cellcolor{acc!7}\yes $\mathbf1$ \ours{thm:main-positive}\restricted
& \fpentry{4}{[D1]}
& \yes $1$ ($k\geq2$) [D4]
& \fpentry{4}{[D1]\unitgate} \\
\addlinespace[2pt]
$S_n$\nl ($n\geq5$)
& \cellcolor{acc!7}\yes $\mathbf3$ \ours{thm:multilayer}\nl
\no $\mathbf1$ FR \ours{thm:s5-impossibility}
& \fpentry{4}{[D1]}
& \yes $1$: $k\geq n-1$ [D1]\nl
\yes $3$: $2\leq k\leq n-2$ [D1]\nl
{\setlength{\fboxsep}{0pt}\colorbox{acc!7}{\parbox{\linewidth}{\centering
\no $\mathbf1$ FR: $k\leq3$ \ours{thm:s5-impossibility}}}}
& \fpentry{4}{[D1]\unitgate} \\
\midrule
Regular languages
& \yes $L_q$ via [D2]
& \yes $L_q$ [D2]
& \yes $L_q$ ($k\geq1$) [D2]
& \yes $L_q$ via [D2]\unitgate \\
\addlinespace[2pt]
\quad (unstable)
& \cellcolor{acc!7}\yes $\mathbf3$ \ours{thm:multilayer}\nl
allow \relaxed{$\beta>2$}
& \yes $4$ via [M11]\nl \relaxed{$\beta\in[0,4]$}
& \yes $3$ ($k\geq2$); $1$ ($k\geq4q$)\nl \relaxed{$\beta\in[0,4]$} via [RL3, M11]
& \yes $4$ [R3] \\
\addlinespace[2pt]
WFAs\nl over $\mathbb Q$ (PP)
& \cellcolor{acc!7}\yes $\mathbf3$ \ours{thm:multilayer}\nl
\relaxed{$\beta\in[0,\infty)$}
& \yes $4$ [M11]\nl \relaxed{$\beta\in\mathbb R$}
& \yes $3$ ($k\geq2$); $1$ ($k\geq B_q$)\nl \relaxed{$\beta\in\mathbb R$} via [M11, ML12]
& \yes $4$ [M6]\nl \relaxed{unrestricted} \\
\bottomrule
\end{tabularx}
\par\vspace{4pt}
\parbox{\linewidth}{\raggedright
\textcolor{good}{$\checkmark$} L: depth L is sufficient; \textcolor{bad}{$\times$} L A: depth L is impossible under A.
$L_q$: a sufficient finite depth depending on the language's $q$-state DFA.
Unstable / \relaxed{amber}: expansive transitions allowed.
\restricted\ CKDA: $\beta\in\{0,2\}$, $\valpha\in\{\pm1\}^d$ ($\valpha=\mathbf1$ for parity).
\unitgate\ RWKV-7/GDN-2 constructions use $D(\valpha)=\gamma\mI$, $\vv=\gamma\beta\mathbf1$, with $\gamma\in[0,1]$, $\beta\in[0,2]$ (stable: $\|\mA\|_2\leq1$); reflections use $\gamma=1$, $\beta=2$.
RWKV-7/GDN-2: unit keys, nonnegative gates (RWKV-7's factor $c=2$ is absorbed into $\vv$).
Both allow expansive transitions by default; GDN-2 entries use its extended erase range $[0,2]^d$.
$D(\valpha)=\Diag(\valpha)$; $q$: automaton size; $B_q=8q^2+5q+1$ ($2q+1$ coordinates).
Group inputs: all elements; complex pairs count multiplicity per head.
FP: fixed precision (exact datatypes).
PP: polynomial bit length over a fixed algebraic number field (App.~\ref{app:tracking-precision}).
FR: per-head finite reachability under exact affine updates (App.~\ref{app:tracking-realization}).
\par
U~\citep{grazzi2025}, D~\citep{siems2025}, R~\citep{peng2025rwkv7}, M~\citep{merrill_why_2026}:
[D1]/[RL1]: Theorem/Lemma~1; ``via'': derived bounds.
Comparison details: App.~\ref{app:table-comparison}.
\ifdefined\expressivitysinglehead
This assumes one head.
\fi}
\endgroup
\endgroup
\end{table}

\begin{figure}[t]
\centering
\begingroup
\def\shFigureAngle{70}
\def\shFigureTrimBottom{5mm}

\adjustbox{
 trim={0mm {\shFigureTrimBottom} 0mm 0mm},
 clip,
 width=0.99\linewidth
}{%
\begin{tikzpicture}[
 x=1cm,y=1cm,font=\fontsize{9}{10}\selectfont,
 text=srInk,>=Latex,line cap=round,line join=round,
 vec/.style={-{Latex[length=1.8mm]},line width=1pt},
 ghost/.style={
   -{Latex[length=1.5mm]},dashed,line width=.7pt
 },
 motion/.style={
   -{Latex[length=1.3mm]},densely dotted,line width=.75pt
 },
 connector/.style={
   -{Latex[length=1.5mm]},line width=.7pt,srGray!70
 },
 labelstate/.style={inner sep=2pt},
 equation/.style={anchor=west,inner sep=0pt}
]
\path[use as bounding box] (0,0) rectangle (16.2,2.88);

\node at (2.20,2.60) {(a) Orthogonal columns};
\node at (7.85,2.60) {(b) Reflect both columns};
\node at (13.60,2.60) {(c) Resulting signed coordinate frame};

\pgfmathsetmacro{\shFigR}{.82}
\pgfmathsetmacro{\shFigAX}{\shFigR*cos(\shFigureAngle)}
\pgfmathsetmacro{\shFigAY}{\shFigR*sin(\shFigureAngle)}
\pgfmathsetmacro{\shFigBX}{-\shFigAY}
\pgfmathsetmacro{\shFigBY}{\shFigAX}
\pgfmathsetmacro{\shFigNormal}{\shFigureAngle/2}
\pgfmathsetmacro{\shFigMirror}{\shFigNormal+90}
\pgfmathsetmacro{\shFigMirrorOpp}{\shFigMirror+180}

\foreach \shFigCenter in {1.30,6.85,12.40} {
 \begin{scope}[shift={(\shFigCenter,1.12)}]
  \draw[srGray!30,line width=.35pt]
   (-1.06,0)--(1.06,0);
  \draw[srGray!30,line width=.35pt]
   (0,-1.01)--(0,1.04);
  \draw[srGray!16,line width=.35pt]
   (0,0) circle (\shFigR);
 \end{scope}
}

\begin{scope}[shift={(1.30,1.12)}]
 \draw[vec,srBlue] (0,0)--(\shFigAX,\shFigAY);
 \draw[vec,srGold] (0,0)--(\shFigBX,\shFigBY);

 \node[above right,labelstate,text=srBlue]
  at (\shFigAX,\shFigAY) {$\va_1$};
 \node[above left,labelstate,text=srGold]
  at (\shFigBX,\shFigBY) {$\va_2$};

 \begin{scope}[rotate=\shFigureAngle]
  \draw[srGray!75,line width=.45pt]
   (.13,0)--(.13,.13)--(0,.13);
 \end{scope}
\end{scope}

\begin{scope}[shift={(6.85,1.12)}]
 \draw[srGreen,line width=.85pt]
  (\shFigMirror:1.12)--(\shFigMirrorOpp:1.12);
 \node[above left,labelstate,text=srGreen]
  at (\shFigMirror:1.12) {$\vk^\perp$};

 \draw[ghost,srBlue!55] (0,0)--(\shFigAX,\shFigAY);
 \draw[ghost,srGold!55] (0,0)--(\shFigBX,\shFigBY);
 \node[above right,labelstate,text=srBlue!65]
  at (\shFigAX,\shFigAY) {$\va_1$};
 \node[above left,labelstate,text=srGold!65]
  at (\shFigBX,\shFigBY) {$\va_2$};

 \draw[motion,srBlue!75]
  (\shFigAX,\shFigAY)--(-\shFigR,0);
 \draw[motion,srGold!85]
  (\shFigBX,\shFigBY)--(0,\shFigR);

 \draw[vec,srBlue] (0,0)--(-\shFigR,0);
 \draw[vec,srGold] (0,0)--(0,\shFigR);
 \node[below left,labelstate,text=srBlue]
  at (-\shFigR,0) {$-\ve_1$};
 \node[above left,labelstate,text=srGold]
  at (0,\shFigR) {$\ve_2$};

 \draw[vec,srGreen]
  (0,0)--(\shFigNormal:\shFigR);
 \node[right,labelstate,text=srGreen]
  at (\shFigNormal:.84) {$\vk$};

 \begin{scope}[rotate=\shFigNormal]
  \draw[srGreen!75,line width=.45pt]
   (0,.11)--(.11,.11)--(.11,0);
 \end{scope}
\end{scope}

\begin{scope}[shift={(12.40,1.12)}]
 \draw[vec,srBlue] (0,0)--(-\shFigR,0);
 \draw[vec,srGold] (0,0)--(0,\shFigR);
 \node[below left,labelstate,text=srBlue]
  at (-\shFigR+1mm,0) {$-\ve_1$};
 \node[above right,labelstate,text=srGold]
  at (0,\shFigR) {$\ve_2$};

 \draw[srGray!75,line width=.45pt]
  (-.13,0)--(-.13,.13)--(0,.13);
\end{scope}

\node[equation] at (2.53,1.12)
 {$\mA=[\va_1,\va_2]$};

\node[equation,text=srGreen] at (8.08,1.12)
 {$\mH=\mI-2\vk\vk^\top$};

\node[equation] at (13.63,1.12)
 {$\begin{gathered}
    \mH\mA=\mS\\[-1pt]
    \mS=\Diag(-1,1)
   \end{gathered}$};

\draw[connector] (4.60,1.12)--(5.15,1.12);
\draw[connector] (10.60,1.12)--(10.95,1.12);

\end{tikzpicture}%
}
\endgroup
\vspace{-3mm}
\caption{Intuition for~\Cref{prop:orthogonal-dplr} in 2D. A single reflection aligns a 2D orthogonal frame with the
coordinate axes, giving $\mH\mA=\mS$ and hence $\mA=\mH\mS$.}
\label{fig:dpr1_orth}
\end{figure}
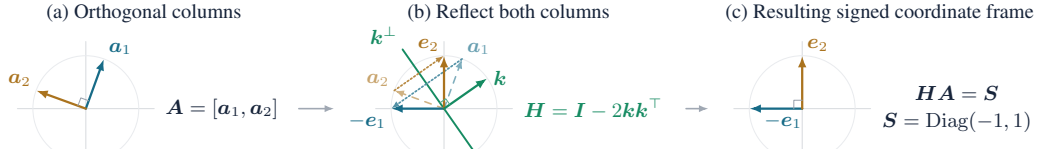

\section{Structure and Spectrum of Complex KDA}\label{sec:dplr}
\vspace{-3mm}
The planar construction shows how signed gating enables rotations within a rank-one diagonal-plus-low-rank (DPLR) transition. In this section, we study CKDA in arbitrary dimensions. Our first finding is that CKDA captures the entire orthogonal rank-one DPLR (DPR1) family.

\begin{theorem}[DPR1 and CKDA]
\label{prop:orthogonal-dplr}
Every orthogonal DPR1 matrix $\mA=\mD+\vu\vv^\top\in\mathbb R^{n\times n}$
with diagonal $\mD$ can be written in the form $\mA=(\mI-2\vk\vk^\top)\mS, \|\vk\|_2=1, \mS=\Diag(s_i), s_i\in\{-1,+1\}$. We call this form a \emph{signed-Householder matrix}, i.e.\@ a CKDA with $\alpha_i \in \{+1,-1\}$, $\beta=2$.
\end{theorem}
\textbf{Geometric intuition in 2D~(\Cref{fig:dpr1_orth}).}
The columns of an orthogonal matrix $\mA$ form a perpendicular pair of unit vectors. Choose a reflection $\mH$ across a line through the origin that sends the first column onto the first coordinate axis. Since reflections preserve lengths and perpendicularity, the second column must then lie on the second coordinate axis. Thus $\mH\mA=\mS$ for a diagonal sign matrix $\mS$, and $\mH^2=\mI$ gives $\mA=\mH\mS$. In higher dimensions, aligning one column does not automatically align the others; the DPR1 structure guarantees that a suitable single reflection still aligns the entire frame.

Thus replacing KDA's structured rank-one term by a non-symmetric $\vu\vv^\top$, such as in RWKV-7, adds no orthogonal transitions.
The next result generalizes the 2D case of the motivation section to higher dimensions: each CKDA transform can be viewed as an element-wise scaling followed by a transform which can be a rotation only in a 2D subspace.
\begin{proposition}[Sign-magnitude decomposition]
\label{prop:CKDA-canonical}
Let $\|\vk\|_2=1$, $\beta\in[0,2]$, and $\alpha_i\in[-1,1]$.
Write $\mA=(\mI-\beta\vk\vk^\top)\Diag(\valpha)
=\widetilde{\mA}\Diag(|\valpha|)$, where
$\widetilde{\mA}=(\mI-\beta\vk\vk^\top)\mS$ and
$\mS=\Diag(\operatorname{sign}\valpha)$, choosing
$\operatorname{sign}(0)=1$ for this factorization.
Let $E_\pm=\ker(\mS\mp\mI)$ and decompose
$\vk=\vk_-+\vk_+$ with $\vk_\pm\in E_\pm$,
$\|\vk_-\|=\cos\theta$, and $\|\vk_+\|=\sin\theta$,  $\theta\in[0,\pi/2]$.
When both components are nonzero, the restriction of $\widetilde{\mA}$ to
$\mathcal U=\operatorname{span}\{\vk_-,\vk_+\}$, in the orthonormal
basis $\vk_-/\|\vk_-\|,\vk_+/\|\vk_+\|$, is
\[
  \mB(\beta,\theta)=
  \begin{pmatrix}
    \beta\cos^2\theta-1&-\beta\sin\theta\cos\theta\\
    \beta\sin\theta\cos\theta&1-\beta\sin^2\theta
  \end{pmatrix}.
\]
On $\mathcal U^\perp$, $\widetilde{\mA}$ agrees with $\mS$.
The block is a rotation by $2\theta$ when $\beta=2$, and its eigenvalues
are non-real exactly when $\beta^2\cos^2(2\theta)<4(\beta-1)$.
If either component vanishes, $\tilde A$ has only real eigenvalues.
\end{proposition}

Therefore, after the initial coordinate-wise scaling, a 2D rotation with complex eigenvalues can occur only if (i) at least two gate coordinates differ in sign, (ii) the key vector spans coordinates with opposite signs and (iii) $\beta > 1$. The 2D example in the previous section fits exactly these criteria. 

The analysis of the full CKDA transition adds complexity and requires a separate argument: \Cref{thm:complex_eigenvalues} in the appendix proves that any CKDA matrix also has at most one non-real conjugate eigenvalue pair, requiring both $\beta>1$ and a negative gate entry. Thus both range extensions are required for non-real eigenvalues.
Moreover, this restriction extends beyond CKDA: every non-expansive DPR1 matrix has at most one non-real conjugate eigenvalue pair on the unit circle, counted with algebraic multiplicity (\Cref{thm:dplr-persistent-pairs-app}). Products of CKDA transitions can represent any square non-expansive matrix, despite the one-plane restriction on each individual transition. Generalized Householder products already have this universality~\citep[Prop.~1]{grazzi2025}; \Cref{prop:CKDA-factorization} shows that CKDA needs at most $\max\{1,2n-2\}$ factors in dimension $n$. For orthogonal matrices, $\max\{1,n-1\}$ factors suffice and this bound is sharp. In contrast, products of RoPE and Selective RoPE rotations remain block-diagonal rotations in the same fixed coordinate planes~\citep{su2024rope,movahedi2025srope}.

\section{State-Tracking Expressivity}\label{sec:groups}
\vspace{-3mm}
\textbf{Tracking a state-transition system.}
Let $\mathcal S$ be a state space with initial state $s_0$ and an update
$T_a:\mathcal S\to\mathcal S$ for each input symbol $a$.
A recurrent model \emph{tracks} this system if a fixed decoder $f$ recovers
its state from the model's hidden state $h_t$ for every input sequence:
\[
  s_t=T_{a_t}(s_{t-1})\quad(t\geq1),\qquad
  f(h_t)=s_t\quad(t\geq0).
\]
Several hidden states may decode to the same target state.
Systems include finite-group products (see
\Cref{fig:shell_game}), deterministic automata, and weighted finite automata
(WFAs).
Our constructions allow sufficiently expressive feed-forward decoders;
our arithmetic and precision conventions are defined in
Appendix~\ref{app:tracking-precision}. A contextualized summary of our results is in \Cref{tab:theory-summary}.

\begin{figure}[!t]
    \centering

  \resizebox{\textwidth}{!}{\begin{tikzpicture}[
    x=1cm,y=1cm,
    font=\small,
    text=cupInk,
    line cap=round,line join=round,
    swap/.style={-{Latex[length=1.7mm,width=1.2mm]},
                 draw=cupMuted,line width=.65pt},
    pics/cup/.style args={#1/#2}{code={
      \path[draw=#1!85!black,fill=#1!32,line width=.75pt,
            rounded corners=1.2pt]
        (-.215,.285)--(.215,.285)--(.315,-.285)--(-.315,-.285)--cycle;
      \draw[draw=#1!65!black,line width=.45pt]
        (-.268,-.205)--(.268,-.205);
      \node[font=\sffamily\bfseries\small,inner sep=0pt]
        at (0,.015) {#2};
    }}
]
  \pic at (0,0)    {cup={cupViolet/1}};
  \pic at (.82,0)  {cup={cupTerracotta/2}};
  \pic at (1.64,0) {cup={cupTeal/3}};

  \pic at (3.34,0) {cup={cupTerracotta/2}};
  \pic at (4.16,0) {cup={cupViolet/1}};
  \pic at (4.98,0) {cup={cupTeal/3}};

  \pic at (6.68,0) {cup={cupTerracotta/2}};
  \pic at (7.50,0) {cup={cupTeal/3}};
  \pic at (8.32,0) {cup={cupViolet/1}};

  \pic at (10.02,0) {cup={cupViolet/1}};
  \pic at (10.84,0) {cup={cupTeal/3}};
  \pic at (11.66,0) {cup={cupTerracotta/2}};

  \draw[swap] (2.09,-.11)--(2.89,-.11);
  \draw[swap] (5.43,-.11)--(6.23,-.11);
  \draw[swap] (8.77,-.11)--(9.57,-.11);
  \node[font=\footnotesize,inner sep=0pt] at (2.49,.12) {$(12)$};
  \node[font=\footnotesize,inner sep=0pt] at (5.83,.12) {$(23)$};
  \node[font=\footnotesize,inner sep=0pt] at (9.17,.12) {$(13)$};
\end{tikzpicture}}\vspace{-3mm}
    \caption{Example of the $S_3$ permutation task,
    analogous to the shell game with 3 hidden objects.}
    \label{fig:shell_game}
\end{figure}
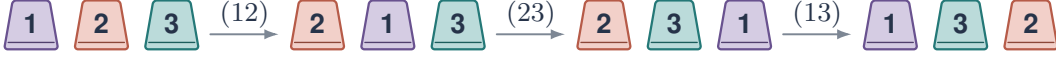

\textbf{Single-layer expressivity.}
Consider a finite group $G$.
A \emph{faithful orthogonal representation} assigns each group element a
distinct orthogonal matrix, so that composing group elements corresponds
exactly to multiplying their matrices. We call a construction that uses
these matrices directly as its transitions a \emph{realization},
and write $d$ for their dimension.
Tracking does not require a faithful representation: a many-to-one decoder
can map distinct hidden states to the same group element, and the input
transitions need not themselves form a representation. Any finite group is isomorphic to a subgroup of a permutation group and can be realized by one Linear RNN layer using permutation matrices as transitions. However, CKDA cannot model arbitrary permutations in a single transition\footnote{We consider the case where inputs range over \emph{all} of $G$. If inputs are restricted to identity and swaps, one-layer DeltaNet suffices \citep{grazzi2025}.}.

\begin{theorem}[Single-layer finite-group expressivity]
\label{thm:main-positive}
A single CKDA layer (one head) tracks every finite group isomorphic to a subgroup of $SO(3)$. Specifically, it  realizes every finite cyclic ($\mathbb{Z}_n$) and dihedral group ($D_n$)
in $d=2$ and $A_4,S_4$ in $d=3$, and tracks $A_5$ in $d=4$.
These constructions use orthogonal transitions and a fixed exact datatype.
\end{theorem}

\textit{Proof sketch.}
The planar rotations from the motivation section, together with reflections,
give the cyclic and dihedral (such as $S_3$) groups. Every three-dimensional rotation is a
product of two Householder reflections, but, apart from the identity and
$180^\circ$ rotations, CKDA requires a rotation axis with a zero coordinate. Reorienting
the cube satisfies this condition for $S_4$ and its subgroup $A_4$, whereas
no orientation works for the icosahedral rotations of $A_5$. In four
dimensions, however, $A_5$ can be tracked using a many-to-one
decoder. See \Cref{app:axis-test,app:planar-groups,app:cube,app:icosahedron}
for the constructions and proofs.

\textbf{Limits of a single CKDA layer.}
Increasing dimension and arbitrary decoders do not remove every
obstruction. The next result shows that non-expansive transitions with only one possible complex eigenvalue pair, such as CKDA (\Cref{thm:complex_eigenvalues}) and (Gated) DeltaProduct$_k$ with $k\leq3$, cannot track $S_5$ in one layer, even with any finite number of independent heads. \Cref{app:s5-spectral-obstruction} gives the proof.

\begin{theorem}[Spectral obstruction to $S_5$ tracking]
\label{thm:s5-impossibility}
Suppose every head transition $\mA$ satisfies $\|\mA\|_2\leq1$ and has
at most one non-real conjugate eigenvalue pair on the unit circle,
counted with algebraic multiplicity.
Then a single recurrent layer with any finite number of independent heads
cannot track $S_5$ when the
updates of each head reach only finitely many states.
\end{theorem}

\textit{Proof sketch.}
Finite reachability and non-expansion let us remove the additive terms and
restrict each head to orthogonal updates, preserving its spectral bound.
The resulting matrices generate a finite group that maps onto $S_5$ through
the decoder. From a five-cycle and its conjugate by a transposition, we
construct two transition-matrix products $\mR_1$ and $\mR_2$ whose blocks in each head are either identities
or planar rotations with the same angle and order divisible by five.
Their commutator $\mC=\mR_1\mR_2\mR_1^{-1}\mR_2^{-1}$ satisfies
$\mC^{10}=\mI$ by \Cref{lem:planar-rotation-product}.
Yet it decodes to a three-cycle, whose tenth power is not the identity,
a contradiction. Finite reachability is natural for exact group tracking: all constructions
considered here and in the cited results satisfy it (\Cref{app:tracking-precision}).
It enables the orthogonal reduction, simplifying the analysis without
requiring a unique hidden state per group element.

\textbf{Multi-layer expressivity.}
The same planar rotations that strengthen one-layer tracking also reduce
the depth needed to simulate general state-transition systems.
\begin{theorem}[Multi-layer expressivity]\label{thm:multilayer}
Three CKDA layers solve every finite group-word problem with a fixed
exact datatype. If $\beta>2$ is allowed, three
layers also recognize every regular language with a fixed exact datatype and
compute every WFA over $\mathbb Q$ in polynomial precision using exact arithmetic
over a fixed algebraic number field containing the clock and normalized-key
parameters (\Cref{app:tracking-precision,app:wfamin}).
\end{theorem}

We adapt the existing clock--buffer--accumulator construction
\citep{peng2025rwkv7,siems2025,merrill_why_2026}.
A clock tracks position modulo a fixed period to schedule updates.
The buffer factors each block transition; the accumulator
applies one factor per step, and a completion readout corrects the delay.
CKDA implements the clock with a planar rotation in one layer,
saving one layer relative to DeltaNet/GDN which instead uses 4 layers for the same result.

\section{Experiments}\label{sec:exp}
\vspace{-3mm}
We evaluate whether CKDA's added expressivity improves state tracking and periodic waveform extrapolation, and assess its language-modeling performance and computational efficiency.

\begin{wrapfigure}[14]{r}{0.365\textwidth}
      \centering
      \vspace{-16pt}
      \includegraphics[width=\linewidth]{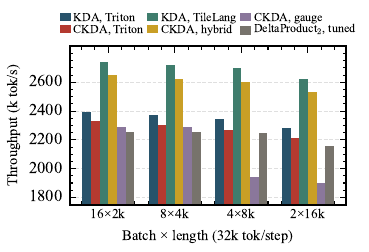}
      \vspace{-22pt}
      \caption{Forward--backward kernel throughput on an H100 in BF16, with 16 heads, $d_k=d_v=128$, and 32k tokens per step. DeltaProduct$_2$ uses no forget gate. Implementation and timing details are in \Cref{app:gauge}.}
      \label{fig:kernel-throughput_recurrence}
\end{wrapfigure}
\textbf{Implementation.}
KDA stores gate magnitudes in log-space, so signed gates require separating signs from magnitudes. We implement the same sign transformation in three ways: (1) compiled PyTorch operations that absorb cumulative signs into keys and queries, without changing kernels (gauge); (2) Triton~\citep{tillet2019triton} kernels that fuse these signs into normalization and its backward pass; and (3) a TileLang-backed~\citep{wang2026tilelang} hybrid that combines these Triton changes with TileLang backward kernels. The kernel implementations retain approximately $96$--$97\%$ of KDA throughput (\Cref{fig:kernel-throughput_recurrence}); \Cref{app:gauge} gives the derivation, implementation details, and benchmark protocol. The ungated DeltaProduct$_2$ baseline achieves competitive throughput while performing two delta-rule updates per token, each with its own value: its measured throughput includes processing an additional value input relative to CKDA. This is due to its identity-plus-rank-one factors, which avoid CKDA's channel-wise gating computations; the gated variant adds only scalar decay. 
We therefore view CKDA as an alternative way to obtain the expressivity similar to DeltaProduct$_2$, without claiming inherent efficiency gains over it. We also use \emph{spread} gate initialization, with approximately equal proportions of positive and negative entries, alongside the original, positive-only \emph{standard} initialization (In Appendix~\ref{app:initialization}, we also propose a variant for $\beta$). Finally, the SiLU activation inherited from DeltaNet~\citep{yang2024deltanet} on key and query projections can bias keys toward positive entries (\Cref{app:silu}). We remove SiLU for KDA state-tracking experiments and ablate this choice for language modeling.

\begin{figure}[t]
\centering
\includegraphics[width=1.0\linewidth]{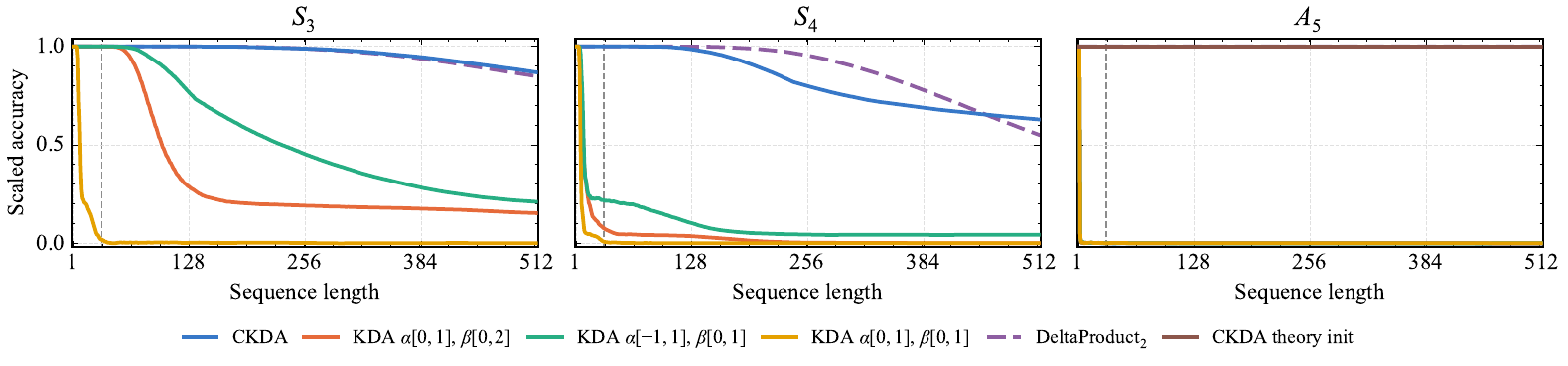}
\vspace{-9mm}
\caption{One-layer KDA variants on the $S_3$, $S_4$, and $A_5$ word problems, by the eigenvalue range allowed
for the diagonal and for the Householder component. Combining a signed gate with a reflection-enabled Householder update gives the strongest extrapolation on $S_3$ and $S_4$ among the tested KDA range settings; $S_4$ accuracy still declines at long lengths. The $A_5$ result uses the separate theory-initialized setup described in the text (see~\Cref{fig:group_word_problems_baselines} for baseline results).}
\label{fig:group_word_problems}
\end{figure}

\begin{figure}[t]
    \centering
    \includegraphics[width=1.0\linewidth]{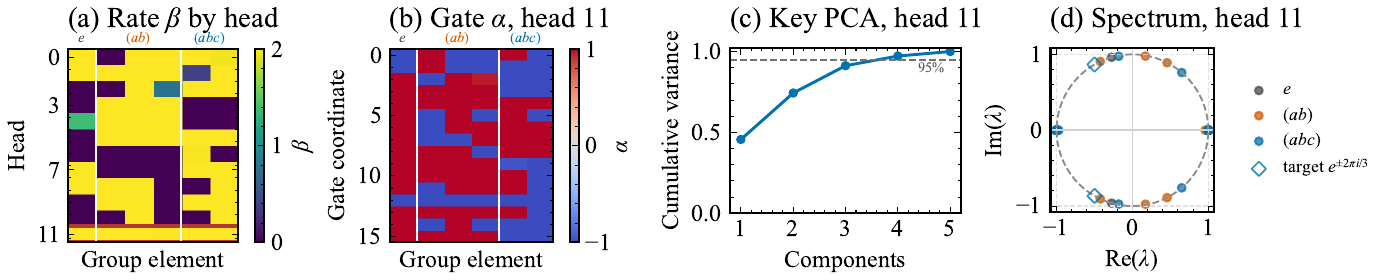}
    \vspace{-8mm}
\caption{
Learned CKDA transitions on the $S_3$ word problem. Here $e$, $(ab)$, and $(abc)$ denote the identity, transpositions, and $3$-cycles, respectively. (a) Learned update strength $\beta$ by head; head~11 approaches the reflection limit $\beta=2$. (b) Its channel-wise gate becomes nearly sign-valued, with coordinates close to $\pm1$. (c) PCA of the keys: the first three principal components explain approximately 95\% of the variance. (d) The resulting transition spectra show complex eigenvalues. The learned model recovers the mechanism predicted by our theory.}
    \label{fig:s3_interpretability}
\end{figure}
\textbf{1. State-Tracking.} We train one-layer models on $S_3$, $S_4$, and $A_5$ group word problems~\citep{merrill2024,terzic2026structured,movahedi2026fixed}, predicting cumulative products from inputs spanning each group. Training lengths reach 32; we report the best of three seeds and scale accuracy from chance (0) to perfect (1). Training details are in Appendix~\ref{app:exp_state_tracking}. CKDA ($\valpha\in[-1,1]^n$, $\beta\in[0,2]$) extrapolates well on $S_3$ and $S_4$ while other settings of KDA fail (\Cref{fig:group_word_problems}; \Cref{fig:group_word_problems_baselines} in Appendix~\ref{appsec:exp_res}), consistent with~\Cref{thm:main-positive}. Extending only the gate or $\beta$ yields long-length $S_3$ scaled accuracy near $0.2$, matching parity-only discrimination of the two $A_3$ cosets. A successful CKDA head learns $\beta\approx2$, nearly sign-valued gates, and complex-conjugate eigenvalues near the unit circle (\Cref{fig:s3_interpretability}), recovering the mechanism of~\Cref{sec:free,prop:CKDA-canonical}. Standard training fails to learn $A_5$, whereas a smaller, fully trainable model initialized near our quaternion construction (Appendix~\ref{app:A5const}) achieves length extrapolation. Training retains ordinary next-state cross-entropy and a standard MLP readout, consistent with representation discovery being an optimization obstacle. The architecture and training schedule also differ, so this comparison does not isolate the effect of initialization.

\textbf{2. Audio continuation.}
To test whether complex eigenvalues help preserve phase beyond symbolic state tracking, we train single-layer models to continue a synthetic periodic groove after a half-bar cue, followed by zero inputs. Whereas the autoregressive audio experiments of Mamba-1~\citep{gu2023mamba} and SaShiMi~\citep{goel2022s} use prediction feedback that can induce nonlinear state dynamics, we compute the continuation in a single forward pass without output feedback. The task requires inferring phase from the cue and maintaining a periodic continuation after the cue ends. Among the four KDA range settings, only CKDA with both range extensions accurately continues the waveform (\Cref{fig:waveform_continuation}). At length 264, beyond the maximum training length of 136, it retains $38.1$ dB SNR, while the causal Transformer falls to $2.8$ dB. Audio samples are \href{https://www.dropbox.com/scl/fo/otpiegwjo6t2w9yonm0uc/ABtNxwvazCXP35_dH6FV8wM?rlkey=zf4igaux3crkcbgm6qztfbfs8&st=v1yp3tcw&dl=0}{available online}. See~\Cref{app:waveform_continuation_details} for experimental details. These results connect learned rotational dynamics to periodic extrapolation, although the GRU~\citep{cho2014properties}, a nonlinear RNN, remains more accurate on this task.

\begin{figure}
    \centering
    \includegraphics[width=0.9\linewidth]{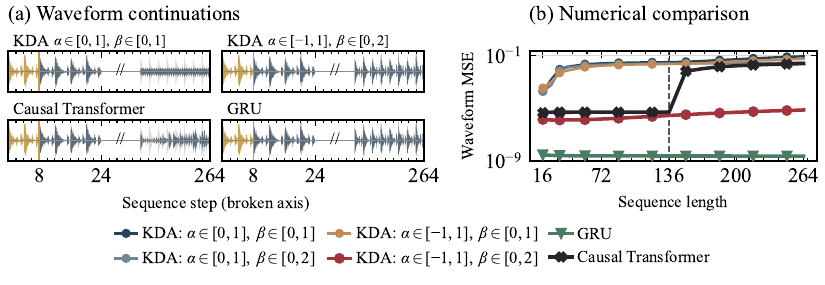}
    \vspace{-7mm}
    \caption{Periodic waveform continuation with single-layer models. (a) Predictions shortly after the half-bar cue and far beyond the training horizon; gold marks the cue, grey the target, and dark blue the prediction. (b) Waveform MSE versus sequence length; the dashed line marks the maximum training length (136). CKDA extrapolates accurately where the other KDA variants and causal Transformer fail, while the GRU achieves the lowest error.}
    \label{fig:waveform_continuation}
\end{figure}

\begin{table*}[t]
  \centering
  \footnotesize
  \setlength{\tabcolsep}{3.5pt}
  \caption{Language modeling and zero-shot common-sense reasoning at 1.3B parameters on 100B tokens of FineWeb-Edu in comparison to values from \citet{hatamizadeh2026gdn2}.}
  \label{tab:fwedu-1p3b-gdn2}
  \resizebox{\textwidth}{!}{
  \begin{tabular}{lrrrrrrrrrrrrrrr}
    \toprule
    & \multicolumn{2}{c}{ppl $\downarrow$} & \multicolumn{10}{c}{accuracy (\%) $\uparrow$} & \multicolumn{3}{c}{recall (\%) $\uparrow$} \\
    \cmidrule(lr){2-3} \cmidrule(lr){4-12} \cmidrule(lr){13-13} \cmidrule(lr){14-16}
    Model & Wiki. & LMB. & LMB. & PIQA & Hella. & Wino. & ARC-e & ARC-c & OBQA & SIQA & BoolQ & Avg. & SQuAD & SWDE & FDA \\
    \midrule
    \multicolumn{16}{l}{\textit{Recurrent models}} \\
    Mamba-2 & 16.79 & 12.38 & 45.24 & 72.58 & 55.51 & 55.33 & 70.68 & 35.26 & 31.00 & 40.63 & 60.19 & 51.82 & -- & -- & -- \\
    Gated DeltaNet & 16.40 & 11.89 & 49.62 & 72.31 & 56.50 & 56.75 & 68.81 & 35.15 & 30.20 & 40.53 & 58.78 & 52.07 & -- & -- & -- \\
    KDA & 16.81 & 11.68 & 48.13 & 72.09 & 55.75 & 55.72 & 70.83 & 35.92 & 30.40 & 40.99 & 60.67 & 52.28 & -- & -- & -- \\
    Mamba-3 (SISO) & 16.30 & 12.99 & 45.06 & 72.31 & 55.58 & 56.20 & 70.45 & 34.56 & 31.00 & \textbf{41.76} & 55.90 & 51.42 & -- & -- & --  \\
    Mamba-3 (MIMO) & 16.45 & 11.66 & 47.82 & 72.36 & 56.49 & 55.78 & 72.38 & 38.07 & 30.00 & 40.89 & 57.74 & 52.39 & -- & -- & --  \\
    Gated DeltaNet-2 & 15.90 & 11.41 & 48.09 & 72.80 & 56.84 & 57.85 & 72.43 & 38.23 & \textbf{31.60} & 40.58 & 59.54 & 53.11 & -- & -- & -- \\
    KDA, bounded gate (ours) & \textbf{15.73} & 10.53 & 50.46 & 72.85 & \textbf{59.32} & \textbf{61.33} & 73.23 & 38.40 & 28.80 & 41.35 & \textbf{61.10} & \textbf{54.09} & 38.17 & \textbf{49.59} & 30.40  \\
    CKDA (ours) & 15.78 & \textbf{10.08} & \textbf{51.66} & \textbf{73.12} & 59.23 & 58.88 & \textbf{73.99} & \textbf{39.08} & 28.60 & 41.61 & 60.34 & 54.06 & \textbf{39.68} & 48.42 & \textbf{31.49}  \\
    \midrule
    \multicolumn{16}{l}{\textit{Attention or hybrid models}} \\
    Transformer (SWA-hybrid) & 19.22 & 13.72 & 48.32 & 70.21 & 56.12 & 55.85 & 69.23 & 33.84 & 25.00 & 39.74 & 59.42 & 50.86 & -- & -- & -- \\
    Mamba-2 (hybrid) & 17.46 & 11.29 & 48.05 & 71.47 & 57.52 & 56.17 & 70.50 & 34.73 & 29.80 & 40.35 & 59.31 & 51.99 & -- & -- & -- \\
    Gated DeltaNet (hybrid) & 16.00 & 10.82 & 48.71 & 70.06 & 57.50 & 56.83 & 70.41 & 35.15 & 30.60 & 40.97 & 60.00 & 52.25 & -- & -- & -- \\
    KDA (hybrid) & 16.01 & 10.66 & 49.21 & 71.06 & 56.89 & 57.77 & 71.59 & 35.07 & 30.00 & 40.53 & 62.03 & 52.68 & -- & -- & --\\
    Mamba-3 (SISO, hybrid) & 15.54 & 10.65 & 49.19 & 71.01 & 58.75 & 57.30 & 70.54 & 36.35 & 32.00 & 41.20 & 57.86 & 52.69 & -- & -- & -- \\
    Mamba-3 (MIMO, hybrid) & 15.81 & 10.92 & 49.82 & 71.98 & 58.19 & 57.06 & 70.54 & \textbf{38.48} & 29.40 & 40.99 & 57.98 & 52.72 & -- & -- & -- \\
    Gated DeltaNet-2 (hybrid) & 15.62 & 10.43 & 50.90 & 72.20 & 58.46 & 58.56 & 71.89 & 36.69 & \textbf{33.00} & 41.50 & 62.57 & 53.97 & -- & -- & -- \\
    KDA + attn 3:1 (ours) & \textbf{15.04} & 9.93 & 51.60 & \textbf{73.88} & 59.57 & \textbf{60.62} & \textbf{72.69} & 37.20 & 27.00 & 42.43 & 60.31 & 53.92 & \textbf{42.69} & \textbf{64.81} & \textbf{64.79}  \\
    CKDA + attn 3:1 (ours) & 15.34 & \textbf{9.90} & \textbf{52.57} & 73.56 & \textbf{59.68} & 60.14 & 72.56 & 36.86 & 28.20 & \textbf{42.58} & \textbf{63.15} & \textbf{54.37} & 37.77 & 63.91 & 58.89 \\
    \bottomrule
  \end{tabular}
  }
\end{table*}

\begin{figure}
\centering
\includegraphics[width=1.0\linewidth, trim={3.5 3 0 0}, clip=true]{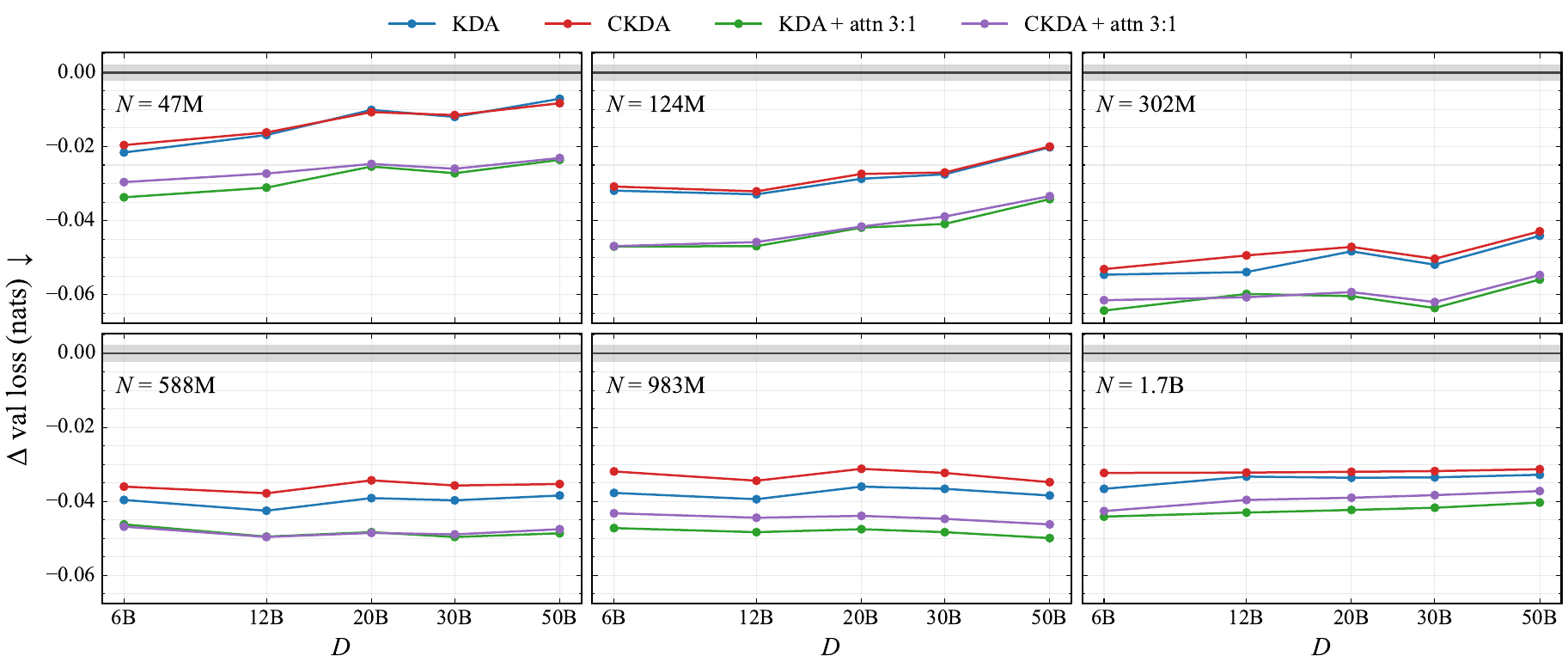}
\vspace{-9mm}
\caption{Advantage of CKDA recurrent and hybrid variants against the Transformer baseline (including QK-Norm, similar to \citet{ajroldi_deriving_2026}) in nats across scales. Grey area is the approximate noise floor. See Appendix~\ref{app:scaling-results} for a detailed scaling analysis of these numbers.}
\label{fig:scaling-advantage}
\end{figure}

\textbf{3. Language modeling.} We first train small language models of 340 million parameters on 15 billion tokens from the Nemotron-CC~\citep{su2025nemotron} dataset, using CKDA as the transformer token mixer along with softmax/dense attention, DeltaProduct, GDN, and KDA as baselines. For architectures that introduce more projections, we reduce the latent dimension of the SwiGLU MLP by 256 to ensure comparable parameter counts. Detailed configurations are provided in Table~\ref{tab:arch_tab}. In the Nemotron-CC block of Table~\ref{tab:lm_res}, CKDA with $\valpha\in[-1,1]$, $\beta\in[0,2]$, and spread initialization achieves the highest observed average downstream accuracy: $52.30\%$, versus $51.32\%$ for standard KDA.
It does not have the lowest validation perplexity. In the small-model FineWeb ablation at 45B tokens, CKDA with gate and $\beta$ spread achieves the highest observed average downstream accuracy of $52.21\%$, compared with $51.85\%$ for KDA without SiLU on the keys (\Cref{tab:lm_res}); the differences are modest and configuration-dependent. Validation curves are provided in Figure~\ref{fig:val_curves}. Separately, we train 1.3B-parameter models on 100B tokens of FineWeb-Edu~\citep{lozhkov2024fineweb-edu}, following the training recipe of \citet{hatamizadeh2026gdn2}. Our recurrent and hybrid CKDA models achieve similar average downstream accuracy to the corresponding KDA controls and the reported Mamba-3~\citep{lahoti2026mamba} and GDN-2~\citep{hatamizadeh2026gdn2} baselines (Table~\ref{tab:fwedu-1p3b-gdn2}). The hybrid recipes differ: ours use full gated attention~\citep{qiu_gated_2025} without RoPE~\citep{kazemnejad_impact_2023} at a $3:1$ recurrent-to-attention ratio, following \citet{team2026kimi}, whereas the reported hybrid baselines use a $1:1$ ratio with sliding-window attention~\citep{fg2gdn2026}. Our KDA controls also use the safe sigmoid gate of \citet{team2026kimi}, which may contribute to the differences from the KDA results as reported in~\citet{hatamizadeh2026gdn2}.
Comparing architecture scaling laws trained on Nemotron-CC, CKDA (and KDA) outperform a Transformer baseline across scales~\citep{ajroldi_deriving_2026}, see Figure~\ref{fig:scaling-advantage}. We don't see a crossing point favoring Transformers going to larger over-training ratios (data to model parameter ratio) as observed for xLSTM in~\citet{beck_xlstm_2026}, although small model behavior indicates a slight advantage reduction going to such regimes.

\begin{figure}
    \centering
    \includegraphics[width=0.75\linewidth]{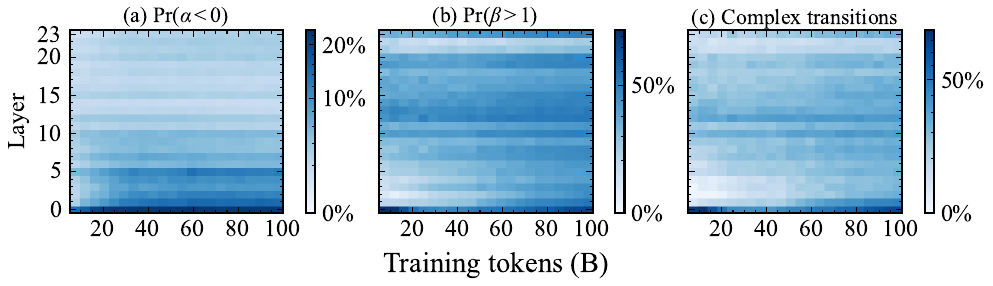}
    \vspace{-4mm}
    \caption{Extended-range use in non-hybrid CKDA 1.3B with standard initialization: per-layer fractions of negative gates (a), rates $\beta>1$ (b), and transitions with complex eigenvalues (c). All three emerge during training. Negative gates occur most frequently in the early layers while $\beta>1$ occurs in every layer, but particularly in the first and second. Colors use square-root scaling. Hybrid CKDA results in~\Cref{fig:signed_kda_interpretability_llm_hybrid}, statistics over the final checkpoint across all layers in~\Cref{fig:final_layer_stats_llm}.}
    \label{fig:signed_kda_interpretability_llm}
\end{figure}
To test whether the mechanism of~\Cref{sec:free} also emerges in language modeling, we track the fraction of negative gate entries, fraction of $\beta>1$, and of complex transitions in every layer of CKDA during training of the 1.3B parameter models (\Cref{fig:signed_kda_interpretability_llm}). The results are on our validation split of FineWeb-Edu at sequence length 4096. We find that despite the standard initialization keeping the gates close to 1, the model learns to use the negative gates early on during training primarily in the first layers. Similarly, the model  learns to use the extended $\beta$ range, first in the initial layers of the model but then also in deeper layers. The combination leads to complex eigenvalue pairs in the state-transition primarily in the first two layers of the trained model. We leave a mechanistic analysis of how the gate and $\beta$ are being used in the model for future work.

\section{Conclusion}
\vspace{-3mm}
We showed that signed channel-wise gates and $\beta\in[0,2]$ enable planar rotations within a single non-expansive diagonal-plus-rank-one KDA transition. CKDA captures every orthogonal matrix of this form and strengthens state-tracking expressivity. Among tested KDA settings, it extrapolates best on $S_3$, $S_4$, and periodic waveforms, with comparable downstream accuracy at 1.3B parameters and near-baseline throughput. Learned transitions exhibit the predicted mechanism in both state tracking and language modeling.

\textbf{Limitations.}
CKDA remains constrained by its rank-one transition structure: a non-expansive transition with a rank-one DPLR correction supports at most one persistent complex-conjugate eigenvalue pair. Our expressivity results also do not imply learnability. Notably $A_5$ is representable but was not learned from random initialization under our standard setup. In addition, the general WFA construction requires $\beta > 2$, sacrificing guaranteed non-expansiveness, and uses exact arithmetic over a fixed algebraic number field. CKDA offers an alternative route to the expressivity capabilities established for DeltaProduct$_2$ through signed channel-wise gating and a single rank-one correction. This structural distinction does not establish an inherent computational advantage over (Gated) DeltaProduct$_2$, and the DeltaProduct$_2$ baseline achieves competitive throughput in our kernel benchmarks.

\textbf{Future Work.} CKDA preserves the structure of the additive updates, leaving the structure of the input to the recurrence untouched. Future work could combine CKDA with the separate gates of GDN-2~\citep{hatamizadeh2026gdn2} that control the additive term and the forgetting of the previous state. While we show that the extended gate and $\beta$ ranges are learned to be used during training, their functional roles remain unclear.

\section*{AI Use Statement}
The central idea of extending KDA's gate to signed values and combining it with an extended range for the Householder coefficient originated with the authors independently of AI assistance after reading the Kimi K3 report. This included the hypothesis that the combination could produce complex eigenvalues, which was subsequently developed with assistance from Claude Opus 4.8. The change of variables described in Appendix~\ref{app:kernels}, which supports signed gates without modifying the existing KDA kernels, was proposed by Claude Opus 5. The authors contextualized and implemented this idea in relation to similar changes of variables used by other models. The integration into the chunk-wise parallel form was proposed by the authors and implemented by the same model. The TileLang kernels were modified using ChatGPT Astra (medium) and then checked through consistency with the naive Python recurrence and the state-tracking experiments by the authors. The orthogonal DPLR characterization (\Cref{prop:orthogonal-dplr}) and norm-constrained DPLR result (\Cref{thm:dplr-persistent-pairs-app}) were developed with assistance from ChatGPT 5.6 Sol (High). ChatGPT 5.6 Sol and 6 Astra produced the first complete versions of the proofs for the single-layer group-word-problem expressivity results (\Cref{thm:main-positive,thm:s5-impossibility}; proofs in Appendix~\ref{app:finite-group-expressivity}). The authors subsequently checked, corrected, and substantially edited these proofs. The adaptation in Appendix~\ref{app:wfamin} showing that three layers suffice to recognize general weighted finite automata, based on~\citep{merrill_why_2026}, was proposed by ChatGPT 5.6 Sol and edited with assistance from Claude Opus 5 and ChatGPT 6 Astra. Claude Opus 5 was used to assist in implementation of language modeling adaptations, implementations, and experimentation. The authors reviewed all AI-assisted work and take responsibility for the final content.

\section*{Acknowledgments}
We would like to thank Alicia Curth for extensive feedback over the whole duration of the project on the manuscript and figures. We are also grateful to Jan Tönshoff, Simon Schrodi, Baohe Zhang, and Adrian Barfuß who provided constructive feedback throughout this project.  Frank Hutter acknowledges financial support by the Hector Foundation. The authors acknowledge support from ELLIS and MPI-IS Tübingen. 
Jaisidh Singh is supported by the Konrad Zuse School of Excellence in Learning and Intelligent Systems (\href{https://eliza.school/}{ELIZA}) through the DAAD programme Konrad Zuse Schools of Excellence in Artificial Intelligence, sponsored by the Federal Ministry of Education and Research. Arber Zela and Volkan Cevher were funded by the Swiss National Science Foundation (SNSF) under grant number 2000-1-240094.
This research was partially supported by the European Commission under the grant No. 101195233 (OpenEuroLLM).
We acknowledge EuroHPC Joint Undertaking for awarding us access to Leonardo at CINECA, Italy, JUWELS and JUPITER at JSC, Germany, Deucalion at MACC, Portugal, MareNostrum5 at BSC, Spain, and the DLC2 Cluster at University of Freiburg, Germany. Views and opinions expressed are however those of the author(s) only and do not necessarily reflect those of the European Union or the ERC. Neither the European Union nor the ERC can be held responsible for them.
\begin{figure}[h]
\begin{center}
\includegraphics[width=0.16\textwidth,trim={10 0 5 0}, clip=true]{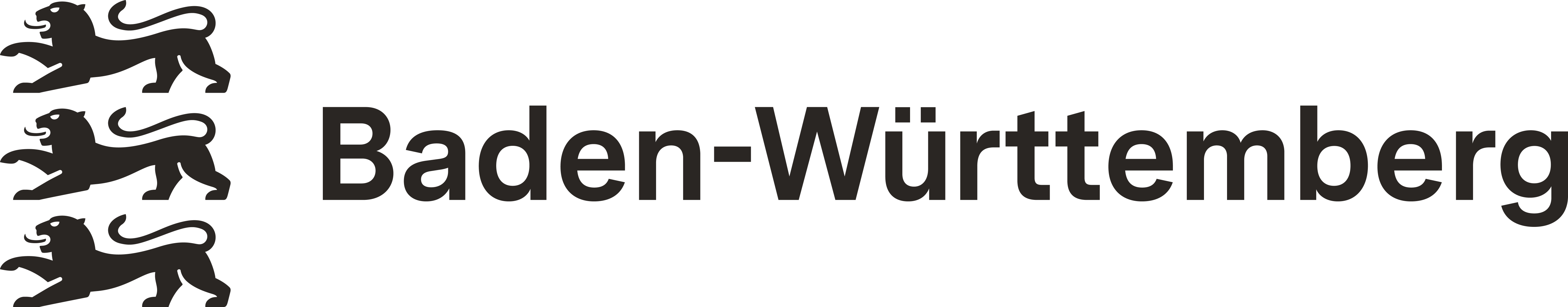} ~~~ \includegraphics[width=0.16\textwidth,trim={10 0 4 0}, clip=true]{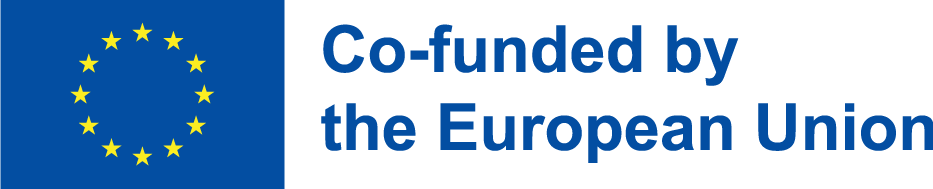} 
\end{center}
\end{figure}

\section*{Author Contributions}
Julien Siems initiated and led the project, proposed combining signed channel-wise gating with an extended Householder range to enable complex eigenvalues, implemented CKDA, developed the initial state-tracking experiments, conducted preliminary language modeling experiments, iterated on kernel changes using Codex, and coordinated the collaboration. Julien Siems and Riccardo Grazzi developed the core theoretical intuition connecting signed gating, Householder transformations, rotations, and DPLR dynamics. Riccardo Grazzi led the theoretical contributions, identified the $A_5$ construction, and verified most theoretical results. Riccardo Grazzi and Korbinian Pöppel led the formalization, verification, and revision of the finite-group and multilayer expressivity results. Arber Zela conducted the state-tracking experiments, building on Julien Siems's initial versions, and contributed to the empirical evaluation and ablations. Jaisidh Singh conducted the small-scale language-modeling experiments. Timur Carstensen conducted the interpretability and additional language-modeling experiments and extended the chunk-wise parallel Triton kernels to include the CKDA gate. Korbinian Pöppel proposed the audio experiments, which Julien Siems conceptualized and carried out. Korbinian Pöppel conducted all large-scale experiments and scaling-law analyses. Aaron Klein enabled the large-scale language-modeling experiments, provided project supervision and crucial early support for the broader empirical study which catalyzed the start of the project. Antonio Orvieto contributed to the theoretical framing and characterization of the transition dynamics, suggested language model evaluations, and provided extensive manuscript feedback. Jenia Jitsev provided feedback on the scaling law analysis and manuscript feedback. Volkan Cevher and Frank Hutter supervised the project and reviewed the manuscript.

\bibliography{refs}
\bibliographystyle{iclr2026_conference}

\appendix
\newpage
\newpage
\definecolor{paperblue}{RGB}{200,220,255}
\tikzset{muted/.style={gray,opacity=0.5}}

\section*{\textbf{Supplementary Material}}
The supplementary material is structured as follows:
\begin{itemize}
    \item Section~\ref{app:canonical-skda} characterizes orthogonal DPLR matrices, analyzes the spectrum of CKDA, bounds persistent rotations, and studies products of CKDA transitions.
    \item Section~\ref{app:state-tracking-preliminaries} defines the state-tracking tasks, recurrent model, realization and tracking conventions, and exact algebraic datatypes and precision assumptions.
    \item Section~\ref{app:finite-group-expressivity} proves the single-layer finite-group results: the planar and cube realizations, the three-dimensional obstruction and four-dimensional tracker for $A_5$, the affine and independent-head reductions, and the $S_5$ spectral obstruction.
    \item Section~\ref{app:wfamin} explains the clock--buffer--accumulator mechanism, proves the three-layer automata result, and gives the factor counts and comparison-table details.
    \item Section~\ref{app:gauge} derives the sign transformation that enables reuse of existing KDA kernels and reports the fused implementation's throughput.
    \item Section~\ref{app:silu} discusses how the SiLU key activation affects learning group-word problems.
    \item Section~\ref{app:exp} gives experimental details for state tracking, periodic waveform continuation, and language modeling.
    \item Section~\ref{app:scaling-results} describes the scaling-law methodology and results.
    \item Section~\ref{appsec:exp_res} presents additional state-tracking baselines and language-modeling validation curves.
\end{itemize}

Our code is open-source on GitHub \url{https://github.com/OpenEuroLLM/ComplexKDA} and is based on Flash-Linear-Attention~\citep{yang2024fla}.

\textbf{Notation.}
Matrices and vectors are denoted by bold uppercase and lowercase letters,
respectively. We write $\mathbb Q$, $\R$, and $\C$ for the rational, real,
and complex numbers. The matrix $\mI_d$ is the $d\times d$ identity,
with the subscript omitted when its size is clear, and $\ve_i$ is the
$i$th standard coordinate vector. For a scalar $z$, $|z|$ denotes its
absolute value or complex modulus; for a finite set $X$, $|X|$ denotes
its cardinality. We use $\overline z$ or $z^*$ for complex conjugation,
$\mA^\top$ for transpose, and $\mA^*$ for conjugate transpose.

For a vector $\vv$, $\Diag(\vv)$ is the diagonal matrix with entries
$\vv$, and $\odot$ denotes the element-wise (Hadamard) product.
The direct sum $\mA\oplus\mB$ places $\mA$ and $\mB$ on the diagonal
of a block-diagonal matrix. We write $\ker\mA$, $\operatorname{im}\mA$,
and $\rank\mA$ for the kernel, image, and rank of $\mA$.
The notation $\operatorname{span}\{\vv_1,\ldots,\vv_k\}$ denotes the
linear span of the indicated vectors, and $V^\perp$ denotes the orthogonal
complement of a subspace $V$. For $\vv\in\R^n$, we similarly write
$\vv^\perp:=\{\vx\in\R^n:\vv^\top\vx=0\}$; when $\vv\ne0$,
this is a hyperplane through the origin. For a subspace $V$, $\mA|_V$
denotes the restriction of $\mA$ to vectors in $V$. If $V$ is invariant
under $\mA$, meaning $\mA V\subseteq V$, this restriction is a linear
map from $V$ to itself.

The vector norm $\|\vv\|=\|\vv\|_2$ is Euclidean, while
$\|\mA\|=\|\mA\|_2$ is the induced operator norm. The Frobenius norm
is $\|\mA\|_F=(\sum_{i,j}|A_{ij}|^2)^{1/2}$.
For a real symmetric matrix, $\mA\succeq0$ and $\mA\succ0$ mean
positive semidefinite and positive definite, respectively.
We write $\spec(\mA)$ for the spectrum (the set of eigenvalues),
$\rho(\mA)$ for the spectral radius, and $\operatorname{tr}(\mA)$
and $\det(\mA)$ for the trace and determinant.

For an alphabet $\Sigma$, $\Sigma^*$ is the set of all finite words,
including the empty word. Transitions act on column states, so the
product $\mA_t\cdots\mA_1$ applies $\mA_1$ first.

\begingroup

\newcommand{\spectyes}{\textcolor{good}{$\checkmark$}\,}
\newcommand{\spectno}{\textcolor{bad}{$\times$}\,}
\newcommand{\spectours}[1]{\textcolor{acc}{\textbf{(Thm.~\ref{#1})}}}
\newcommand{\spectboundary}{\textsuperscript{\ensuremath{\dagger}}}
\newcommand{\spectrowrule}{\specialrule{0.35pt}{1.5pt}{1.5pt}}
\providecommand{\mLambda}{\boldsymbol{\Lambda}}

\begin{table*}[t]
\centering
\hypersetup{hidelinks}

\caption{\textbf{Spectral structure of complex-capable delta recurrences}}
\label{tab:recurrence-spectrum}

\fontsize{8.2}{9.2}\selectfont
\setlength{\tabcolsep}{1.5pt}
\renewcommand{\arraystretch}{1.10}
\renewcommand{\tabularxcolumn}[1]{m{#1}}

\begin{tabularx}{\linewidth}{@{}
>{\raggedright\arraybackslash}m{.20\linewidth}
>{\hsize=.96\hsize\linewidth=\hsize\centering\arraybackslash}X
>{\hsize=1.02\hsize\linewidth=\hsize\centering\arraybackslash}X
>{\hsize=1.02\hsize\linewidth=\hsize\centering\arraybackslash}X@{}}

\toprule
& \textcolor{acc}{\textbf{CKDA}}
& \textbf{MDN}\newline (quadrant-constrained)
& \textbf{SFDA}\newline (default complex)
\\

\spectrowrule

Transition
& $(\mI-\beta\vk\vk^\top)D(\valpha)$
& $\left[\begin{smallmatrix}
\alpha(\mI-\beta\eta\vk\vk^\top)&-\beta\mu\mI\\
\alpha\eta\vk\vk^\top&\mu\mI
\end{smallmatrix}\right]$
& $(\mI-\beta\vk\vk^*)\mLambda$
\\

\midrule

Recurrence order
& First order
& First order in $(\mS,\mM)$;\newline
  generically second order in $\mS$
& First order over $\mathbb C$
\\

\spectrowrule

State per\newline value column
& $d$ real
& $2d$ real
& $m$ complex\newline $\simeq 2m$ real
\\

\spectrowrule

Ranges\newline compared
& $\beta\in[0,2]$\newline
  $\valpha\in[-1,1]^d$
& $\alpha\in(0,1)$, $\mu\in[e^{-1},1)$\newline
  $0\leq\beta\leq1-\alpha$\newline
  $\eta\in(0,2)$
& $\beta\in[0,1]$\newline
  $\valpha\in[0,1]^m$\newline
  $\vtheta\in[-\theta_{\max},\theta_{\max}]^m$
\\

\midrule

Non-real\newline conjugate pairs
& \cellcolor{acc!7}
  $\boldsymbol{\leq1}$\newline
  \spectours{thm:complex_eigenvalues}
& $\leq1$
& Up to $m$\newline after realification
\\

\spectrowrule

Origin of\newline complex modes
& \cellcolor{acc!7}
  Signed gate--delta\newline interaction
& Memory--momentum\newline coupling
& Explicit phases\newline plus delta correction
\\

\spectrowrule

Non-real\newline $|\lambda|=1$ modes\spectboundary
& \cellcolor{acc!7}
  \spectyes Any angle in a\newline two-dimensional block
& \spectno None under\newline the stated constraints
& \spectyes Undamped\newline phase modes
\\

\spectrowrule

Unit-circle\newline support\spectboundary
& \cellcolor{acc!7}
  \textbf{Full circle}
& Only $\{+1\}$
& $\{e^{\pm i\theta}:|\theta|\leq\theta_{\max}\}$;\newline
  full circle if $\theta_{\max}\geq\pi$
\\

\midrule

Whole-state\newline isometries\spectboundary
& \cellcolor{acc!7}
  Orthogonal DPR1;\newline need not commute
& None in the interior;\newline
  only $\mI$ in the closure
& Diagonal unit phases\newline
  ($\beta=0$, $\valpha=\mathbf1$);\newline
  commute
\\

\spectrowrule

Non-expansive\newline by construction
& \spectyes $\|\mA\|_2\leq1$
& \spectno Not guaranteed in the\newline
  augmented Euclidean norm
& \spectyes $\|\mA\|_2\leq1$
\\

\bottomrule
\end{tabularx}

\par\vspace{4pt}

\parbox{\linewidth}{\raggedright
Time subscripts are suppressed; keys have unit norm, $d\geq2$,
$m\geq1$, and $\theta_{\max}>0$.
$D(\valpha)=\Diag(\valpha)$ and
$\mLambda=\Diag(\alpha_j e^{i\theta_j})$.
Dimensions and pair counts concern the transition acting on one value
column, not the vectorized whole memory.
\spectboundary\ Unit-circle support and isometry entries include the
boundaries/closures of the stated parameter families; sigmoid/tanh
parameterizations need not attain these endpoints at finite logits.
Isometry means $\mA^*\mA=\mI$ on the entire state space.
MDN uses the constraint family in its Sec.~3.2; its experiments also
report other momentum floors.
SFDA refers specifically to its complex Eq.~(13): its exact realification
has rotation--decay blocks plus a correction of real rank at most two.
Its isometry statement does not apply to an extension with $\beta=2$ or
to enlarged reflection/control families, and does not assert that general
SFDA transitions commute.
All spectral statements are single-step.
}

\end{table*}
\endgroup

\section{Characterization of a CKDA transition}\label{app:canonical-skda}
\subsection{Orthogonal DPLR matrices reduce to signed Householders}\label{app:orthogonal_dplr}
Rank-one perturbations of isometries and low-rank representations of orthogonal transformations have a classical literature; see, e.g.,~\citet{scherk1950decomposition,nakamura1986one,schreiber1988block,sun1995basis}. The following result gives a more specific rigidity statement for diagonal-plus-rank-one matrices: orthogonality forces the diagonal factor to reduce, up to the rank-one correction, to a sign matrix.

\begin{theorem}[Characterization of orthogonal DPLR matrices; restatement of Theorem~\ref{prop:orthogonal-dplr}]
\label{prop:orthogonal-dplr-app}
Let $\mA=\mD+\vu\vv^\top\in\mathbb R^{n\times n}$, where $\mD=\Diag(d_1,\ldots,d_n)$, and suppose $\mA^\top\mA=\mI$. Then there exist a unit vector $\vk\in\mathbb R^n$ and a diagonal sign matrix $\mS=\Diag(s_1,\ldots,s_n)$, with $s_i\in\{-1,+1\}$, such that
\[
    \boxed{\mA=(\mI-2\vk\vk^\top)\mS.}
\]
Consequently, $\mA$ has at most one non-real conjugate pair of eigenvalues.
\end{theorem}

\begin{proof}
If $\mA$ is diagonal, orthogonality makes it a diagonal sign matrix. Take $\vk=\ve_1$ and let $\mH=\mI-2\ve_1\ve_1^\top$, which flips the first coordinate. Then $\mS:=\mH\mA$ is also a diagonal sign matrix, and $\mH^2=\mI$ gives $\mA=\mH\mS$. Its eigenvalues are already $\pm1$. 

If instead $\mA$ is not diagonal, so $\vu,\vv\neq0$,
we seek column sign flips $\mS$ that turn $\mA\mS$ into a reflection. Since reflections are symmetric, we first ask whether these flips can make the rank-one term symmetric. Orthogonality gives $\mA^\top\mA=\mA\mA^\top=\mI$, so the $i$th column and row both have squared norm one:
\[
\begin{aligned}
    1&=\|d_i\ve_i+v_i\vu\|^2=d_i^2+2d_i u_i v_i+\|\vu\|^2v_i^2,\\
    1&=\|d_i\ve_i+u_i\vv\|^2=d_i^2+2d_i u_i v_i+\|\vv\|^2u_i^2.
\end{aligned}
\]
Subtracting the two equations yields $\|\vu\|^2v_i^2=\|\vv\|^2u_i^2$, hence $|v_i|=c|u_i|$ with $c:=\|\vv\|/\|\vu\|>0$. Thus $u_i=0$ exactly when $v_i=0$; on these coordinates, the preceding equations also give $d_i^2=1$.

We can therefore choose signs so that $\mS\vv=-c\vu$, making the rank-one term a negative multiple of $\vu\vu^\top$, as in a Householder reflection. Explicitly, take $\mS=\Diag(s_1,\ldots,s_n)$ with
\[
    s_i=\begin{cases}
        -cu_i/v_i, & u_i\neq0,\\
        d_i, & u_i=0.
    \end{cases}
\]
Each $s_i$ is $\pm1$. Consequently, $\mQ:=\mA\mS=\mD\mS+\vu(\mS\vv)^\top=\mD\mS-c\vu\vu^\top$ is symmetric and orthogonal, hence $\mQ^2=\mI$.

A reflection moves every vector along its normal direction. To find that direction, consider the displacement $\vr=\vx-\mQ\vx$ of any vector $\vx$. Since $\mQ\vr=\mQ\vx-\mQ^2\vx=\mQ\vx-\vx=-\vr$, every displacement is reversed by $\mQ$. Substituting the formula for $\mQ$ gives $\mD\mS\vr-c\vu(\vu^\top\vr)=-\vr$, or coordinatewise, $(1+d_i s_i)r_i=cu_i(\vu^\top\vr)$.

We can solve these equations coordinatewise. Indeed, $|Q_{ii}|\le1$ because $\mQ$ is orthogonal, so for $u_i\neq0$ we have $1+d_i s_i=1+Q_{ii}+cu_i^2>0$. For $u_i=0$, the same denominator equals $1+d_i^2=2$. Thus $\vr=c(\vu^\top\vr)\vw$, where $w_i:=u_i/(1+d_i s_i)$. Every displacement therefore lies along the same nonzero vector $\vw$.

Since $\mA$ is not diagonal, $\mQ\neq\mI$; otherwise $\mA=\mS$. Hence some displacement is a nonzero multiple of $\vw$, and $\mQ\vr=-\vr$ implies $\mQ\vw=-\vw$. For arbitrary $\vx$, write $\vx-\mQ\vx=t\vw$. Symmetry determines the coefficient:
\[
    t\|\vw\|^2=\vw^\top(\vx-\mQ\vx)=\vw^\top\vx-(\mQ\vw)^\top\vx=2\vw^\top\vx.
\]
Therefore $\mQ\vx=\vx-2\vw(\vw^\top\vx)/\|\vw\|^2$, so $\mQ=\mI-2\vk\vk^\top$ with $\vk:=\vw/\|\vw\|$. Since $\mS^2=\mI$, we obtain $\mA=\mQ\mS=(\mI-2\vk\vk^\top)\mS$.

Finally, $\mA\vx=\mS\vx-2\vk(\mS\vk)^\top\vx$. The correction singles out $\vk$ and $\mS\vk$, so consider $\mathcal U:=\operatorname{span}\{\vk,\mS\vk\}$. The matrix $\mS$ exchanges these two vectors and is orthogonal, so it preserves both $\mathcal U$ and $\mathcal U^\perp$. The formula for $\mA\vx$ shows that $\mA$ also preserves both spaces and agrees with $\mS$ on $\mathcal U^\perp$. Its eigenvalues there are therefore $\pm1$. All non-real eigenvalues lie in the restriction to $\mathcal U$, whose dimension is at most two, allowing at most one conjugate pair.
\end{proof}
\emph{Remark.} The characterization above can also be derived from \citep[Proposition~3.1]{Ionascu2001} which characterizes normal rank-one perturbations of normal operators, specialized to the finite-dimensional real orthogonal case; here we state it explicitly in the signed-Householder form relevant to CKDA.
\subsection{The Sign-Magnitude Decomposition of CKDA}\label{app:spectrum_kda}

\begin{proposition}[Expanded Proposition~\ref{prop:CKDA-canonical}]
\label{prop:CKDA-canonical-app}
Let \(\|\vk\|=1\), $\beta \in [0,2]$, and $\alpha_i \in [-1,1]$. Define

$$
\mA=(\mI-\beta\vk\vk^\top)\Diag(\valpha)
=\widetilde{\mA}\Diag(|\valpha|),
\qquad
\widetilde{\mA}:=(\mI-\beta\vk\vk^\top)\mS,
\qquad
\mS:=\Diag(\operatorname{sign}\valpha),
$$
with $\operatorname{sign}(0)=1$.
Let
\(E_-:=\ker(\mS+\mI)\) be the space spanned by the coordinate vectors whose signs are flipped by $\mS$, with dimension \(m_-\), and
\(E_+:=\ker(\mS-\mI)\) with dimension \(m_+\), and write

$$
\vk=\vk_-+\vk_+,\qquad
\vk_-\in E_-,\quad \vk_+\in E_+,\qquad
\|\vk_-\|=\cos\theta,\quad \|\vk_+\|=\sin\theta,
$$

where \(\theta\in[0,\pi/2]\) is the angle of \(\vk\) from \(E_-\).
When nonzero, let
$
\ve_-:=\vk_- / \|\vk_-\|, \,
\ve_+:= \vk_+ / \|\vk_+\| ,
$
and let \(\mP\) be an orthogonal matrix beginning with the relevant
vectors \(\ve_-,\ve_+\) and then completed by orthonormal bases of \(E_-\) first and
then \(E_+\).  Then

$$
\boxed{
\begin{aligned}
0<\theta<\frac{\pi}{2}:\qquad
\widetilde{\mA}
&=\mP\!\left[
\mB(\beta,\theta)\oplus(-\mI_{m_--1})\oplus\mI_{m_+-1}
\right]\!\mP^\top,\\
\theta=0:\qquad
\widetilde{\mA}
&=\mP\!\left[
(\beta-1)\oplus(-\mI_{m_--1})\oplus\mI_{m_+}
\right]\!\mP^\top,\\
\theta=\frac{\pi}{2}:\qquad
\widetilde{\mA}
&=\mP\!\left[
(1-\beta)\oplus(-\mI_{m_-})\oplus\mI_{m_+-1}
\right]\!\mP^\top .
\end{aligned}}
$$
Here \(\oplus\) denotes the direct sum of matrix blocks, which places the two operands as the blocks of a block-diagonal matrix. The only nontrivial block is

$$
\mB(\beta,\theta)=
\begin{pmatrix}
\beta\cos^2\theta-1 &-\beta\sin\theta\cos\theta\\
\beta\sin\theta\cos\theta&1-\beta\sin^2\theta
\end{pmatrix},
\qquad
\mB(2,\theta)=
\begin{pmatrix}
\cos2\theta&-\sin2\theta\\
\sin2\theta&\cos2\theta
\end{pmatrix}.
$$
More specifically, \(\widetilde{\mA}\) has one non-real conjugate pair exactly when
the eigenvalues of \(\mB(\beta,\theta)\) are non-real, i.e. when
$
\beta^2\cos^2(2\theta)<4(\beta-1),
$ which only holds when $\beta > 1$. When $\beta =2$, $\mB(\beta,\theta)$ becomes a pure
rotation by \(2\theta\).

\end{proposition}

\begin{proof}
For \(x\in E_-\cap\vk_-^\perp\) and
\(y\in E_+\cap\vk_+^\perp\),

$$
\widetilde{\mA}x=-x,\qquad
\widetilde{\mA}y=y,
$$

so only the span of the nonzero components of \(\vk\) remains. When both
are nonzero,

$$
\vk=\cos\theta\,\ve_-+\sin\theta\,\ve_+,\qquad
\mS\ve_-=-\ve_-,\qquad
\mS\ve_+=\ve_+,
$$

and evaluating \(\widetilde{\mA}\) on \(\ve_-,\ve_+\) gives
\(\mB(\beta,\theta)\). The endpoint cases follow by setting
\(\vk=\vk_-\) or \(\vk=\vk_+\), and the form of \(\mB(2,\theta)\) follows
from the double-angle identities.
\end{proof}

\begin{figure}
    \centering
    \includegraphics[width=1.0\linewidth]{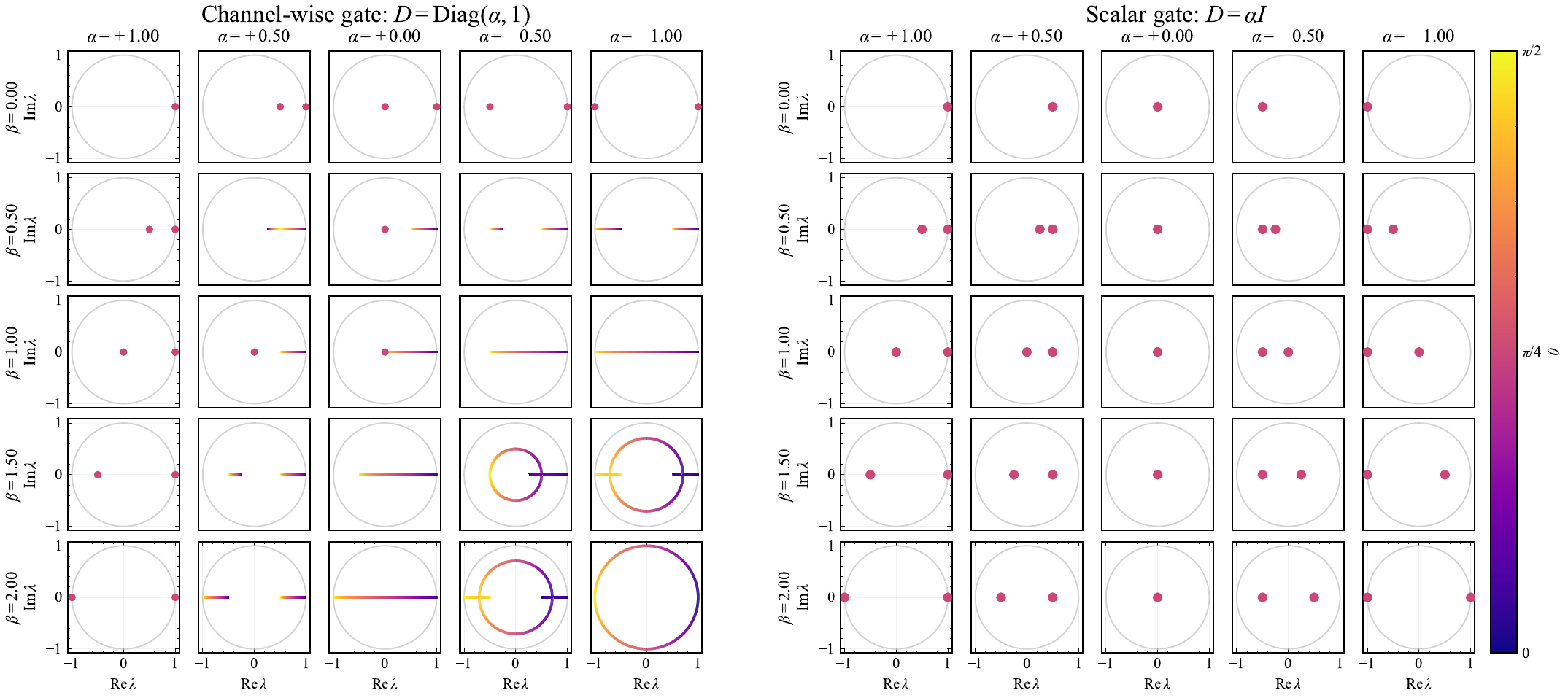}
    \caption{Extension of~\Cref{fig:alpha_sweep_2d} varying both $\beta$ and $\alpha$. The figure demonstrates that both components need to be extended from their standard ranges to obtain complex eigenvalues.}
    \label{fig:beta_alpha_sweep}
\end{figure}
\subsection{The Window for Complex Eigenvalues}\label{app:complex_eigenvalues}

We now consider the eigenvalues of the full CKDA transition. 

\begin{theorem}[Complex Eigenvalues in Diagonal Householder Product]
\label{thm:complex_eigenvalues}
\label{th:gate_plus_householder}
The matrix $\mA=(\mI-\beta\vk\vk^\top)\Diag(\valpha)$ has at most one pair $\lambda,\lambda^*$ of non-real eigenvalues, and such a pair can occur only if $\beta>1$ and $\alpha_i<0$ for at least one $i$.
\end{theorem}

This means it is not sufficient to extend $\beta\in[0,2]$ alone within KDA as in \citet{grazzi2025}, nor is it sufficient to extend only the range of $\alpha_i$ while keeping $\beta\leq1$. This is similar to the requirement shown in~\citep[Prop.~1.3]{siems2025}, where both $\beta$s in a product of two generalized Householders need to be greater than $1$ for the product to obtain complex eigenvalues.

\begin{proof}
Let $\mH^\beta_{\vk}:=\mI-\beta\vk\vk^\top$ and $\mD:=\Diag(\valpha)$. The matrix $\mH^\beta_{\vk}$ has eigenvalues $(1-\beta,1,\ldots,1)$, with $\vk$ as an eigenvector for $1-\beta$ and $\vk^\perp$ as the eigenspace for $1$. Thus we can choose an orthogonal matrix $\mQ$ with $\vk$ as its first column such that
\[
\mQ^\top\mH^\beta_{\vk}\mQ=\mLambda:=\Diag(1-\beta,1,\ldots,1).
\]
Defining $\mB:=\mQ^\top\mD\mQ$, which is symmetric, we obtain
\begin{equation}
\label{eq:char_polynomial_eigenvalues}
\mQ^\top\mA\mQ=\mLambda\mB.
\end{equation}

First suppose $\beta<1$. Then $\mLambda\succ0$, and
\[
\mLambda^{-\sfrac12}(\mLambda\mB)\mLambda^{\sfrac12}=\mLambda^{\sfrac12}\mB\mLambda^{\sfrac12},
\]
which is symmetric. Hence all eigenvalues are real.

For $\beta=1$, write
\[
\mB=\begin{pmatrix} b_{11} & \vb^\top\\ \vb & \mB_{22} \end{pmatrix}.
\]
Then
\[
\mLambda\mB=\begin{pmatrix} 0 & 0\\ \vb & \mB_{22} \end{pmatrix}.
\]
This matrix has eigenvalues $0$ together with the eigenvalues of the symmetric matrix $\mB_{22}$, so its spectrum is again real. Therefore non-real eigenvalues require $\beta>1$.

Now suppose $\beta>1$. Define
\[
\mSigma:=\Diag(\beta-1,1,\ldots,1),\qquad \mJ:=\Diag(-1,1,\ldots,1),
\]
so that $\mLambda=\mJ\mSigma$. Applying the similarity transformation with $\mSigma^{\sfrac12}$ gives
\begin{equation}
\mSigma^{-\sfrac12}(\mLambda\mB)\mSigma^{\sfrac12}
=
\mJ\underbrace{\mSigma^{\sfrac12}\mB\mSigma^{\sfrac12}}_{\mS}
=:\mT,
\end{equation}
where $\mS$ is symmetric. Hence $\mT^*\mJ=\mJ\mT$.

Let $\mT\vx=\lambda\vx$ with $\lambda\notin\R$. Then
\begin{equation}
\vx^*\mJ\mT\vx=\lambda\vx^*\mJ\vx=\vx^*\mS\vx\in\R.
\end{equation}
Since $\vx^*\mJ\vx\in\R$ and $\lambda\notin\R$, this implies $\vx^*\mJ\vx=0$. Similarly, for two eigenvectors $\mT\vx=\lambda\vx$ and $\mT\vy=\mu\vy$,
\begin{equation}
(\mu-\overline{\lambda})\vx^*\mJ\vy=0.
\end{equation}
Thus eigenvectors corresponding to eigenvalues in the upper half-plane are mutually $\mJ$-orthogonal and $\mJ$-isotropic.

The same statement holds for the corresponding generalized eigenspace. Let $\mX$ span the generalized eigenspace associated with all eigenvalues in the upper half-plane, and write $\mT\mX=\mX\mR$. Setting $\mG:=\mX^*\mJ\mX$, the identity $\mT^*\mJ=\mJ\mT$ gives
\[
\mR^*\mG=\mG\mR.
\]
The spectrum of $\mR$ lies in the upper half-plane, while that of $\mR^*$ lies in the lower half-plane. Hence the corresponding Sylvester equation has the unique solution $\mG=0$, so the generalized eigenspace is totally $\mJ$-isotropic.

Since $\mJ=\Diag(-1,1,\ldots,1)$ has only one negative direction, every totally $\mJ$-isotropic subspace has dimension at most one. Indeed, if a vector in such a subspace has first coordinate zero, then
\[
\vx^*\mJ\vx=\sum_{i=2}^n|x_i|^2=0,
\]
and hence $\vx=0$. Projection onto the first coordinate is therefore injective. Thus there is at most one eigenvalue in the upper half-plane, counting algebraic multiplicity, and since $\mA$ is real, at most one corresponding complex-conjugate pair.

Finally, if $\alpha_i\geq0$ for all $i$, then $\mD\succeq0$. The matrices $\mH^\beta_{\vk}\mD$ and
\[
\mD^{\sfrac12}\mH^\beta_{\vk}\mD^{\sfrac12}
\]
have the same characteristic polynomial, while the latter is symmetric. Hence all eigenvalues of $\mA$ are real. Therefore non-real eigenvalues require $\alpha_i<0$ for at least one $i$.
\end{proof}

\subsection{Non-expansive DPLR matrices}\label{app:non_expansive_dplr}

\begin{theorem}[Persistent rotations in non-expansive DPLR]
\label{thm:dplr-persistent-pairs-app}
Let $\mA=\mD+\mR\in\mathbb{R}^{n\times n}$, where $\mD=\mD^\top$ is diagonal and $\operatorname{rank}(\mR)\le r$, and suppose $\norm{\mA}_2\le1$. Then $\mA$ has at most $r$ non-real conjugate pairs of eigenvalues on the unit circle. In particular, a non-expansive rank-one DPLR transition has at most one such pair.
\end{theorem}

\begin{proof}
Since $\mD$ is symmetric, the difference between $\mA$ and its transpose comes entirely from the low-rank term:
\[
    \operatorname{rank}(\mA-\mA^\top)
    =\operatorname{rank}(\mR-\mR^\top)
    \le\operatorname{rank}(\mR)+\operatorname{rank}(\mR^\top)
    \le2r.
\]
We will show that each non-real unit-circle pair contributes two independent vectors to this range.

Let $\mA\vx=\lambda\vx$ with $\vx\in\mathbb C^n\setminus\{0\}$ and $|\lambda|=1$, and write $\vx^*$ for conjugate transpose. This eigenvector loses no norm under $\mA$; we first show that $\mA^\top$ reverses its action. Since $\|\mA^\top\|_2=\|\mA\|_2\le1$,
\[
\begin{aligned}
    \|\mA^\top\vx-\bar\lambda\vx\|^2
    &=\|\mA^\top\vx\|^2+\|\vx\|^2-2\operatorname{Re}(\bar\lambda\vx^*\mA\vx)\\
    &=\|\mA^\top\vx\|^2+\|\vx\|^2-2\operatorname{Re}(\bar\lambda\lambda\|\vx\|^2)\\
    &=\|\mA^\top\vx\|^2-\|\vx\|^2
    \le0.
\end{aligned}
\]
A squared norm cannot be negative, so $\mA^\top\vx=\bar\lambda\vx$. Consequently, $(\mA-\mA^\top)\vx=(\lambda-\bar\lambda)\vx$. If $\lambda$ is non-real, the coefficient is nonzero, and $\vx=(\lambda-\bar\lambda)^{-1}(\mA-\mA^\top)\vx$ belongs to the range of $\mA-\mA^\top$.

To count these vectors with multiplicity, observe that $\mA$ also preserves $\vx^\perp$: whenever $\vx^*\vy=0$, we have $\vx^*\mA\vy=(\mA^\top\vx)^*\vy=\lambda\vx^*\vy=0$. Thus $\operatorname{span}\{\vx\}$ and its orthogonal complement are both invariant, and the restriction to the complement remains a contraction. Splitting off this one-dimensional block and repeating produces independent eigenvectors for every unit-circle eigenvalue, counted with algebraic multiplicity.

If there are $c$ non-real conjugate pairs on the unit circle, we therefore obtain $2c$ independent vectors in the range of $\mA-\mA^\top$. A real matrix has the same rank over $\mathbb R$ and $\mathbb C$, so $2c\le2r$, giving $c\le r$.
\end{proof}

\subsection{Expressivity of products of CKDA transitions}\label{app:CKDA-products}

The diagonal gate reduces the number of factors needed to represent a
non-expansive matrix compared with the $3n$ generalized-Householder construction
of \citet[Prop.~1, item~2]{grazzi2025}.

\begin{proposition}[Expressivity of CKDA products]\label{prop:CKDA-factorization}
Let $n\geq1$ and $\mA\in\mathbb R^{n\times n}$ satisfy $\|\mA\|_2\leq1$.
Then $\mA$ is a product of at most $\max\{1,2n-2\}$ CKDA
transitions, each with a unit key, $\beta\in[0,2]$, and diagonal gate
entries in $[-1,1]$. If $\mA$ is orthogonal, at most
$\max\{1,n-1\}$ factors
suffice, and this orthogonal bound is sharp. In particular, the permutation
matrix of an $n$-cycle cannot be expressed using fewer than $n-1$
CKDA factors.
\end{proposition}

\begin{proof}
The zero matrix is a single zero-gate transition, and for $n=1$ any
scalar $A\in[-1,1]$ is a single diagonal gate with $\beta=0$.
The scalar $A=-1$ also shows sharpness of the orthogonal bound in dimension one.
Henceforth let $n\geq2$ and $\mA\ne0$.

\textit{Orthogonal matrices.}
By Householder QR decomposition, any orthogonal matrix can be written as
$\mQ=\mH_1\cdots\mH_{n-1}\mS$, where $\mS$ is a diagonal sign matrix
and each $\mH_i$ is a reflection or an identity padding factor.
Absorbing $\mS$ into the last factor gives $n-1$ CKDA transitions.

\textit{General contractions.}
Take a singular value decomposition $\mA=\mU\mSigma\mV^\top$, with
$\mSigma=\Diag(\sigma_1,\ldots,\sigma_n)$ and $\sigma_i\in[0,1]$.
Apply Householder QR to both orthogonal factors:
\[
  \mU=\mH_1\cdots\mH_{n-1}\mS_U,
  \qquad
  \mV=\mG_1\cdots\mG_{n-1}\mS_V.
\]
Each $\mH_i$ and $\mG_i$ is a reflection or an identity padding factor,
and $\mS_U,\mS_V$ are diagonal sign matrices. Thus, with
$\mD:=\mS_U\mSigma\mS_V$, whose entries lie in $[-1,1]$,
\[
  \mA=\mH_1\cdots\mH_{n-2}
  (\mH_{n-1}\mD)\mG_{n-1}\cdots\mG_1.
\]
Every reflection is a CKDA transition with $\beta=2$ and identity
 gate, and an identity factor uses $\beta=0$. Absorbing $\mD$ into
$\mH_{n-1}$ gives $(n-2)+1+(n-1)=2n-2$ CKDA factors.

\textit{Sharpness for orthogonal matrices.}
Each CKDA factor is diagonal plus a matrix of rank at most one.
Inductively, a product of $m$ such factors is diagonal plus a matrix of
rank at most $m$: multiplying $\mD+\mR$, with $\rank(\mR)\leq m$, by
$\mE+\vu\vv^\top$ gives the diagonal part $\mE\mD$ and correction
$\mE\mR+\vu\vv^\top(\mD+\mR)$, of rank at most $m+1$.

Now let $\mC_n$ be the cyclic permutation matrix defined by
$\mC_n\ve_i=\ve_{i+1}$ for $i<n$ and $\mC_n\ve_n=\ve_1$.
For any diagonal $\mD=\Diag(d_1,\ldots,d_n)$, the equation
$(\mC_n-\mD)\vx=0$ reads
\[
  x_{i-1}=d_i x_i\quad (i=2,\ldots,n),
  \qquad x_n=d_1x_1.
\]
Starting from $x_n$, the first $n-1$ equations determine all other
coordinates; the last equation can only impose an additional constraint.
Thus $\dim\ker(\mC_n-\mD)\leq1$, so
$\rank(\mC_n-\mD)\geq n-1$ for every diagonal $\mD$.
Recall that a product of $m$ CKDA factors differs from some diagonal
matrix by a correction of rank at most $m$. Thus any such factorization
of $\mC_n$ would give a diagonal $\mD$ satisfying
$n-1\leq\rank(\mC_n-\mD)\leq m$, forcing $m\geq n-1$.
Since $\mC_n$ is orthogonal, the upper bound is attained.
\end{proof}

\emph{Remark.} This bound is also related to the Hermitian-plus-low-rank characterization of~\citet[Theorem~12]{del2019matrix}; the proof above gives a direct contraction-specific rank argument.

\section{State-tracking preliminaries and arithmetic conventions}
\label{app:state-tracking-preliminaries}
This section fixes the tasks, recurrent model, and arithmetic conventions
used in the single-layer results of \Cref{app:finite-group-expressivity}
and the multilayer results of \Cref{app:wfamin}. We follow the
state-tracking setup of \citet{grazzi2025,siems2025}, making explicit the
distinction between tracking a state, recognizing a language, and computing
a weighted-automaton output.

\subsection{State-transition systems and tracking}\label{app:tracking-tasks}

\begin{definition}[State tracking]\label{def:state-tracking}
A state-transition system consists of a finite input alphabet $\Sigma$,
a state space $\mathcal S$, an initial state $s_0\in\mathcal S$, and a
map $T_a:\mathcal S\to\mathcal S$ for each $a\in\Sigma$.
For a word $w=a_1\cdots a_T\in\Sigma^*$, its states satisfy
$s_t=T_{a_t}(s_{t-1})$. For a recurrent model with hidden updates $F_a$
and initial state $h_0$, a hidden state is \emph{reachable} if it is
obtained from $h_0$ by processing some finite input word. This includes
$h_0$ itself, reached by the empty word. The model \emph{tracks} the
system if a fixed decoder $d$ satisfies
\[
 d(h_0)=s_0,\qquad d(F_a(h))=T_a(d(h))
 \quad\text{for every reachable }h\text{ and every }a\in\Sigma.
\]
Thus $d(h_t)=s_t$ at every prefix of every word, including the empty
word. The decoder need not be injective, and $\mathcal S$ need not be finite.
\end{definition}

For each target system, the network dimensions, parameters, initial state,
and decoder are fixed independently of input length. Under the
polynomial-precision convention of \Cref{app:tracking-precision}, each
scalar is represented exactly using a number of bits bounded by a fixed
polynomial in the input length.

\paragraph{Groups and prefix products.}
We write $e$ for a group identity and $|G|$ for the number of elements of
$G$. The group $S_n$ consists of permutations of $n$ objects, and $A_n$
is its subgroup of even permutations, those expressible as an even
number of swaps of two objects. We write $\mathbb Z_n$ for the cyclic
group of $n$ elements and $D_n$ for the dihedral group of $2n$ elements.
The \emph{order} of an element $h$ is the smallest positive integer $m$
with $h^m=e$, when such an integer exists.

We use \emph{cycle notation} for permutations: $(12)$ swaps $1$ and $2$,
while $(1245)$ sends $1\mapsto2\mapsto4\mapsto5\mapsto1$, leaving
unlisted objects fixed. A swap is called a \emph{transposition}.
Disjoint cycles act on separate objects; for example, $(1245)(36)$
also swaps $3$ and $6$. Products are composed from right to left, so
$(12)(23)=(123)$ applies $(23)$ first. Disjoint cycles commute, but
overlapping cycles need not: reversing this product gives
$(23)(12)=(132)$.

For a finite group $G$ with identity $e$, take $\Sigma=\mathcal S=G$,
$s_0=e$, and $T_g(s)=gs$. The desired state is
$s_t=g_t\cdots g_1$. Throughout our group results, inputs range over
\emph{all} of $G$. Restricting inputs to a chosen set of generators gives
a different task. For example, tracking $\mathbb Z_n$ means computing
the running sum modulo $n$. We use ``solving the group-word problem of
$G$'' to mean tracking this system in the sense of
\Cref{def:state-tracking}, that is, recovering every prefix product.

\paragraph{Deterministic finite automata.}
A DFA is a tuple $(Q,\Sigma,\delta,q_0,F)$, where $Q$ is a finite state
set, $q_0\in Q$ is the initial state, $F\subseteq Q$ is the accepting
set, and $\delta:Q\times\Sigma\to Q$ is the transition function.
Indexing coordinates by $Q$, define the column-state transition matrix by
\[
 \mM_a\ve_q=\ve_{\delta(q,a)},\qquad
 \bm p_0=\ve_{q_0},\qquad \bm p_t=\mM_{a_t}\bm p_{t-1}.
\]
Each column of $\mM_a$ is one-hot. Tracking recovers the current state
$q_t$; recognition only asks whether $q_t\in F$. For example, a
two-state parity automaton swaps its states on input $1$, preserves them
on input $0$, and accepts in its even state. The functions $Q\to Q$
induced by words form a finite monoid under composition: composition is
associative and the empty word acts as the identity. If all symbol maps
are bijective, this monoid is a group.

\paragraph{Rational weighted finite automata.}
A WFA over $\mathbb Q$ has a finite alphabet $\Sigma$, an initial vector
$\bm\alpha_0\in\mathbb Q^n$, a matrix $\mM_a\in\mathbb Q^{n\times n}$
for each symbol, and a final vector $\bm\omega\in\mathbb Q^n$.
Its weighted state and scalar output are
\begin{equation}\label{eq:wfa-state-output}
 \bm p_0=\bm\alpha_0,\qquad
 \bm p_t=\mM_{a_t}\bm p_{t-1},\qquad
 y_t=\bm\omega^\top\bm p_t
     =\bm\omega^\top\mM_{a_t}\cdots\mM_{a_1}\bm\alpha_0.
\end{equation}
Unlike a DFA state, $\bm p_t$ can range over an infinite set.
Tracking $\bm p_t$ suffices to compute $y_t$ by applying the final vector
$\bm\omega$. The converse need not hold: distinct weighted states can
have the same scalar output, so computing $y_t$ need not recover
$\bm p_t$.

\subsection{Recurrent model, initialization, and readout}
\label{app:tracking-model}
All matrices are real, and $\mI_d$ denotes the $d\times d$ identity matrix.
A linear RNN head has a matrix-valued state and an affine update
\begin{equation}\label{eq:tracking-affine-model}
 \mS_t=\mA(\vx_t)\mS_{t-1}+\mB(\vx_t),\qquad
 \mS_t\in\mathbb R^{n\times d_v}.
\end{equation}
The functions are position-independent: time dependence comes through
the layer input $\vx_t$. In CKDA,
\[
 \mA(\vx_t)=(\mI-\beta_t\vk_t\vk_t^\top)\Diag(\valpha_t),
 \qquad \mB(\vx_t)=\beta_t\vk_t\vv_t^\top,
 \qquad \|\vk_t\|_2=1.
\]
Unless stated otherwise, $\beta_t\in[0,2]$ and
$\valpha_t\in[-1,1]^n$. These theoretical ranges include their endpoints
and zero gates. The practical gate parameterization and its magnitude
floor are described separately in \Cref{app:gauge}. The regular-language
and WFA simulation in \Cref{app:wfamin} explicitly allows $\beta_t>2$.
The $S_5$ lower bound allows arbitrary additive matrices $\mB(\vx_t)$,
and therefore also covers the CKDA additive term.

\paragraph{Heads and layers.}
Independent heads have separate recurrent states and receive the same
layer input; their updates do not depend on the previous states of other
heads. A joint decoder may use their outputs. Stacked layers may apply
nonlinear feed-forward maps between recurrent layers, and may pass the
current input along with the recurrent readout.

\paragraph{Initial state and decoder.}
We permit a prescribed initial state; in particular, the
group constructions use $\mS_0=\mI_d$ and $\mB(\vx_t)=0$ at every step.
A head is usually read through $\mS_t^\top\vq_t$, followed by a
feed-forward decoder, with $\vq_t$ determined by the current layer input.
For finite orthogonal state sets, \Cref{lem:finite-query-readout} shows
that one fixed query retains all information required by a state decoder.
We allow sufficiently expressive feed-forward maps to implement finite
lookups, as in \citet{grazzi2025,siems2025,merrill_why_2026}.
In the WFA construction, the readout is a linear functional of the
accumulator with coefficients selected from a finite lookup.

\subsection{Realization and decoded tracking}
\label{app:tracking-realization}
\paragraph{Group representations.}
The groups $O(d)=\{\mQ:\mQ^\top\mQ=\mI_d\}$ and
$SO(d)=\{\mQ\in O(d):\det\mQ=1\}$ contain the orthogonal matrices and
the orientation-preserving ones, respectively. The notation
$\langle\mA_h:h\in G\rangle$ denotes the group generated by the indicated
matrices: all finite products of them and their inverses, including $\mI_d$.

\begin{definition}[Homomorphism, isomorphism, and faithful representation]
For groups $G$ and $H$, a map $\phi:G\to H$ is a \emph{homomorphism} if
$\phi(ab)=\phi(a)\phi(b)$ for every $a,b\in G$. It is an
\emph{isomorphism} if it is also bijective; we write $G\cong H$ when such
a map exists. An \emph{orthogonal representation} is a homomorphism
$\rho:G\to O(d)$, and it is \emph{faithful} if it is injective.
\end{definition}

A faithful representation identifies $G$ with the matrix group $\rho(G)$.
For a homomorphism $\phi:G\to H$, its \emph{kernel} is
$\ker\phi=\{a\in G:\phi(a)=e_H\}$, and the \emph{fiber}
$\phi^{-1}(b)=\{a\in G:\phi(a)=b\}$ collects elements with the same image.
The kernel is trivial, $\ker\phi=\{e_G\}$, exactly when $\phi$ is injective.

For a subgroup $N\subset G$, a \emph{left coset} $hN=\{hn:n\in N\}$
is a copy of $N$ obtained by multiplying each element on the left by $h$.
Distinct cosets partition $G$. Their number is the \emph{index} $[G:N]$,
which equals $|G|/|N|$ when $G$ is finite. In particular, index two means
that the two cosets each contain half the elements of $G$.
Every nonempty fiber of $\phi$ is a coset of its kernel:
if $\phi(a_0)=b$, then $\phi^{-1}(b)=a_0\ker\phi$.
Thus each attained value has exactly $|\ker\phi|$ preimages when the
kernel is finite.

A subgroup $N\subset G$ is \emph{normal} if $hNh^{-1}=N$ for every
$h\in G$. Kernels are normal: for $a\in\ker\phi$,
$\phi(hah^{-1})=\phi(h)e_H\phi(h)^{-1}=e_H$.
A surjective homomorphism also maps normal subgroups to normal subgroups
of its codomain. A nontrivial group is \emph{simple} if its only normal
subgroups are $\{e\}$ and the whole group. We use that $A_5$ is a
noncyclic simple group of order $60$~\citep{conrad_ansimple}.

\paragraph{Permitted transitions, realization, and tracking.}
For the orthogonal group constructions we specialize to
\[
 \mS_0=\mI_d,\qquad \mS_t=\mA_{g_t}\mS_{t-1}.
\]
For a unit vector $\vu\in\R^d$, the Householder matrix
$\mH_{\vu}=\mI_d-2\vu\vu^\top$ reflects across the hyperplane
$\vu^\perp$. A \emph{sign diagonal} is a diagonal matrix with entries
in $\{-1,+1\}$. We consider all orthogonal transition matrices of a CKDA layer:
\[
  \SH_d:=\{\mH_{\vu}\mD:\|\vu\|=1,\ \mD\text{ a sign diagonal}\}.
\]
Its elements are \emph{signed-Householder matrices}. This family contains
$\mI_d=\mH_{\ve_1}\mH_{\ve_1}$, since the first coordinate vector
$\ve_1$ gives the sign diagonal $\mH_{\ve_1}=\Diag(-1,1,\ldots,1)$.
The family is not generally closed under multiplication.

A realization represents the accumulated group element directly by a
matrix. Tracking allows several hidden matrices to represent the same group
element, provided a fixed decoder recovers the correct product.

\begin{definition}[Signed-Householder realization and tracking]
\label{def:sh-realization}\label{def:sh-tracking}
For a finite group $G$, choose transitions $\mA_h\in\SH_d$, $h\in G$,
and let $\mathcal R$ be the states reachable from $\mI_d$ under these
updates. They \emph{realize $G$ in dimension $d$} if
$h\mapsto\mA_h$ is a faithful representation, and \emph{track $G$}
if a decoder $p:\mathcal R\to G$ satisfies
\begin{equation}\label{eq:signed-householder-tracking}
  p(\mI_d)=e,\qquad p(\mA_h\mS)=h\,p(\mS)
  \quad(h\in G,\ \mS\in\mathcal R).
\end{equation}
The tracker has \emph{finite reachability}, or is \emph{finite-state},
if $\mathcal R$ is finite (\Cref{def:finite-reachability}).
\end{definition}

For finite-state tracking, the update rule itself forces the decoder to
preserve multiplication. This lets us use group structure even when the
decoder was initially allowed to be an arbitrary function.

\begin{lemma}[Finite tracking induces a group homomorphism]
\label{lem:finite-tracking-homomorphism}
For a finite-state signed-Householder tracker,
$\mathcal R=\Gamma:=\langle\mA_h:h\in G\rangle$, and its decoder
$p:\Gamma\to G$ is a surjective homomorphism. Such a tracker has a
bijective decoder exactly when its transitions form a realization.
\end{lemma}

\begin{proof}
The finite set $\mathcal R$ contains the identity and is closed under
multiplication. For each $\mQ\in\mathcal R$, two powers coincide;
invertibility gives $\mQ^m=\mI_d$ for some $m\geq1$. Hence
$\mQ^{-1}=\mQ^{m-1}\in\mathcal R$, so $\mathcal R=\Gamma$.
Writing $\mQ_1=\mA_{h_t}\cdots\mA_{h_1}$ and iterating the tracking
rule gives $p(\mQ_1\mQ_2)=h_t\cdots h_1p(\mQ_2)=p(\mQ_1)p(\mQ_2)$.
Also $p(\mA_h)=h$, proving surjectivity. If $p$ is bijective, its inverse
is a faithful representation with $p^{-1}(h)=\mA_h$. Conversely, a
realization $\rho$ has reachable set $\rho(G)$ and decoder $\rho^{-1}$.
\end{proof}

\paragraph{Query projection.}
The following lemma relates these matrix-state definitions to the
query readout in \Cref{app:tracking-model}.

\begin{lemma}[A fixed query distinguishes finite orthogonal states]
\label{lem:finite-query-readout}
For every finite subgroup $\Gamma\subset O(d)$, there is a unit query
$\vq\in\R^d$ such that $\mS\mapsto\mS^\top\vq$ is injective on
$\Gamma$. For a finite-state tracker with decoder $p$, the output decoder
$\pi(\mS^\top\vq):=p(\mS)$ is therefore well defined and satisfies
\begin{equation}\label{eq:tracking-quotient}
  \pi(\vq)=e,\qquad
  \pi((\mA_h\mS)^\top\vq)=h\,\pi(\mS^\top\vq).
\end{equation}
Distinct reachable query outputs have positive separation, so the precision
needed to distinguish them is independent of sequence length.
\end{lemma}

\begin{proof}
Choose a nonzero vector $\vq$ outside the fixed-point subspaces of all
nonidentity matrices in $\Gamma$, and normalize it. Each fixed-point subspace
of $\mS\in\Gamma\setminus\{\mI_d\}$ is proper, and a finite union of
proper linear subspaces cannot cover $\R^d$. Let $\mS_1, \mS_2 \in \Gamma$. If
$\mS_1^\top\vq=\mS_2^\top\vq$, then
$\mS_2\mS_1^\top \vq = \vq$. By the choice of $\vq$, the
only matrix in $\Gamma$ that fixes $\vq$ is $\mI_d$. Hence
$\mS_2\mS_1^\top=\mI_d$, which gives $\mS_1=\mS_2$. The output tracking
rule follows from that of $p$. The distinct outputs form a finite set,
which has a positive minimum pairwise distance whenever it has at least
two elements; a singleton needs no distinction.
\end{proof}

The query can be chosen with algebraic coordinates: first choose a
rational vector outside the finitely many excluded proper subspaces,
then normalize it. Positive separation concerns the readout of exact
reachable states; it does not by itself bound errors accumulated in the
recurrence.

\subsection{Precision and finite reachability}
\label{app:tracking-precision}
In practice, floating-point recurrent implementations round arithmetic
results within each update and pass the resulting state to the next step.
A simplified model rounds once per update:
\[
 \widehat{\mS}_t=\operatorname{round}\!\left(
 \mA(\vx_t)\widehat{\mS}_{t-1}+\mB(\vx_t)\right).
\]
Actual implementations may round several times, depending on the
arithmetic and fused operations. This generally makes the update
nonlinear, and need not give the same result as rounding the exact
unrolled state at the end. We therefore need to specify both the
datatype and how arithmetic is evaluated.

\paragraph{Simplifying the arithmetic.}
\citet[Section~2.1 and Appendix~A]{merrill2024} evaluate iterated sums
and products in the unrolled computation exactly before casting to
$O(\log T)$-bit floats. Following this approach,
\citet[Appendices~A.4 and~B]{grazzi2025} use a fixed finite datatype for
their lower bounds. \citet[Sections~2.1 and~2.4]{merrill_why_2026} instead use exact
rational arithmetic with bounds on representation size, preserving
associativity and distributivity.

\paragraph{Our arithmetic convention.}
Our constructive results use \emph{polynomial precision over a fixed
real algebraic number field} $K\subset\mathbb R$. This extends the
rational setting of \citet{merrill_why_2026} to include irrational
constants, such as $\sqrt3/2$ in a rotation through $2\pi/3$ or
$1/\sqrt2$ in a normalized transposition key. Neither fixed-width floats
nor rational numbers with polynomially many bits represent these exactly.
Fix a basis $\eta_1,\ldots,\eta_r$ of $K$ over $\mathbb Q$ and store
\[
 x=\sum_{j=1}^r c_j\eta_j,\qquad c_j\in\mathbb Q.
\]
The numerators and denominators of the $c_j$ have polynomial bit length
in the input length $T$. The field, basis, and network parameters are
fixed independently of $T$. Arithmetic is exact: each $\eta_i\eta_j$
is a fixed rational combination of basis elements, so only the rational
coefficients need updating. Floating-point approximations require a
separate error analysis; a small rotation-angle error can accumulate
over repeated updates. Our exact constructions do not establish
numerical robustness at arbitrary lengths.

\paragraph{Finite reachability.}
A finite target state space does not force the model to visit finitely
many hidden states: different words with the same group product may
lead to different states that decode to that product. Even the DFA
encoding $\vh_t=2^{-t}\ve_{q_t}$ retains the current state while visiting
infinitely many hidden states. We distinguish this from the following
property of the exact recurrence.

\begin{definition}[Finite reachability]\label{def:finite-reachability}
The reachable set of a recurrent model is
$\mathcal R=\{h_0\}\cup\{F_{a_t}\circ\cdots\circ F_{a_1}(h_0):
t\geq1,\ a_i\in\Sigma\}$.
The model has \emph{finite reachability}, or is \emph{finite-state}, if
$\mathcal R$ is finite under its exact updates.
\end{definition}

Our finite-group and DFA constructions use only finitely many exact
values, including intermediate computations. They admit a fixed finite
datatype $\mathbb D\subset K$, possibly containing irrationals, as in
\citet[Appendix~A.4]{grazzi2025} and \citet[Appendix~B]{siems2025}.
It contains all values needed by the construction, without needing to
be closed under arbitrary arithmetic. For the block simulation of
\citet[Appendix~D.2]{peng2025rwkv7} and our adaptation, finiteness follows
from the finite clocks, buffers, DFA states, and partial factor products.
Thus these constructions have finite reachability without rounding.

\paragraph{What the lower bounds assume.}
We assume finite reachability in the $S_5$ obstruction to simplify the
analysis: together with non-expansion, it lets us remove the additive
term and restrict to orthogonal transitions
(\Cref{prop:affine-orthogonal-reduction,prop:independent-head-reduction}).
Finitely many rounded states do not imply finitely many exact states,
so choosing a finite datatype alone does not justify this reduction.

\citet[Thms.~1--2 and Appendix~B.2]{grazzi2025} do not require finite
exact reachability. Under their finite-datatype casting convention,
they rule out tracking $S_2\cong\mathbb Z_2$ (parity) with any fixed
number of layers when every transition has only nonnegative real
eigenvalues. Their single-layer argument also rules out tracking
$\mathbb Z_n$ for every $n>2$ when all transition eigenvalues are real.
On a repeated input, the cast state of one layer eventually becomes
constant in the first case, or alternates between at most two values in
the second, preventing the required counting. Our $S_5$ result allows a
non-real conjugate pair per head,
so the real-spectrum argument does not apply directly. Replacing finite
reachability by a suitable datatype and casting assumption remains open.

\section{Single-layer finite-group expressivity}
\label{app:finite-group-expressivity}
We use the tracking and realization definitions of
\Cref{app:tracking-realization} and the exact arithmetic convention of
\Cref{app:tracking-precision}. Inputs range over every element of the
target group. We first construct planar and cube realizations, then prove
the three-dimensional obstruction and four-dimensional tracker for $A_5$.
Finally, we extend the lower-bound setup to affine updates and independent
heads before proving the spectral obstruction for $S_5$.

\subsection{Axis criterion for three-dimensional rotations}\label{app:axis-test}

A rotation in $\R^3$ fixes its axis and rotates the perpendicular plane.
The following proposition characterizes exactly when such a rotation is a
signed-Householder matrix.  Its proof uses only the elementary fact that
the product of reflections in two planes is a rotation about their line of
intersection.

\begin{proposition}[Axis test]\label{prop:axis-test}\label{prop:axis-test-app}
Let $\mR\in SO(3)$ be a rotation about the axis spanned by a unit vector $\va\in\R^3$.
\begin{enumerate}
  \item The identity and every half-turn belong to $\SH_3$.
  \item If its angle is different from $0$ and $\pi$, then
  $\mR=\mH_{\vu}\mD$ for some sign diagonal $\mD$ if and only if its rotation axis $\va$
  is contained in a coordinate plane.  Explicitly, this means that
  $\va\perp\ve_i$ for some coordinate axis $\R\ve_i$, or equivalently that
  $\va$ has at least one zero coordinate.
\end{enumerate}
\end{proposition}

\begin{proof}
For part 1, the matrix $2\va\va^\top-\mI$ fixes $\va$ and sends every
$\vx\perp\va$ to $-\vx$, so it is the half-turn about $\R\va$.  Moreover,
\[
  2\va\va^\top-\mI=(\mI-2\va\va^\top)(-\mI)=\mH_{\va}(-\mI),
\]
which has the required form. The identity case was established when
introducing the permitted family.

For part 2, suppose $\mR=\mH_{\vu}\mD$.  Since
$\det\mH_{\vu}=-1$ and $\det\mR=1$, the diagonal $\mD$ has either one or
three negative entries.  In the latter case $\mD=-\mI$ and
$\mR=2\vu\vu^\top-\mI$, the half-turn from part 1, contradicting the
assumption on the angle.  Hence $\mD$ has exactly one negative entry.
If that entry is in position $i$, then
$\mD=\mI-2\ve_i\ve_i^\top$, which is precisely the reflection across the
coordinate plane $\ve_i^\perp$.

Both factors are now reflections in planes: $\mD$ reflects in
$\ve_i^\perp$ and $\mH_{\vu}$ reflects in $\vu^\perp$.  Their product
rotates about the line $\ve_i^\perp\cap\vu^\perp$.  This is the axis
$\R\va$, so $\va\perp\ve_i$; in other words, $\va$ lies in the coordinate
plane $\ve_i^\perp$.

Conversely, every rotation in $SO(3)$ is a product of reflections in two
planes through its axis, separated by half the rotation angle.  If
$\va\perp\ve_i$, we may choose one plane to be the coordinate plane
$\ve_i^\perp$.  Its reflection is
$\mD_i=\mI-2\ve_i\ve_i^\top$; the other reflection is therefore
$\mH_{\vu}:=\mR\mD_i$, and $\mR=\mH_{\vu}\mD_i$.
\end{proof}

\subsection{Cyclic and dihedral groups}\label{app:planar-groups}

We write $\mathbb Z_n$ for the cyclic group of $n$ elements and $D_n$ for
its dihedral extension of $2n$ elements, represented by planar rotations and
reflections. The following realizations also give finite-state tracking.

\begin{proposition}[planar groups]\label{prop:planar-groups}
Every cyclic group $\mathbb{Z}_n$ and every dihedral group $D_n$ has a
signed-Householder realization in two dimensions.
\end{proposition}

\begin{proof}
In the plane, take $\mD_0=\Diag(1,-1)$, reflection in the horizontal axis.
If $\mS_\alpha$ denotes reflection in the line making angle $\alpha$ with
the horizontal axis, then $\mS_\alpha\mD_0$ is a rotation through $2\alpha$.
Choosing $\alpha=\pi k/n$ realizes every rotation through $2\pi k/n$, so
every element in the standard representation of $\mathbb{Z}_n$ has the required
form.

The dihedral group adds reflections of the regular $n$-gon.  Each is already
a Householder matrix, so choose $\mD=\mI$.  Thus the same is true for the
standard planar representation of $D_n$.
\end{proof}

\subsection{Realizing \texorpdfstring{$S_4$}{S4} via the cube rotational symmetries}\label{app:cube}

The group $S_4$ has a familiar geometric realization: it is isomorphic to
the group of rotational symmetries of a cube.  We now show that rotating the
cube to a suitable orientation makes all of these rotations admissible at
once. Since restricting a faithful representation to a subgroup
preserves both faithfulness and the permitted matrix form, the same
construction will also realize $A_4\subset S_4$.

Figure~\ref{fig:cube-axis-types} shows the face and body-diagonal axes;
edge axes give only half-turns.

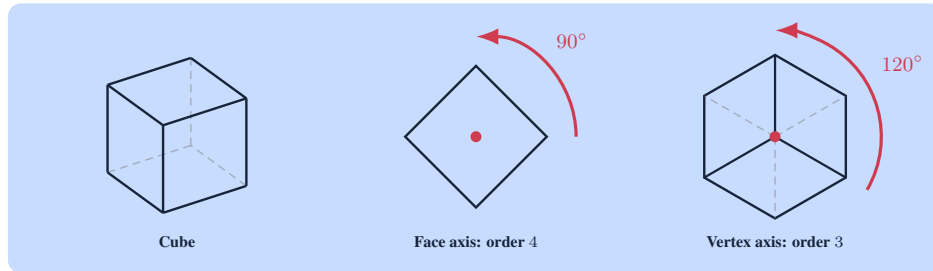
\begin{figure}[!ht]
\centering
\begin{tikzpicture}[>=Latex,font=\small,scale=.86,transform shape]
\fill[paperblue,rounded corners=5pt] (-7.2,-2.0) rectangle (7.2,2.15);

\begin{scope}[shift={(-4.6,.12)},
  x={(.640679cm,.207183cm)},
  y={(-.427119cm,.310775cm)},
  z={(0cm,.673345cm)}]
  \coordinate (ca) at (-1,-1,-1); \coordinate (cb) at (1,-1,-1);
  \coordinate (cc) at (1,1,-1);   \coordinate (cd) at (-1,1,-1);
  \coordinate (ce) at (-1,-1,1);  \coordinate (cf) at (1,-1,1);
  \coordinate (cg) at (1,1,1);    \coordinate (ch) at (-1,1,1);
  \foreach \a/\b in {cb/cc,cc/cd,cc/cg}
    \draw[muted,densely dashed,line width=.55pt] (\a)--(\b);
  \foreach \a/\b in {ca/cb,cd/ca,ce/cf,cf/cg,cg/ch,ch/ce,
                       ca/ce,cb/cf,cd/ch}
    \draw[ink,line width=.9pt] (\a)--(\b);
\end{scope}
\node[text=ink,font=\bfseries\scriptsize] at (-4.6,-1.52) {Cube};

\begin{scope}[shift={(0,.1)}]
  \draw[ink,line width=.9pt] (0,-1.088944)--(1.088944,0)--
    (0,1.088944)--(-1.088944,0)--cycle;
  \fill[bad] (0,0) circle (2.4pt);
  \draw[bad,very thick,->] (0:1.54)
    arc[start angle=0,end angle=90,radius=1.54];
  \node[text=bad,font=\bfseries\footnotesize]
    at (45:2.09) {$90^\circ$};
\end{scope}
\node[text=ink,font=\bfseries\scriptsize] at (0,-1.52) {Face axis: order $4$};

\begin{scope}[shift={(4.6,.1)},x={(1.257405cm,0cm)},y={(0cm,1.257405cm)}]
  \coordinate (cv2) at (0,-1);          \coordinate (cv3) at (-.866025,.5);
  \coordinate (cv4) at (-.866025,-.5);  \coordinate (cv5) at (.866025,.5);
  \coordinate (cv6) at (.866025,-.5);   \coordinate (cv7) at (0,1);
  \draw[ink,line width=.9pt] (cv2)--(cv6)--(cv5)--(cv7)--(cv3)--(cv4)--cycle;
  \draw[muted,densely dashed,line width=.55pt]
    (0,0)--(cv2) (0,0)--(cv3) (0,0)--(cv5);
  \draw[ink,line width=.85pt] (0,0)--(cv4) (0,0)--(cv6) (0,0)--(cv7);
  \fill[bad] (0,0) circle (2.4pt);
  \draw[bad,very thick,->] (-30:1.30)
    arc[start angle=-30,end angle=90,radius=1.30];
  \node[text=bad,font=\bfseries\footnotesize]
    at (30:1.80) {$120^\circ$};
\end{scope}
\node[text=ink,font=\bfseries\scriptsize] at (4.6,-1.52) {Vertex axis: order $3$};
\end{tikzpicture}
\caption{Smallest symmetry-preserving rotations (excluding half-turns) for the cube, grouped by axis. The red point marks the axis viewed end-on; visible edges are solid and hidden edges are dashed.}
\label{fig:cube-axis-types}
\end{figure}

\begin{proposition}[cube realization]\label{prop:cube}
The group $S_4$, realized as the rotational symmetry group of a cube, has a
signed-Householder realization in three dimensions. Restricting it to the
subgroup $A_4$ gives a realization of $A_4$ in the same dimension.
\end{proposition}

\begin{proof}
The orientation-preserving symmetry group of a cube is isomorphic to
$S_4$: it permutes the four unoriented body diagonals faithfully.  Its
non-half-turn rotation axes are
\begin{itemize}
  \item the three face axes, supporting rotations through $90^\circ$.
  \item the four body-diagonal axes, supporting rotations through $120^\circ$.
\end{itemize}
All remaining non-identity rotations are half-turns and are automatically
admissible by the axis test (Proposition~\ref{prop:axis-test-app}).

Place the cube first in its standard orientation.  The three face axes and
four body-diagonal axes are represented, respectively, by
\[
 \mathcal F=\{(1,0,0),(0,1,0),(0,0,1)\},\qquad
 \mathcal B=\{(1,1,1),(1,1,-1),(1,-1,1),(-1,1,1)\}.
\]
Every vector in $\mathcal B$ has three nonzero coordinates, so this
orientation does not pass the axis test of
Proposition~\ref{prop:axis-test-app}.

We can change the orientation of these axes while preserving the group
being represented. Write $\rho(h)$ for the matrix of a cube symmetry in the
standard orientation. Reorienting the whole cube by a rotation $\mR$
replaces each symmetry matrix by $\mR\rho(h)\mR^\top$, an operation called
\emph{conjugation}. Using the same $\mR$ for every symmetry preserves
multiplication, since
\[
  (\mR\rho(h)\mR^\top)(\mR\rho(k)\mR^\top)
    =\mR\rho(hk)\mR^\top.
\]
At the same time, each rotation axis $\vu$ moves to $\mR\vu$. Our task is
therefore to find one rotation that places all seven axes in coordinate
planes. We choose a $45^\circ$ rotation about the third coordinate axis:

\[
 \mR=\begin{pmatrix}
  1/\sqrt2& 1/\sqrt2&0\\
 -1/\sqrt2& 1/\sqrt2&0\\
  0&0&1
 \end{pmatrix}\in SO(3).
\]
Applying $\mR$ to each of the seven axes gives
\[
 \begin{aligned}
 \mR\mathcal F&=
 \left\{\frac{(1,-1,0)}{\sqrt2},\frac{(1,1,0)}{\sqrt2},(0,0,1)\right\},\\
 \mR\mathcal B&=\{(\sqrt2,0,1),(\sqrt2,0,-1),
 (0,-\sqrt2,1),(0,\sqrt2,1)\}.
 \end{aligned}
\]
Thus the rotation does not spoil the three face axes: they still have a zero
coordinate, as do all four transformed body diagonals. The axis test
(Proposition~\ref{prop:axis-test-app}) therefore shows that every cube rotation is a
signed-Householder matrix. Restriction to $A_4$ proves the remaining claim.
\end{proof}

\subsection{\texorpdfstring{How to track $A_5$}{A5}}\label{app:icosahedron}

The orientation-preserving symmetry group of the icosahedron is isomorphic
to $A_5$. We prove that $A_5$ has no finite-state signed-Householder tracker
in three dimensions, then construct one in four dimensions. The latter
uses a many-to-one decoder, so the accumulated hidden matrix carries more
information than the group element being tracked.

An icosahedron has $12$ vertices, $20$ faces, and $30$ edges.  Pairing each
feature with its opposite gives the rotation axes below.  An axis has
\emph{order $k$} if the rotations about it form a cyclic group of $k$
elements; equivalently, its smallest positive rotation has angle $2\pi/k$
and returns to the identity after $k$ applications.  Thus there are
\[
  6\text{ axes of order }5,\qquad
  10\text{ axes of order }3,\qquad
  15\text{ axes of order }2.
\]
The order-$2$ axes do not matter: the axis test
(Proposition~\ref{prop:axis-test-app}) already makes every half-turn admissible.  The difficulty comes from the other sixteen axes.
Figure~\ref{fig:ico-axis-types} illustrates these two relevant axis types.

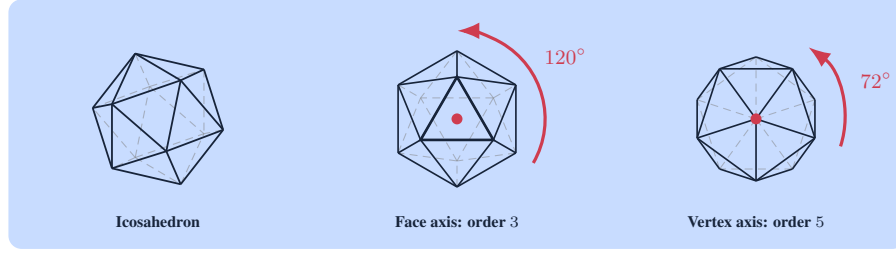
\begin{figure}[!ht]
\centering
\begin{tikzpicture}[>=Latex,font=\small,scale=.86,transform shape]
\fill[paperblue,rounded corners=5pt] (-6.9,-1.82) rectangle (6.9,2.02);

\begin{scope}[shift={(-4.6,.18)},
  x={(.437287cm,.164688cm)},
  y={(-.349829cm,.205860cm)},
  z={(0cm,.494064cm)}]
  \coordinate (ia1) at (0,1,1.618);   \coordinate (ia2) at (0,-1,1.618);
  \coordinate (ia3) at (0,1,-1.618);  \coordinate (ia4) at (0,-1,-1.618);
  \coordinate (ia5) at (1,1.618,0);   \coordinate (ia6) at (-1,1.618,0);
  \coordinate (ia7) at (1,-1.618,0);  \coordinate (ia8) at (-1,-1.618,0);
  \coordinate (ia9) at (1.618,0,1);   \coordinate (ia10) at (-1.618,0,1);
  \coordinate (ia11) at (1.618,0,-1); \coordinate (ia12) at (-1.618,0,-1);
  \foreach \a/\b in {ia1/ia5,ia3/ia4,ia3/ia5,
    ia3/ia6,ia3/ia11,ia3/ia12,ia4/ia11,ia5/ia6,ia5/ia9,
    ia5/ia11,ia7/ia11,ia9/ia11}
    \draw[muted,densely dashed,line width=.42pt] (\a)--(\b);
  \foreach \a/\b in {ia1/ia2,ia1/ia6,ia1/ia9,ia1/ia10,
    ia2/ia7,ia2/ia8,ia2/ia9,ia2/ia10,ia4/ia7,ia4/ia8,
    ia4/ia12,ia6/ia10,ia7/ia8,ia6/ia12,ia7/ia9,ia8/ia10,
    ia8/ia12,ia10/ia12}
    \draw[ink,line width=.68pt] (\a)--(\b);
\end{scope}
\node[text=ink,font=\bfseries\scriptsize] at (-4.6,-1.4) {Icosahedron};

\begin{scope}[shift={(0,.18)},x={(.56cm,0cm)},y={(0cm,.56cm)}]
  \coordinate (iaf1) at (-1,-.577350);       \coordinate (iaf2) at (0,-1.868345);
  \coordinate (iaf3) at (0,1.868345);        \coordinate (iaf4) at (1,.577350);
  \coordinate (iaf5) at (0,1.154701);        \coordinate (iaf6) at (-1.618034,.934172);
  \coordinate (iaf7) at (1.618034,-.934172); \coordinate (iaf8) at (0,-1.154701);
  \coordinate (iaf9) at (1,-.577350);        \coordinate (iaf10) at (-1.618034,-.934172);
  \coordinate (iaf11) at (1.618034,.934172); \coordinate (iaf12) at (-1,.577350);
  \foreach \a/\b in {iaf2/iaf8,iaf3/iaf4,iaf3/iaf12,iaf4/iaf7,
    iaf4/iaf8,iaf4/iaf11,iaf4/iaf12,iaf6/iaf12,iaf7/iaf8,
    iaf8/iaf10,iaf8/iaf12,iaf10/iaf12}
    \draw[muted,densely dashed,line width=.4pt] (\a)--(\b);
  \foreach \a/\b in {iaf1/iaf2,iaf1/iaf5,iaf1/iaf6,iaf1/iaf9,
    iaf1/iaf10,iaf2/iaf7,iaf2/iaf9,iaf2/iaf10,iaf3/iaf5,
    iaf3/iaf6,iaf3/iaf11,iaf5/iaf6,iaf5/iaf9,iaf5/iaf11,
    iaf6/iaf10,iaf7/iaf9,iaf7/iaf11,iaf9/iaf11}
    \draw[ink,line width=.62pt] (\a)--(\b);
  \draw[ink,line width=1pt] (iaf1)--(iaf5)--(iaf9)--cycle;
  \fill[bad] (0,0) circle (2.4pt);
  \draw[bad,very thick,->] (-30:2.42)
    arc[start angle=-30,end angle=90,radius=2.42];
  \node[text=bad,font=\bfseries\footnotesize]
    at (30:3.44) {$120^\circ$};
\end{scope}
\node[text=ink,font=\bfseries\scriptsize] at (0,-1.4) {Face axis: order $3$};

\begin{scope}[shift={(4.6,.18)},x={(.56cm,0cm)},y={(0cm,.56cm)}]
  \coordinate (iav1) at (0,0);             \coordinate (iav2) at (0,-1.701302);
  \coordinate (iav3) at (0,1.701302);      \coordinate (iav4) at (0,0);
  \coordinate (iav5) at (1,1.376382);      \coordinate (iav6) at (-1,1.376382);
  \coordinate (iav7) at (1,-1.376382);     \coordinate (iav8) at (-1,-1.376382);
  \coordinate (iav9) at (1.618034,-.525731);\coordinate (iav10) at (-1.618034,-.525731);
  \coordinate (iav11) at (1.618034,.525731);\coordinate (iav12) at (-1.618034,.525731);
  \foreach \a/\b in {iav3/iav4,iav3/iav11,iav3/iav12,iav4/iav7,
    iav4/iav8,iav4/iav11,iav4/iav12,iav7/iav8,iav7/iav11,
    iav8/iav12}
    \draw[muted,densely dashed,line width=.4pt] (\a)--(\b);
  \foreach \a/\b in {iav1/iav2,iav1/iav5,iav1/iav6,iav1/iav9,
    iav1/iav10,iav2/iav7,iav2/iav8,iav2/iav9,iav2/iav10,
    iav3/iav5,iav3/iav6,iav5/iav6,iav5/iav9,iav5/iav11,
    iav6/iav10,iav6/iav12,iav7/iav9,iav8/iav10,iav9/iav11,
    iav10/iav12}
    \draw[ink,line width=.62pt] (\a)--(\b);
  \fill[bad] (0,0) circle (2.4pt);
  \draw[bad,very thick,->] (-18:2.44)
    arc[start angle=-18,end angle=54,radius=2.44];
  \node[text=bad,font=\bfseries\footnotesize]
    at (18:3.47) {$72^\circ$};
\end{scope}
\node[text=ink,font=\bfseries\scriptsize] at (4.6,-1.4) {Vertex axis: order $5$};
\end{tikzpicture}
\caption{Smallest symmetry-preserving rotations (excluding half-turns) for the icosahedron, grouped by axis. The red point
marks the axis viewed end-on; visible edges are solid and hidden edges are
dashed.}
\label{fig:ico-axis-types}
\end{figure}

By the axis test (Proposition~\ref{prop:axis-test-app}), representing all
odd-order rotations would require three coordinate planes to cover all
sixteen axes.  We prove the stronger,
orientation-independent fact that \emph{any} plane covers at most four.

\subsubsection{The obstruction to finite-state tracking in three dimensions}

\begin{lemma}[icosahedral plane bound]\label{lem:ico-plane}
Among the six order-$5$ axes and ten order-$3$ axes of an icosahedron, no
plane through the center contains more than four axes.
\end{lemma}

\begin{proof}
Let $\varphi=(1+\sqrt5)/2$ be the golden ratio. In the usual coordinates, representative vectors for
the six unoriented vertex axes (the order-$5$ axes) are
\[
\mathcal V=
\left\{
(0,1,\pm\varphi),\ (1,\pm\varphi,0),\ (\varphi,0,\pm1)
\right\}.
\]
The face centers form the vertices of the dual dodecahedron. Representatives
of their ten unoriented axes (the order-$3$ axes) are
\[
\begin{split}
\mathcal T={}&\{(1,1,1),(1,1,-1),(1,-1,1),(-1,1,1)\}\\
&\cup\{(0,\varphi,\pm(\varphi-1)),
       (\varphi,\pm(\varphi-1),0),
       (\varphi-1,0,\pm\varphi)\}.
\end{split}
\]
It suffices to check that no three axes in either family are coplanar.
Indeed, substituting the displayed vectors and using
$\varphi^2=\varphi+1$ gives the following possible absolute determinants
of three distinct representatives:
\[
\begin{array}{c@{\qquad}l}
\toprule
\text{family}&|\det(\vv_1,\vv_2,\vv_3)|\\
\midrule
\mathcal V&1+\sqrt5,\ 3+\sqrt5\\
\mathcal T&\sqrt5-1,\ 2,\ 1+\sqrt5,\ 2\sqrt5,\ 4\\
\bottomrule
\end{array}
\]
All these determinants are nonzero. Thus a plane contains at most two vertex
axes and at most two face axes, hence at most four axes in total. The bound
is tight: the plane $x=0$ contains
\[
 (0,1,\varphi),\ (0,1,-\varphi),\
 (0,\varphi,\varphi-1),\ (0,\varphi,-(\varphi-1)).
\]
\end{proof}

Using the group facts recalled in \Cref{app:tracking-realization}, we now
turn the geometric bound into a tracking obstruction. We also use the
classification of finite subgroups of $SO(3)$~\citep[Sec.~3]{olive2017symmetry}:
cyclic groups, dihedral groups, and the rotational symmetry groups of the
tetrahedron, cube, and icosahedron, isomorphic to $A_4$, $S_4$, and $A_5$.
These occur up to an orthogonal change of coordinates.

\begin{theorem}[No finite-state three-dimensional $A_5$ tracker]
\label{thm:a5-three}
The group $A_5$ admits no finite-state signed-Householder tracking in
dimension three, even with an arbitrary decoder. In particular, it admits
no signed-Householder realization in that dimension.
\end{theorem}

\begin{proof}
Suppose a finite-state signed-Householder tracker exists, with transitions
$\mA_h$ and decoder $p$. The homomorphism property of finite tracking
(Lemma~\ref{lem:finite-tracking-homomorphism}) gives a finite hidden group
$\Gamma=\langle\mA_h:h\in A_5\rangle\subset O(3)$ and a surjective
homomorphism $p:\Gamma\to A_5$ satisfying $p(\mA_h)=h$.

\emph{Turn the hidden matrices into rotations.}
In three dimensions, multiplying an orientation-reversing orthogonal
matrix by $-\mI_3$ makes it orientation-preserving. Define
\[
  \psi(\mQ):=(\det\mQ)\mQ,\qquad
  H:=\psi(\Gamma)\subset SO(3).
\]
Indeed, $\det\psi(\mQ)=(\det\mQ)^4=1$. This map also preserves
multiplication:
\[
  \psi(\mQ_1\mQ_2)
    =(\det\mQ_1)(\det\mQ_2)\mQ_1\mQ_2
    =\psi(\mQ_1)\psi(\mQ_2).
\]
Thus $H$ is a finite rotation group. The map identifies at most a matrix
and its negative, since its kernel $K$ is contained in
$\{\pm\mI_3\}$. These identifications cannot change the decoded group
element. Since $K$ is a kernel, it is normal in $\Gamma$. As $p$ is
surjective, its image $p(K)$ is therefore normal in $A_5$. This image has
at most two elements, so simplicity of $A_5$ forces $p(K)=\{e\}$.
Consequently the decoder
descends to a well-defined surjective homomorphism
\[
  \bar p:H\to A_5,\qquad \bar p(\psi(\mQ)):=p(\mQ).
\]
To check that this is well defined, if $\psi(\mQ_1)=\psi(\mQ_2)$, then
$\mQ_1\mQ_2^{-1}\in K$, so $p(\mQ_1)=p(\mQ_2)$.

\emph{Identify the resulting rotation group.}
The classification of finite subgroups of $SO(3)$ now forces $H$ to be an
icosahedral rotation group. The tetrahedral and cubic groups have only
$12$ and $24$ elements, so cannot map onto $A_5$. A cyclic or dihedral
group has a cyclic normal subgroup of index at most two. Its image under
$\bar p$ would be a cyclic normal subgroup of $A_5$ with at most two
cosets. Simplicity forces that image to be trivial or all of $A_5$:
the former has $60$ cosets, while the latter is not cyclic. These cases
are therefore also impossible.

It follows that $|H|=60$, so the surjective map $\bar p:H\to A_5$ is
an isomorphism. In particular, the matrices $\psi(\mA_h)$, one for each
$h\in A_5$, exhaust $H$, because $\bar p(\psi(\mA_h))=h$.
Each is still signed-Householder: if $\mA_h=\mH_{\vu_h}\mD_h$, then
\[
  \psi(\mA_h)
    =\mH_{\vu_h}\bigl((\det\mA_h)\mD_h\bigr),
\]
and the factor in parentheses is again a sign diagonal.

\emph{Apply the geometric obstruction.}
We have obtained an orientation of the icosahedron in which every rotation
is signed-Householder. But its six order-$5$ axes and ten order-$3$ axes
cannot all lie in the three coordinate planes: the bound in
Lemma~\ref{lem:ico-plane} allows at most four axes per plane, hence at
most $12$ of the $16$ axes altogether. A nonidentity rotation about any
remaining axis fails the signed-Householder axis test
(Proposition~\ref{prop:axis-test-app}), a contradiction.

Finally, a signed-Householder realization would itself give a finite-state
tracker, with the inverse representation as decoder. It is therefore
excluded as well.
\end{proof}

\subsubsection{A four-dimensional construction that tracks $A_5$}\label{app:A5const}

\begin{figure}[t]
\centering
\begingroup
\adjustbox{width=.99\linewidth}{%
\begin{tikzpicture}[
 x=1cm,y=1cm,font=\fontsize{9}{10}\selectfont,
 text=srInk,>=Latex,line cap=round,line join=round,
 every node/.style={inner sep=1.5pt,outer sep=0pt},
 swap/.style={-{Latex[length=1.7mm,width=1.2mm]},
              draw=cupMuted,line width=.65pt},
 readout/.style={-{Latex[length=1.3mm,width=1mm]},
                 draw=srBlue!65,dashed,line width=.55pt},
 frame/.style={draw=srGray!50,line width=.45pt,rounded corners=2pt},
 rowlabel/.style={anchor=east,align=right},
 entry/.style={anchor=base west,inner sep=0pt},
 pics/afcup/.style args={#1/#2/#3}{code={
  \path[draw=#1!85!black,fill=#1!32,line width=.75pt,
        rounded corners=1.2pt]
    (-.215,.285)--(.215,.285)--(.315,-.285)--(-.315,-.285)--cycle;
  \draw[draw=#1!65!black,line width=.45pt]
    (-.268,-.205)--(.268,-.205);
  \node[font=\sffamily\bfseries\fontsize{10}{11}\selectfont,
        text=cupInk,inner sep=0pt] (#3) at (0,.015) {#2};
 }}
]
\path[use as bounding box] (0,0) rectangle (16.2,3.45);

\node[rowlabel] (afrowcups) at (1.50,2.58) {(a) Cups};
\node[rowlabel] (afrowlabels) at (1.50,1.35) {(b) Labels};
\node[rowlabel,text=srGray] (afrowpair) at (1.50,.95) {$F(\mQ_i)$};

\node[text=srGreen] (afh) at (5.00,3.20) {$h=(123)$};
\node[text=srGreen] (afk) at (8.90,3.20) {$k=(124)$};
\node[text=srGreen] (afell) at (12.80,3.20) {$\ell=(14)(23)$};

\pic at (2.030,2.58) {afcup={cupViolet/1/afcup00}};
\pic at (2.710,2.58) {afcup={cupTerracotta/2/afcup01}};
\pic at (3.390,2.58) {afcup={cupTeal/3/afcup02}};
\pic at (4.070,2.58) {afcup={srGold/4/afcup03}};

\pic at (5.930,2.58) {afcup={cupTeal/3/afcup10}};
\pic at (6.610,2.58) {afcup={cupViolet/1/afcup11}};
\pic at (7.290,2.58) {afcup={cupTerracotta/2/afcup12}};
\pic at (7.970,2.58) {afcup={srGold/4/afcup13}};

\pic at (9.830,2.58) {afcup={srGold/4/afcup20}};
\pic at (10.510,2.58) {afcup={cupTeal/3/afcup21}};
\pic at (11.190,2.58) {afcup={cupTerracotta/2/afcup22}};
\pic at (11.870,2.58) {afcup={cupViolet/1/afcup23}};

\pic at (13.730,2.58) {afcup={cupViolet/1/afcup30}};
\pic at (14.410,2.58) {afcup={cupTerracotta/2/afcup31}};
\pic at (15.090,2.58) {afcup={cupTeal/3/afcup32}};
\pic at (15.770,2.58) {afcup={srGold/4/afcup33}};

\draw[swap] (4.490,2.55)--(5.510,2.55);
\draw[swap] (8.390,2.55)--(9.410,2.55);
\draw[swap] (12.290,2.55)--(13.310,2.55);

\draw[frame,fill=srGray!3] (1.680,.55) rectangle (4.420,1.89);
\draw[srGray!28,line width=.35pt] (1.810,1.45)--(4.290,1.45);
\node[inner sep=0pt] (afQ0) at (3.050,1.68) {$\mQ_0=\mI_4$};
\node[entry,text=srBlue] (afL0) at (1.970,1.14) {$L_0=e$};
\node[entry,text=srGold] (afR0) at (1.970,.75) {$R_0=e$};
\draw[readout] (3.050,1.97)--(3.050,2.23);
\node[text=srBlue,anchor=west,font=\fontsize{8}{9}\selectfont,inner sep=0pt] (aff0) at (3.180,2.10) {$f$};

\draw[frame,fill=srGray!3] (5.580,.55) rectangle (8.320,1.89);
\draw[srGray!28,line width=.35pt] (5.710,1.45)--(8.190,1.45);
\node[inner sep=0pt] (afQ1) at (6.950,1.68) {$\mQ_1=g(h)$};
\node[entry,text=srBlue] (afL1) at (5.870,1.14) {$L_1=(123)$};
\node[entry,text=srGold] (afR1) at (5.870,.75) {$R_1=(132)$};
\draw[readout] (6.950,1.97)--(6.950,2.23);
\node[text=srBlue,anchor=west,font=\fontsize{8}{9}\selectfont,inner sep=0pt] (aff1) at (7.080,2.10) {$f$};

\draw[frame,fill=srGray!3] (9.480,.55) rectangle (12.220,1.89);
\draw[srGray!28,line width=.35pt] (9.610,1.45)--(12.090,1.45);
\node[inner sep=0pt] (afQ2) at (10.850,1.68) {$\mQ_2=g(k)\mQ_1$};
\node[entry,text=srBlue] (afL2) at (9.770,1.14) {$L_2=(14)(23)$};
\node[entry,text=srGold] (afR2) at (9.770,.75) {$R_2=(13)(24)$};
\draw[readout] (10.850,1.97)--(10.850,2.23);
\node[text=srBlue,anchor=west,font=\fontsize{8}{9}\selectfont,inner sep=0pt] (aff2) at (10.980,2.10) {$f$};

\draw[frame,fill=srGreen!7] (13.380,.55) rectangle (16.120,1.89);
\draw[srGray!28,line width=.35pt] (13.510,1.45)--(15.990,1.45);
\node[inner sep=0pt] (afQ3) at (14.750,1.68) {$\mQ_3=g(\ell)\mQ_2$};
\node[entry,text=srBlue] (afL3) at (13.670,1.14) {$L_3=e$};
\node[entry,text=srGold] (afR3) at (13.670,.75) {$R_3=(12)(34)$};
\draw[readout] (14.750,1.97)--(14.750,2.23);
\node[text=srBlue,anchor=west,font=\fontsize{8}{9}\selectfont,inner sep=0pt] (aff3) at (14.880,2.10) {$f$};

\draw[srGold,line width=.65pt] (13.67,.665)--(15.70,.665);

\node[inner sep=0pt] (affooter) at (8.10,.16)
 {Same cups: {\color{srBlue}$f(\mQ_3)=f(\mQ_0)=e$}
  \qquad Different hidden states: $\mQ_3\ne\mQ_0$};

\ifdefined\afLayoutQA
\foreach \afNode in {afrowcups,afrowlabels,afrowpair,afh,afk,afell,afcup00,afcup01,afcup02,afcup03,afcup10,afcup11,afcup12,afcup13,afcup20,afcup21,afcup22,afcup23,afcup30,afcup31,afcup32,afcup33,afQ0,afL0,afR0,aff0,afQ1,afL1,afR1,aff1,afQ2,afL2,afR2,aff2,afQ3,afL3,afR3,aff3,affooter} {
 \path let \p1=(\afNode.south west), \p2=(\afNode.north east)
   in \pgfextra{\typeout{AFNODE|\afNode|\x1|\y1|\x2|\y2}};
}
\fi
\end{tikzpicture}%
}
\endgroup
\vspace{-2mm}
\caption{Many-to-one tracking, illustrated by $h=(123)$, $k=(124)$, and
$\ell=(kh)^{-1}=(14)(23)$.
The cups show identities at successive positions; object $5$ is fixed and omitted.
Each box groups the two labels $F(\mQ_i)=(L_i,R_i)$ of one hidden matrix,
where $L_i=f(\mQ_i)$ and $R_i=f_-(\mQ_i)$.
An input $a$ gives $\mQ\mapsto g(a)\mQ$ and
$(L,R)\mapsto(aL,a^{-1}R)$.
The cups and $L$ return to their initial values, but $R_3=(12)(34)\ne R_0=e$;
hence $\mQ_3\ne\mQ_0$, although both decode to $e$.}
\label{fig:a5-many-to-one-cups}
\end{figure}
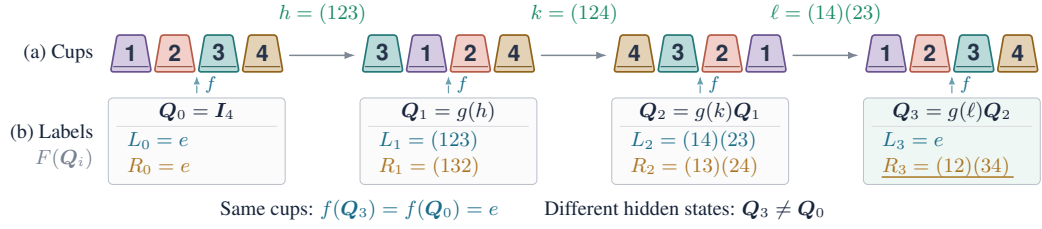
The preceding theorem shows that a non-injective decoder cannot rescue a
finite orthogonal hidden group in three dimensions. Four dimensions do
admit such a construction. We first show how to compute iterated products
in $SO(3)$ using signed-Householder transitions in $SO(4)$. An encoder $g$
converts each input rotation into a permitted transition, and a fixed
decoder $f$ recovers the accumulated product from the hidden matrix.
The hidden matrix need not itself equal the encoding of that product:
what matters is that decoding each hidden update gives the correct
three-dimensional update. For inputs belonging to a finite rotation group,
the construction also has only finitely many reachable hidden matrices.

The underlying $SO(4)$ geometry is the classical left/right quaternionic
decomposition of four-dimensional rotations~\citep{mebius2005matrix}; the
$A_5$ case uses the binary icosahedral lift~\citep{choi2018binary}. Please see~\Cref{fig:a5-many-to-one-cups} for an intuitive example of the construction.

\begin{proposition}[Tracking rotations with signed-Householder transitions]
\label{prop:so3-four-dimensional-encoding}
There exist an injective encoder $g:SO(3)\to SO(4)$ and a decoder
$f:SO(4)\to SO(3)$ such that every $g(\mA)$ is signed-Householder,
$f(\mI_4)=\mI_3$, and
\begin{equation}\label{eq:so3-decoded-update}
    f(g(\mA)\mS)=\mA\,f(\mS)
    \qquad
    \text{for every }\mA\in SO(3),\quad \mS\in SO(4).
\end{equation}
Consequently, for every $n\geq1$ and every sequence of rotations
$\mA_1,\ldots,\mA_n\in SO(3)$,
\begin{equation}\label{eq:so3-decoded-product}
    f\bigl(g(\mA_1)\cdots g(\mA_n)\bigr)
    =\mA_1\cdots\mA_n.
\end{equation}
In particular, $f(g(\mA))=\mA$.

Explicitly, let $\mA=R_{\vu,\theta}$ be the three-dimensional rotation
through angle $\theta\in[0,\pi]$ about the axis spanned by a unit vector
$\vu\in\R^3$, with the direction of rotation specified by the right-hand
rule around $\vu$. The encoder can be chosen as
\begin{equation}\label{eq:so3-encoding}
    g(\mA)=(\mI_4-2\va_{\mA}\va_{\mA}^\top)\mD_0,
    \qquad
    \va_{\mA}:=
    \begin{pmatrix}
        \cos(\theta/2)\\
        \sin(\theta/2)\,\vu
    \end{pmatrix},
    \qquad
    \mD_0:=\Diag(-1,1,1,1).
\end{equation}
Thus all encoded transitions use the same diagonal sign matrix $\mD_0$.
To define the decoder, write a matrix $\mS\in SO(4)$ in $1+3$ block form:
\[
    \mS=
    \begin{pmatrix}
        \alpha & \vr^\top\\
        \vs & \mC
    \end{pmatrix},
    \qquad
    \alpha\in\R,\quad \vr,\vs\in\R^3,\quad
    \mC\in\R^{3\times3}.
\]
Then the decoder is the fixed quadratic map
\begin{equation}\label{eq:so3-explicit-decoder}
    f(\mS)=\alpha\mC-\vs\vr^\top-\mC[\vr]_\times,
\end{equation}
where $[\vr]_\times$ is the $3\times3$ cross-product matrix, defined by
$[\vr]_\times\vx=\vr\times\vx$ for every $\vx\in\R^3$.

Moreover, for every finite subgroup $H\subset SO(3)$, the group
\begin{equation}\label{eq:so3-finite-hidden-group}
    \Gamma_H:=\langle g(\mA):\mA\in H\rangle\subset SO(4)
    \qquad\text{satisfies}\qquad
    |\Gamma_H|\leq 2|H|^2.
\end{equation}
Here $\Gamma_H$ is the group generated by the encoded transitions.
In particular, starting from $\mI_4$, products of encoded inputs from $H$
reach at most $2|H|^2$ distinct hidden matrices.
\end{proposition}

\begin{proof}
We first verify the encoder, then construct the decoder through its action
on a three-dimensional space of matrices, and finally prove the finite-state
bound.

\emph{The encoded transitions.}
The vector $\va_{\mA}$ in \eqref{eq:so3-encoding} has unit norm, so
$\mI_4-2\va_{\mA}\va_{\mA}^\top$ is a Householder reflection.
The fixed matrix $\mD_0$ is also a reflection: it changes the sign of the
first coordinate. Their product is therefore orthogonal with determinant
$+1$, and has the required signed-Householder form. At $\theta=0$, the
construction gives $g(\mI_3)=\mI_4$ independently of the chosen axis.
At $\theta=\pi$, the two choices $\vu$ and $-\vu$ replace
$\va_{\mA}$ by its negative and hence give the same Householder matrix.
Thus the encoder is well defined for every rotation.

Multiplying the two factors gives
\begin{equation}\label{eq:g-block-form}
    g(R_{\vu,\theta})=
    \begin{pmatrix}
        c & -s\vu^\top\\
        s\vu & \mI_3+(c-1)\vu\vu^\top
    \end{pmatrix},
    \qquad c:=\cos\theta,\quad s:=\sin\theta.
\end{equation}
This matrix rotates the plane spanned by the first coordinate vector
$\ve_0=(1,0,0,0)^\top$ and $(0,\vu)^\top$ through angle $\theta$,
and fixes the orthogonal complement of that plane.

\emph{A decoder that preserves multiplication.}
Let $\ve_0,\ldots,\ve_3$ be the standard basis of $\R^4$, and define
\[
    \mB_{ij}:=\ve_i\ve_j^\top-\ve_j\ve_i^\top,
    \qquad
    \mJ_1:=\mB_{01}+\mB_{23},\quad
    \mJ_2:=\mB_{02}+\mB_{31},\quad
    \mJ_3:=\mB_{03}+\mB_{12}.
\]
The three skew-symmetric matrices $\mJ_1,\mJ_2,\mJ_3$ are linearly
independent. For a coefficient vector $\vx\in\R^3$, write their linear
combination as
\[
    \mJ(\vx):=\sum_{i=1}^3x_i\mJ_i
      =\begin{pmatrix}
          0 & \vx^\top\\
          -\vx & -[\vx]_\times
        \end{pmatrix}.
\]
We will show that conjugation $\mX\mapsto\mQ\mX\mQ^\top$ by any
$\mQ\in SO(4)$ maps this three-dimensional space to itself. Its action
on the coefficient vector $\vx$ will be the decoder $f(\mQ)$.

To verify this invariance, it suffices to check two elementary types of
rotation. First, for $\mT_{\mR}:=\Diag(1,\mR)$ with $\mR\in SO(3)$,
block multiplication and the identity
$\mR[\vx]_\times\mR^\top=[\mR\vx]_\times$ give
\[
    \mT_{\mR}\mJ(\vx)\mT_{\mR}^\top=\mJ(\mR\vx).
\]
Second, let $\mU_\phi$ rotate the first two coordinate directions in
$\R^4$, and let $\mR_\phi$ rotate the last two coordinate directions in
$\R^3$:
\[
    \mU_\phi:=
    \begin{pmatrix}
      \cos\phi&-\sin\phi&0&0\\
      \sin\phi&\cos\phi&0&0\\
      0&0&1&0\\
      0&0&0&1
    \end{pmatrix},
    \qquad
    \mR_\phi:=
    \begin{pmatrix}
      1&0&0\\
      0&\cos\phi&-\sin\phi\\
      0&\sin\phi&\cos\phi
    \end{pmatrix}.
\]
Multiplication on the three basis matrices gives
\[
\begin{aligned}
    \mU_\phi\mJ_1\mU_\phi^\top&=\mJ_1,\\
    \mU_\phi\mJ_2\mU_\phi^\top&=\cos\phi\,\mJ_2+\sin\phi\,\mJ_3,\\
    \mU_\phi\mJ_3\mU_\phi^\top&=-\sin\phi\,\mJ_2+\cos\phi\,\mJ_3.
\end{aligned}
\]
Equivalently,
$\mU_\phi\mJ(\vx)\mU_\phi^\top=\mJ(\mR_\phi\vx)$.
Rotations among the last three coordinates are included in $\mT_{\mR}$.
Conjugating $\mU_\phi$ by suitable $\mT_{\mR}$ gives rotations between
the first coordinate and either of the other two coordinates. Thus these
matrices generate all coordinate-plane rotations in $\R^4$. Every matrix
in $SO(4)$ is a product of such rotations, as follows, for example, by
eliminating its entries with Givens rotations. Invariance therefore holds
for every $\mQ\in SO(4)$.

There is consequently a unique $3\times3$ matrix $\mM(\mQ)$ such that
\begin{equation}\label{eq:J-action}
    \mQ\mJ(\vx)\mQ^\top=\mJ(\mM(\mQ)\vx)
    \qquad(\vx\in\R^3).
\end{equation}
On the two elementary types above, this matrix is respectively $\mR$ and
$\mR_\phi$, both in $SO(3)$. Applying two conjugations in succession gives
\[
    (\mQ_1\mQ_2)\mJ(\vx)(\mQ_1\mQ_2)^\top
    =\mJ\bigl(\mM(\mQ_1)\mM(\mQ_2)\vx\bigr).
\]
Uniqueness of the coefficients implies that $\mM$ preserves multiplication.
Since the elementary rotations generate $SO(4)$, it also follows that
$\mM(\mQ)\in SO(3)$ for every $\mQ\in SO(4)$.

It remains to identify this induced matrix with the explicit decoder.
Write $\mQ$ in the block form used in the statement. The top-right block
of $\mQ\mJ(\vx)\mQ^\top$ is
\[
    -(\vr^\top\vx)\vs^\top
       +(\alpha\vx^\top-\vr^\top[\vx]_\times)\mC^\top
    =\bigl((\alpha\mC-\vs\vr^\top-\mC[\vr]_\times)\vx\bigr)^\top.
\]
Here we used $\vr^\top[\vx]_\times=([\vr]_\times\vx)^\top$.
The top-right block of $\mJ(\mM(\mQ)\vx)$ is $(\mM(\mQ)\vx)^\top$,
so $\mM(\mQ)=f(\mQ)$ as defined in \eqref{eq:so3-explicit-decoder}.
We have therefore proved
\begin{equation}\label{eq:f-homomorphism}
    f(\mQ)\in SO(3),\qquad f(\mI_4)=\mI_3,\qquad
    f(\mQ_1\mQ_2)=f(\mQ_1)f(\mQ_2).
\end{equation}

\emph{Decoding the accumulated product.}
For the blocks in \eqref{eq:g-block-form},
\[
    \alpha\mC-\vs\vr^\top
       =\vu\vu^\top+c(\mI_3-\vu\vu^\top),
    \qquad
    -\mC[\vr]_\times=s[\vu]_\times,
\]
where the second identity uses $\vu^\top[\vu]_\times=0$.
Hence Rodrigues' rotation formula gives
\begin{equation}\label{eq:f-on-g}
    f(g(\mA))
       =\vu\vu^\top+c(\mI_3-\vu\vu^\top)+s[\vu]_\times
       =R_{\vu,\theta}=\mA.
\end{equation}
This also proves injectivity: if $g(\mA)=g(\mB)$, applying $f$ gives
$\mA=\mB$. Combining \eqref{eq:f-homomorphism} and \eqref{eq:f-on-g}
proves \eqref{eq:so3-decoded-update} for every hidden matrix in $SO(4)$.
Repeated application gives \eqref{eq:so3-decoded-product}.

\emph{Finitely many hidden states for a finite input group.}
We introduce a second triple of matrices
\[
    \mK_1:=\mB_{01}-\mB_{23},\quad
    \mK_2:=\mB_{02}-\mB_{31},\quad
    \mK_3:=\mB_{03}-\mB_{12},
\]
and write their linear combinations as
\[
    \mK(\vx):=\sum_{i=1}^3x_i\mK_i
       =\begin{pmatrix}0&\vx^\top\\-\vx&[\vx]_\times\end{pmatrix}.
\]
The same elementary rotations act on this triple by
\[
    \mT_{\mR}\mK(\vx)\mT_{\mR}^\top=\mK(\mR\vx),
    \qquad
    \mU_\phi\mK(\vx)\mU_\phi^\top=\mK(\mR_{-\phi}\vx).
\]
Thus its span is also invariant. The induced matrix is an auxiliary map
$f_-:SO(4)\to SO(3)$ that preserves multiplication. Reading the top-right
block as before gives
\begin{equation}\label{eq:so3-auxiliary-decoder}
    \mQ\mK(\vx)\mQ^\top=\mK(f_-(\mQ)\vx),
    \qquad
    f_-(\mQ)=\alpha\mC-\vs\vr^\top+\mC[\vr]_\times.
\end{equation}
Substituting \eqref{eq:g-block-form} now reverses the sign of the sine term
in Rodrigues' formula, so
\begin{equation}\label{eq:fminus-on-g}
    f_-(g(\mA))=R_{\vu,-\theta}=\mA^{-1}.
\end{equation}

The two maps together determine a hidden matrix up to sign. To see this,
suppose $f(\mQ)=f_-(\mQ)=\mI_3$. Then conjugation by $\mQ$ fixes all
six matrices $\mJ_i,\mK_i$. Their sums and differences span all the
$\mB_{ij}$, so $\mQ$ commutes with every $\mB_{ij}$. It therefore also
commutes with $-\mB_{ij}^2$, the orthogonal projector onto the coordinate
plane spanned by $\ve_i,\ve_j$. Thus $\mQ$ preserves every coordinate
plane, and also every coordinate axis, obtained by intersecting two such
planes. This makes $\mQ$ diagonal. Commutation with $\mB_{ij}$ then
forces its $i$-th and $j$-th diagonal entries to agree. Hence $\mQ$ is a
scalar matrix, and orthogonality gives $\mQ=\pm\mI_4$. Conversely, both
of these matrices act trivially by conjugation.

Define the joint map $F(\mQ):=(f(\mQ),f_-(\mQ))$. It preserves
multiplication and satisfies
$F(\mQ)=(\mI_3,\mI_3)$ exactly when $\mQ=\pm\mI_4$.
Its kernel therefore has two elements, so every nonempty fiber contains
exactly two matrices in $SO(4)$, differing by sign.

Now let $H\subset SO(3)$ be finite. For every encoded generator,
\[
    F(g(\mA))=(\mA,\mA^{-1})\in H\times H.
\]
Since $H\times H$ is closed under multiplication and inverses,
$F(\Gamma_H)\subseteq H\times H$. There are at most $|H|^2$ such
pairs and at most two hidden matrices for each pair. Consequently
$|\Gamma_H|\leq2|H|^2$, proving \eqref{eq:so3-finite-hidden-group}.
\end{proof}

Applying \Cref{prop:so3-four-dimensional-encoding} to the icosahedral group gives our
four-dimensional $A_5$ tracker.

\begin{corollary}[Four-dimensional \(A_5\) tracker]
\label{cor:a5-four}
The group $A_5$ admits finite-state signed-Householder tracking in dimension
four, implemented by one CKDA layer. Its sixty transitions
use the encoder in \eqref{eq:so3-encoding}, with the fixed diagonal gate
$\Diag(-1,1,1,1)$, and its decoder is the quadratic map in
\eqref{eq:so3-explicit-decoder}. There are at most $7200$ reachable hidden
matrices, and the decoder is many-to-one on this set.
\end{corollary}

\begin{proof}
Identify $A_5$ with the icosahedral rotation group and apply
\Cref{prop:so3-four-dimensional-encoding} with $H=A_5$.
Starting from $\mS_0=\mI_4$, the transitions $\mA_h=g(h)$ track the
group with decoder $f$ and zero additive terms. They are distinct and
orthogonal, and the reachable set has size at most $2\cdot60^2=7200$.

For the many-to-one claim, choose noncommuting $h,k\in A_5$ and set
$\ell=(hk)^{-1}$. The two maps from the proposition satisfy
\[
    f(\mA_h\mA_k\mA_\ell)=e,
    \qquad
    f_-(\mA_h\mA_k\mA_\ell)=h^{-1}k^{-1}hk\ne e.
\]
Thus this reachable product differs from $\mI_4$, although both decode
to the identity.
\end{proof}

\subsection{Additive terms and independent heads}
\label{app:affine-writes}
Under the finite-reachability convention of \Cref{app:tracking-realization},
for group tracking with non-expansive updates, the additive
term $\mB_g$ can be removed. Restricting to a suitable invariant subspace
then gives orthogonal update matrices, without increasing eigenvalue
multiplicities or the number of state columns.

\begin{proposition}[Removal of the additive term and orthogonal restriction]
\label{prop:affine-orthogonal-reduction}
Let $G$ be a finite group and let
$T_g(\mS)=\mA_g\mS+\mB_g$, $g\in G$, act on $\mS\in\R^{n\times m}$,
with $\|\mA_g\|_2\leq1$ for every $g$.
Suppose the reachable set $\mathcal R$ from $\mS_0$ is finite and admits a
decoder satisfying $d(\mS_0)=e$ and
$d(T_g(\mS))=g\,d(\mS)$ for all $g\in G$ and $\mS\in\mathcal R$.
Then there is a finite-state tracker of $G$ with states in $\R^{r\times m}$,
where $0\leq r\leq n$, and updates
$\widehat{\mS}_t=\widehat{\mA}_{g_t}\widehat{\mS}_{t-1}$ with orthogonal
matrices $\widehat{\mA}_g$. These matrices represent the restrictions of the
original $\mA_g$ to a common invariant subspace in an orthonormal basis.
Every eigenvalue of $\widehat{\mA}_g$ is an eigenvalue of $\mA_g$ with
no increase in algebraic multiplicity. Both decoders may be many-to-one.
\end{proposition}

\begin{proof}
We first choose a word whose linear part has minimum Frobenius norm and
repeat it until it acts as a projection without changing the decoded value.
Minimality and non-expansion then force the original matrices to act
orthogonally on its image. The resulting affine updates permute a finite
set of states; subtracting their average removes the additive term, and
expressing the centered states in an orthonormal basis of the image gives
the orthogonal tracker.

\emph{Step 1: Find a repeatable word that leaves the decoded value unchanged.}
Let $V$ be the span of all columns of differences $\mX-\mY$ with
$\mX,\mY\in\mathcal R$. Since
$\mA_g(\mX-\mY)=T_g(\mX)-T_g(\mY)$, each $\mA_g$ preserves $V$.
For a word $u=g_1\cdots g_t$, write
$T_u=T_{g_t}\circ\cdots\circ T_{g_1}$,
$\mP_u=\mA_{g_t}\cdots\mA_{g_1}$, and $g_u=g_t\cdots g_1$.
Thus $d(T_u(\mS))=g_u d(\mS)$.

Only finitely many linear actions $\mP_u|_V$ occur. Indeed, there are
finitely many maps $\mathcal R\to\mathcal R$, and if two word maps agree
there, subtracting their values shows that their linear parts agree on the
columns spanning $V$.
Choose $u$ minimizing the Frobenius norm $\|\mP_u|_V\|_F$.
Put $N=|\mathcal R|!$.
The \emph{orbit} $\mS,T_u(\mS),T_u^2(\mS),\ldots$ enters a cycle
within $|\mathcal R|-1$ steps: among its first $|\mathcal R|+1$ states, two coincide.
Every cycle length is at most $|\mathcal R|$ and divides $N$. Thus another $N$
applications after the first $N$ return to the same state, giving
$T_u^{2N}=T_u^N$ on $\mathcal R$.

Let $z$ be $u$ repeated $N$ times. Then $T_z=T_u^N$ and
$T_z^2=T_z$ on $\mathcal R$. Subtracting this identity at two reachable
states shows that its linear part $\mE:=\mP_z|_V$ satisfies
$\mE^2=\mE$. It still has minimum Frobenius norm: further applications of a
non-expansive matrix cannot increase that norm, and $\mE$ is itself a
linear word action. Decoding at $\mS_0$ gives
$g_z^2=g_z$ because $d(\mS_0)=e$. Multiplying by $g_z^{-1}$ gives
$g_z=e$. Thus applying $z$ may merge states, but never changes their
decoded group element.

\emph{Step 2: Obtain orthogonal updates on the surviving subspace.}
The identity $\mE^2=\mE$ makes $\mE$ a projection onto
$\operatorname{im}\mE$. It is an orthogonal projection because it cannot
increase lengths. Indeed, for $\vw\in \operatorname{im}\mE$, let $\vz$ be the orthogonal
projection of $\vw$ onto $\ker\mE$. Then $\mE(\vw-\vz)=\vw$, so
\[
    \|\vw\|^2=\|\mE(\vw-\vz)\|^2
    \leq\|\vw-\vz\|^2=\|\vw\|^2-\|\vz\|^2.
\]
Hence $\vz=0$ and $\operatorname{im}\mE\perp\ker\mE$.
In words, removing a kernel component shortens the vector, and a
non-expansive map cannot reconstruct the longer vector.

Write $r=\dim \operatorname{im}\mE$ and choose any orthonormal basis
$\vw_1,\ldots,\vw_r$ of $\operatorname{im}\mE$. Since $\mE$ is an orthogonal projection,
$\|\mE\|_F^2=r$. The product $\mE\mA_g\mE$ is another linear word
action on $V$, so minimality gives
\[
    r\leq\|\mE\mA_g\mE\|_F^2
    =\sum_{i=1}^{r}\|\mE\mA_g\vw_i\|^2\leq r.
\]
The sum has $r$ terms, each at most one, so all must equal one:
no basis vector can lose length. Any unit vector $\vw\in \operatorname{im}\mE$ can be
included in such a basis. Consequently,
\[
    1=\|\mE\mA_g\vw\|\leq\|\mA_g\vw\|\leq1.
\]
Equality for the orthogonal projection implies $\mA_g\vw\in \operatorname{im}\mE$,
and the remaining equality gives $\|\mA_g\vw\|=\|\vw\|$.
Thus every $\mA_g$ preserves $\operatorname{im}\mE$ and acts orthogonally there.

\emph{Step 3: Center the states and construct the orthogonal tracker.}
Let $\mathcal R_*:=T_z(\mathcal R)$. Every state in $\mathcal R_*$ is
fixed by $T_z$, so its differences satisfy
$\mX-\mY=\mE(\mX-\mY)$ and have columns in $\operatorname{im}\mE$.
For each input $g$, define $T_g^*:=T_z\circ T_g$:
apply $g$, then $z$. This sends $\mathcal R_*$ into itself. For
$\mX,\mY\in\mathcal R_*$, Step 2 gives
\[
    T_g^*(\mX)-T_g^*(\mY)
    =\mE\mA_g(\mX-\mY)=\mA_g(\mX-\mY).
\]
This preserves distances, so distinct states stay distinct.
A one-to-one map of a finite set into itself is a permutation.
Consequently, $T_g^*$ fixes the average
$\overline{\mS}:=|\mathcal R_*|^{-1}\sum_{\mS\in\mathcal R_*}\mS$:
affine maps preserve averages, and a permutation leaves the average unchanged.
Since differences $\mS-\overline{\mS}$ have columns in $\operatorname{im}\mE$, we obtain
\[
    T_g^*(\mS)-\overline{\mS}
    =\mA_g(\mS-\overline{\mS}),\qquad \mS\in\mathcal R_*.
\]
Thus subtracting this common average turns every update into multiplication
by the original $\mA_g$.

Choose $\mU\in\R^{n\times r}$ whose columns form an orthonormal basis
of $\operatorname{im}\mE$. For a vector in this subspace, $\mU^\top$
returns its coordinates in that basis and $\mU$ reconstructs the vector.
Define
\[
    \widehat{\mA}_g=\mU^\top\mA_g\mU,
    \qquad
    \widehat{\mS}_0=\mU^\top\bigl(T_z(\mS_0)-\overline{\mS}\bigr),
    \qquad
    \widehat{\mS}_t=\widehat{\mA}_{g_t}\widehat{\mS}_{t-1}.
\]
Step 2 gives $\mA_g\mU=\mU\widehat{\mA}_g$ and shows that
$\widehat{\mA}_g$ is orthogonal: it is the same length-preserving action
on $\operatorname{im}\mE$, expressed in orthonormal coordinates.
For every $\mS\in\mathcal R_*$, the centering identity above gives
\[
    \widehat{\mA}_g\mU^\top(\mS-\overline{\mS})
    =\mU^\top\bigl(T_g^*(\mS)-\overline{\mS}\bigr).
\]
Thus the new updates preserve the finite set
$\mU^\top(\mathcal R_*-\overline{\mS})$.
On this set, define $\widehat d(\mX)=d(\mU\mX+\overline{\mS})$:
reconstruct the original state and apply its decoder.
Since $g_z=e$, we have $d(T_z(\mS_0))=e$ and
$d(T_g^*(\mS))=g\,d(\mS)$. Consequently,
\[
    \widehat d(\widehat{\mS}_0)=e,
    \qquad
    \widehat d(\widehat{\mA}_g\mX)
    =d\bigl(T_g^*(\mU\mX+\overline{\mS})\bigr)
    =g\,\widehat d(\mX).
\]
This is the required tracker with orthogonal matrices and no additive term.
Its states have $r$ rows and the same $m$ columns.
If $r=0$, the centered set is a singleton, forcing $G=\{e\}$;
the zero-dimensional state suffices.

Finally, completing the columns of $\mU$ to a basis of $\R^n$ puts
$\mA_g$ in block triangular form with $\widehat{\mA}_g$ as a diagonal
block. Its characteristic polynomial therefore contains that of
$\widehat{\mA}_g$ as a factor: every eigenvalue of $\widehat{\mA}_g$
occurs in $\mA_g$ at least as many times.
\end{proof}

\begin{proposition}[Orthogonal reduction for independent heads]
\label{prop:independent-head-reduction}
Let $G$ be a finite group and consider $H<\infty$ independent heads with
exact affine updates
\[
    T_g^{(h)}(\mS^{(h)})
    =\mA_g^{(h)}\mS^{(h)}+\mB_g^{(h)},\qquad
    \|\mA_g^{(h)}\|_2\leq1,
\]
for $g\in G$ and $h=1,\ldots,H$.
Suppose each head has a finite reachable set from its fixed initial state
and a joint decoder tracks $G$.
Then there is a finite-state tracker of $G$ with homogeneous orthogonal
updates on the same heads, allowing zero-dimensional heads. Each reduced
head transition is a restriction of its original transition to a common
invariant subspace, so eigenvalue multiplicities do not increase in any
head. The joint transition matrices generate a finite group $\Gamma$
with a surjective homomorphism $p:\Gamma\to G$ taking the transition
for each input $g$ to $g$.
\end{proposition}

\begin{proof}
Let $\mathcal R_h$ be the finite reachable set of head $h$.
The joint reachable set satisfies
$\mathcal R\subseteq\prod_{h=1}^{H}\mathcal R_h$ and is therefore finite,
since $H<\infty$.
Stack the head states, padding their columns with zeros if needed, so the
joint linear part is
$\mA_g=\operatorname{diag}(\mA_g^{(1)},\ldots,\mA_g^{(H)})$.
Apply \Cref{prop:affine-orthogonal-reduction} and express the resulting
centered states in the original coordinates. Their reachable set is
finite, and every $\mA_g$ acts orthogonally on its column span $W$.
Let $P_h$ project onto head $h$ and put $W_h=P_hW$.
Block diagonality makes each $W_h$ invariant. For $\vw\in W$,
\[
 \sum_{h=1}^{H}\|\mA_g^{(h)}P_h\vw\|^2
 =\|\mA_g\vw\|^2=\|\vw\|^2
 =\sum_{h=1}^{H}\|P_h\vw\|^2
\]
forces equality in each head, since no head can increase the norm.
Thus $\mA_g^{(h)}|_{W_h}$ is orthogonal. Express these restrictions in
orthonormal bases of the $W_h$. Their eigenvalues, with algebraic
multiplicities, are inherited from the original head transitions.
Projecting the centered states onto all heads preserves their finite
joint reachable set and its decoder: the projections together determine
the entire centered state.

Each reduced head transition permutes its projected finite reachable
set. Its columns span $W_h$, so the action on that set determines the
matrix. Consequently only finitely many joint matrix products occur;
being orthogonal, they form a finite group $\Gamma$.
For the reduced initial state $\widehat{\mS}_0$ and decoder $\widehat d$,
define $p(\mP)=\widehat d(\mP\widehat{\mS}_0)$.
Every element of $\Gamma$ is represented by an input word, and the
tracking identity gives $p(\mP\mQ)=p(\mP)p(\mQ)$ and
$p(\widehat{\mA}_g)=g$. Thus $p$ is a surjective homomorphism.
\end{proof}

Stacking the reduced heads gives one block-diagonal orthogonal update,
whose spectrum is the multiset union of the head spectra. Non-real
eigenvalue pairs therefore add across heads: this combination does not
preserve a bound on their number per head or the original single-head
transition family.

\subsection{A spectral obstruction to finite-state tracking of \texorpdfstring{$S_5$}{S5}}
\label{app:s5-spectral-obstruction}

The previous construction tracks $A_5$, the $60$ even permutations of five
objects. We now show that the same spectral constraint cannot support
$S_5$, which contains all $120$ permutations. The model must handle every
permutation and their compositions, even when several hidden states
represent the same output.

\begin{theorem}[Spectral obstruction to $S_5$ tracking; restatement of Theorem~\ref{thm:s5-impossibility}]
\label{thm:s5-impossibility-app}
Suppose every head transition $\mA$ satisfies $\|\mA\|_2\leq1$ and has
at most one non-real conjugate eigenvalue pair on the unit circle,
counted with algebraic multiplicity.
Then a single recurrent layer with any finite number of independent heads
cannot track $S_5$ when inputs range over all of $S_5$ and the exact affine
updates of each head reach only finitely many states from its fixed initial state.
This holds for arbitrary state dimensions, input-dependent additive matrices,
and many-to-one joint state decoders.
\end{theorem}

\paragraph{Proof idea.}
After the orthogonal reduction, we construct two conjugate transition-matrix products
that act in each head either as the identity or as rotations of one plane
through the same angle. Their commutator, the product
$\mR_1\mR_2\mR_1^\top\mR_2^\top$, returns to the identity after ten
repetitions by \Cref{lem:planar-rotation-product}. Its decoded permutation
is a three-cycle, whose tenth power is not the identity. This contradiction
applies to any number of heads and does not require an injective decoder.
We prove the rotation lemma after the main argument.

\begin{proof}[Proof of Theorem~\ref{thm:s5-impossibility}]
Suppose a tracker exists. By \Cref{prop:independent-head-reduction},
we may remove the additive terms and restrict each head to orthogonal
transitions, without increasing eigenvalue multiplicities. The resulting
block-diagonal matrices generate a finite group $\Gamma$, with a
surjective homomorphism $p:\Gamma\to S_5$ satisfying $p(\mA_g)=g$.
This reduction applies to one head as well as to several.

\emph{Construct the rotations in each head.}
Let $c=(12345)$ and $t=(12)$. Since $p(\mA_c)=c$, the order of
$\mA_c$ is divisible by five; write it as $5^a r$, with $a\geq1$ and $5$ does not divide $r$.
Choose $E>0$ divisible by $2r$ and congruent to $1$ modulo $5$;
such a choice exists because $2r$ is coprime to $5$. Set
\[
    \mR_1=\mA_c^E,\qquad
    \mR_2=\mA_t\mR_1\mA_t^\top.
\]
The even exponent turns all real eigenvalues $\pm1$ into $1$, and its
factor $r$ removes every order factor coprime to five. Thus each head
of $\mR_1$ is either the identity or a planar rotation of nontrivial
power-of-five order. The corresponding head of $\mR_2$ is its orthogonal
conjugate, with the same angle and order. Meanwhile,
\[
    p(\mR_1)=c^E=c,\qquad
    p(\mR_2)=tct^{-1}=(13452)=:b.
\]

\emph{Compare the matrix product with its decoded permutation.}
Let $\mR_{j,h}$ denote the block of $\mR_j$ in head $h$, and define
\[
    \mC_h=\mR_{1,h}\mR_{2,h}\mR_{1,h}^\top\mR_{2,h}^\top,
    \qquad
    \mC=\mR_1\mR_2\mR_1^\top\mR_2^\top
        =\operatorname{diag}(\mC_1,\ldots,\mC_H).
\]
Each pair of head rotations belongs to a finite group, as the image
of $\Gamma$ on that head. By \Cref{lem:planar-rotation-product},
$\mC_h^{10}=\mI$ in every nonidentity head, and the same holds in
identity heads. Hence $\mC^{10}=\mI$. But direct permutation composition
gives $p(\mC)=cbc^{-1}b^{-1}=(142)$, so
\[
    e=p(\mC^{10})=(142)^{10}=(142)\ne e,
\]
a contradiction.
\end{proof}

\begin{lemma}[A constraint on two planar rotations]
\label{lem:planar-rotation-product}
Let $\mR_1,\mR_2\in O(n)$ belong to a finite group. Suppose they are
planar rotations through the same angle, possibly in different planes,
and their common order is a multiple of five. Angles are measured in
$[0,\pi]$, ignoring direction. Then
\[
    (\mR_1\mR_2\mR_1^\top\mR_2^\top)^{10}=\mI.
\]
\end{lemma}
\begin{proof}
\emph{First prove a three-dimensional version.}
Let $\mX,\mY\in SO(3)$ satisfy the lemma's assumptions, and write
$\theta$ for their common angle. In particular, their traces are equal:
\[
    \operatorname{tr}\mX=\operatorname{tr}\mY=1+2\cos\theta.
\]
We claim that $\mD:=\mX\mY\mX^\top\mY^\top$ satisfies $\mD^5=\mI_3$.
If their axes coincide, the rotations commute and $\mD=\mI_3$.
If their axes differ, we use the finite rotation-group classification
recalled before \Cref{thm:a5-three}. In a cyclic or dihedral rotation
group, all rotations of order greater than two share an axis.
Tetrahedral and octahedral rotations have orders at most four.
Therefore only the rotations of an icosahedron remain possible.

In this case, $\mX,\mY$ rotate about distinct axes through opposite vertices,
and $\theta$ is $2\pi/5$ or $4\pi/5$. We now choose these axes explicitly.
Let $\varphi=(1+\sqrt5)/2$. Place the icosahedron at the origin with vertices
\[
    (0,\pm1,\pm\varphi),\qquad
    (\pm1,\pm\varphi,0),\qquad
    (\pm\varphi,0,\pm1),
\]
where the signs are independent. Each pair of opposite vertices defines
one of the six vertex axes. By symmetry, we may take the axis of $\mX$
through $(0,1,\varphi)$. Rotation by $2\pi/5$ about this axis cycles
the other five vertex axes, so we may take the axis of $\mY$ through
$(1,\varphi,0)$. Therefore we can choose unit axis vectors
\[
    \vu=\pm\frac{(0,1,\varphi)^\top}{\sqrt{1+\varphi^2}},\qquad
    \vv=\pm\frac{(1,\varphi,0)^\top}{\sqrt{1+\varphi^2}}.
\]
Choose their signs so that $\mX,\mY$ both rotate through $\theta$
with the right-hand rule. The signs do not affect the squared inner product:
\[
    (\vu^\top\vv)^2
    =\frac{\varphi^2}{(1+\varphi^2)^2}=\frac15,
\]
where the last equality follows from $\varphi^2=\varphi+1$ and
$(1+\varphi^2)^2=5\varphi^2$.

It remains to show that $\operatorname{tr}\mD=1+2\cos\theta$.
Write $x=\cos\theta$.
Rodrigues' formula expresses the two rotations as
\[
    \mX=x\mI_3+(1-x)\vu\vu^\top+\sin\theta[\vu]_\times,
    \qquad
    \mY=x\mI_3+(1-x)\vv\vv^\top+\sin\theta[\vv]_\times,
\]
where $[\vw]_\times\vz=\vw\times\vz$.
Substituting into $\mD=\mX\mY\mX^\top\mY^\top$ and taking the trace gives
\[
\begin{aligned}
    \operatorname{tr}\mD
    =\left[2-(1-x)^2\bigl(1-(\vu^\top\vv)^2\bigr)\right]^2-1
    =\left[2-\tfrac45(1-x)^2\right]^2-1.
\end{aligned}
\]
Thus the factor $4/5$ comes from the angle between the two vertex axes.
It remains to substitute the rotation angle: since
$\theta\in\{2\pi/5,4\pi/5\}$, we have $x=(-1\pm\sqrt5)/4$.
Both values satisfy $4x^2+2x-1=0$, which implies
$\tfrac45(1-x)^2=1-2x$. Therefore
\[
    \operatorname{tr}\mD
    =(1+2x)^2-1=1+2x
    =\operatorname{tr}\mX=\operatorname{tr}\mY,
\]
where the second equality again uses $4x^2+2x-1=0$.
For a three-dimensional rotation, the trace $1+2\cos\theta$ uniquely
determines the angle in $[0,\pi]$. Hence $\mD$ also rotates through
$2\pi/5$ or $4\pi/5$, so $\mD^5=\mI_3$ as claimed.

\emph{Apply this calculation to the original planar rotations.}
Their two rotating planes span a space $U$ of dimension at most four.
Both matrices preserve $U$ and leave every vector in $U^\perp$ unchanged. We can
therefore restrict to $U$ and, if needed, add coordinates on which both
matrices act as the identity. This gives matrices in $SO(4)$ with the
same angles and orders, still generating a finite group.

The proof of \Cref{prop:so3-four-dimensional-encoding} constructs two
maps $f,f_-:SO(4)\to SO(3)$ that preserve multiplication. Each sends
a planar rotation to a three-dimensional rotation with the same
angle: in a basis aligned with the rotating plane, the images rotate
through $\theta$ and $-\theta$, respectively. Moreover, if both maps
send a matrix to the identity, that matrix must be $\mI_4$ or $-\mI_4$:
\[
    f(\mP)=f_-(\mP)=\mI_3\quad\Longrightarrow\quad
    \mP\in\{\mI_4,-\mI_4\}.
\]
For either map, the images of $\mR_1,\mR_2$ thus satisfy the
three-dimensional claim. Writing
$\mC=\mR_1\mR_2\mR_1^\top\mR_2^\top$, we obtain
\[
    f(\mC^5)=f_-(\mC^5)=\mI_3.
\]
Therefore $\mC^5=\pm\mI_4$, and squaring gives $\mC^{10}=\mI_4$.
Since the original matrices fix $U^\perp$, the same identity holds
in the original dimension.
\end{proof}

\paragraph{Consequences for CKDA and Gated DeltaProduct.}
With unit keys, CKDA transitions with $\beta\in[0,2]$ and
$\alpha_i\in[-1,1]$ are non-expansive, and
\Cref{thm:dplr-persistent-pairs-app} bounds their non-real unit-circle pairs
by one. Gated DeltaProduct transitions with $\beta_j\in[0,2]$ and scalar
$\alpha_g\in[0,1]$ are also non-expansive. Expanding the product of
$k$ Householder transformations gives
\[
    \rank(\mA_g-\alpha_g\mI)\leq k.
\]
Since $\alpha_g$ is real, at most $k$ eigenvalues can be non-real,
counted with algebraic multiplicity. They occur in conjugate pairs,
so $k\leq3$ allows at most one such pair per head.
\Cref{thm:s5-impossibility} therefore rules out both architectures under
finite reachability in each head, including input-dependent additive matrices and arbitrary joint
decoders.

For sufficiency at $k=4$, use the permutation-matrix construction
of \citet[Thm.~1]{siems2025}. Every permutation of five objects is a
product of at most four transpositions. Each transposition $(ij)$ has
matrix $\mI_5-2\vk_{ij}\vk_{ij}^\top$, where
$\vk_{ij}=(\ve_i-\ve_j)/\sqrt2$ is a unit key. Thus four Householder
transformations suffice, padding shorter products with $\beta_j=0$ and choosing
$\beta_j=2$ for the transpositions. Set the scalar gate to $\alpha=1$,
the additive term to zero at every step, and the initial state to $\mI_5$.
The state is then exactly the permutation matrix of the running group
product, so a single head tracks $S_5$ with precisely $120$ reachable
states and orthogonal transitions. Any $k>4$ is also sufficient by
padding with identities. Together with the lower bound, this proves that
four is the minimum number of Householder transformations needed for
single-layer Gated DeltaProduct tracking of $S_5$ under these assumptions.

\section{Three-layer simulation of finite and weighted automata}\label{app:wfamin}
\providecommand{\vc}{{\bm{c}}}
We prove \Cref{thm:multilayer} using the tasks and arithmetic conventions
of \Cref{app:tracking-tasks,app:tracking-precision}. The construction has
three recurrent layers: a clock, a buffer, and an accumulator. CKDA
implements the clock in one layer, saving one layer relative to the cited
four-layer DeltaNet construction.

\subsection{The clock, buffer, and accumulator}\label{app:automata-mechanism}
The idea is to give the recurrence several steps to apply a complicated
transition. Choose $m$ large enough that any block transitions  (the product of all transition matrices can be
factored into at most $m$ permitted matrices. Following
\citet[Thm.~3 and its proof]{siems2025}, collect $m$ consecutive inputs
and factor their combined transition, padding with identities when
necessary. Apply these
factors one at a time while reading the next block:
\begin{itemize}
 \item The \emph{clock} records the position modulo $2m$, telling the
 later layers which buffer slot and matrix factor to use.
 \item The \emph{buffer} retains the last $2m$ tokens. Its feed-forward
 lookup selects the factors for the previous block while retaining the
 inputs arriving in the current block.
 \item The \emph{accumulator} applies the selected factor at each step,
 carrying forward the product from earlier blocks.
\end{itemize}
For example, with $m=3$, inputs $4,5,6$ provide three steps to apply the
three factors representing inputs $1,2,3$. The readout corrects this
delay: it completes the unapplied factors and then applies the current
block's prefix. The first block is handled directly by lookup.

The same organization applies to the column-state matrix products in
\Cref{app:tracking-tasks}. The buffer and completion readout are those in
\citet[Appendix~D]{peng2025rwkv7} and
\citet[Thm.~3, Lems.~2--3]{siems2025}, extended to rational WFAs by
\citet[Thms.~6 and~11]{merrill_why_2026}; we refer to these works for
their constructions. The buffer uses DeltaNet updates, which are also
CKDA updates with identity diagonal gate. Below we give the
CKDA clock, the factor counts determining $m$, and the resulting
simulation argument.

\subsection{A one-layer clock}\label{app:automata-clock}
DeltaNet obtains the clock rotation from two alternating reflections,
using a parity layer to select between them
\citep[Lem.~7]{merrill_why_2026}. CKDA combines the two reflections
within one transition: its signed diagonal supplies one of them.

\begin{lemma}[One-step modular counter]\label{lem:wfamincount}
For every integer $m\geq1$ there is a \emph{one}-layer CKDA head whose readout determines
$i\bmod 2m$ at every position $i$.
\end{lemma}
\begin{proof}
With $\beta=2$, $\valpha=(1,-1)$ and
$\vk=(-\sin\tfrac{\varphi}{2},\cos\tfrac{\varphi}{2})^{\!\top}$ the transition
$(\mI-2\vk\vk^{\!\top})\Diag(\valpha)$ is a rotation by $\varphi$. Set
$\varphi=\pi/m$ and $\vc_0=\ve_1$. Then
$\vc_i=(\cos(\pi i/m),\sin(\pi i/m))^{\!\top}$ visits $2m$ distinct points.
Choose a fixed unit query $\vq$ such that
$\vq^\top(\vc_i-\vc_j)\ne0$ for $0\leq i<j<2m$. Such a query
exists because only finitely many lines are excluded. The single scalar
readout $\vq^\top\vc_i$ then distinguishes all clock states, and a
finite lookup returns the residue.
\end{proof}

This is similar to the  DeltaProduct$_k$ $k \geq 2$ counter
in \citet[Lem.~2(ii)]{siems2025}. Its rotation parameters are algebraic;
the fixed query can also be chosen algebraic by the same argument as in
\Cref{app:tracking-realization}. The clock uses the fixed exact datatype
of \Cref{app:tracking-precision}.

\subsection{Accumulator factorizations}\label{app:automata-factors}
The accumulator applies one matrix factor per recurrent step. We therefore
bound the number of factors required for permutation, deterministic, and
rational transitions. These bounds determine the block length $m\geq1$
in \Cref{app:automata-mechanism}: any combined block transition must admit
a factorization with at most $m$ factors.

\begin{proposition}[Factor-count upper bounds for the accumulator]\label{prop:automata-factor-counts}
Let $n\geq1$ be the automaton-state dimension. A deterministic transition
is an $n\times n$ matrix with one entry equal to $1$ in each column and
all other entries zero. The following numbers of recurrent matrix factors
are sufficient; these construction-dependent bounds need not be optimal:
\begingroup
\setlength{\arraycolsep}{3pt}
\[
\begin{array}{lccc}
\toprule
\text{matrix family or primitive} & \text{RWKV-7} & \text{DeltaNet/GDN} & \text{CKDA}\\
\midrule
\text{one transposition} & 1 & 1 & 1\\
\text{arbitrary permutation} & n-1 & n-1 & n-1\\
\text{one DFA column-copy primitive} & 1 & 4 & 2\\
\text{deterministic transition, explicit construction} & n & 4n & 2n-2\\
\text{deterministic transition, up to positive scale} & n & 3n & \max\{1,2n-2\}\\
\text{arbitrary rational }n\times n\text{ matrix} & 2n & (8n^2+5n+1)^{\dagger} & 7n^2+3n\\
\bottomrule
\end{array}
\]
\endgroup
The GDN entries $4$ and $4n$ retain a conservative allowance of four
factors per DFA column-copy primitive, defined below.
The positive-scale row uses an exact-real singular value decomposition
(SVD) for GDN and CKDA: it factors $\rho\mM$ for some $\rho>0$, where
$\mM$ is the transition, and gives no fixed rounded-datatype guarantee.
The rational-matrix row permits scratch coordinates (auxiliary storage):
RWKV-7 operates in dimension $2n$, and DeltaNet/GDN and CKDA in dimension
$2n+1$. The dagger marks the cited DeltaNet/GDN budget with unrestricted
unit-key learning rate $\beta\in\mathbb R$; the CKDA budget uses
$\beta\geq0$ and diagonal entries in $[-1,1]$.
The explicit CKDA deterministic construction needs only $\beta\in[0,3]$;
permutations use $\beta=2$ reflections. A product with no factors is the
identity; the swap and column-copy rows require $n\geq2$.
\end{proposition}
\begin{proof}
Write $\ve_j$ for the $j$th standard basis vector and $\mI$ for the
identity in the current state dimension. For a scalar $\gamma$ and vector
$\vu$, write $\mathcal H_\gamma(\vu)=\mI-\gamma\vu\vu^{\!\top}$.
A CKDA factor $\mathcal H_\gamma(\vu)\mD$ applies a diagonal gate $\mD$
first and a generalized Householder transformation second. For $\vu\ne0$,
its unit key is $\vk=\vu/\|\vu\|_2$ and its learning rate is
$\beta=\gamma\|\vu\|_2^2$. Rational
transition entries do not imply rational normalized keys.

\textit{Transpositions and permutations.}
A transposition (a swap of coordinates $i\ne j$) is
$\mathcal H_1(\ve_i-\ve_j)$, so a permutation costs at most $n-1$ factors
\citep[Thm.~3, proof]{siems2025}. RWKV-7 also implements each swap in one factor
\citep[Lem.~4]{peng2025rwkv7}. For CKDA, this worst-case count is sharp by the
$n$-cycle example in \Cref{prop:CKDA-factorization}.

\textit{Deterministic transitions: explicit construction.}
For distinct coordinates $s,d$, the DFA \emph{column-copy} primitive
merges state $d$ into state $s$: on a column state $\vx$, it sends
$(x_s,x_d)$ to $(x_s+x_d,0)$
and leaves other coordinates unchanged. Set
$\mD_d=\mI-\ve_d\ve_d^\top$, the diagonal gate that clears coordinate $d$.
The merge has the factorization
\[
 \mM_{d\to s}:=\mI-\ve_d\ve_d^\top+\ve_s\ve_d^\top
 =\bigl[\mathcal H_3(\ve_s)\mD_d\bigr]
   \mathcal H_{1/6}(3\ve_s+\ve_d).
\]
In the $(s,d)$ coordinates, this identity is simply
\[
 \begin{pmatrix}-2&0\\0&0\end{pmatrix}
 \begin{pmatrix}-1/2&-1/2\\-1/2&5/6\end{pmatrix}
 =\begin{pmatrix}1&1\\0&0\end{pmatrix}.
\]
The effective learning rates are $3$ and $5/3$, so each column copy uses
two CKDA factors. Separating $\mD_d=\mathcal H_1(\ve_d)$ gives three GDN
factors. RWKV-7 uses one factor per identity, swap, or column copy and at
most $n$ such primitives per deterministic transition
\citep[Lems.~3--4]{peng2025rwkv7}, yielding the RWKV-7 and GDN bounds.

For CKDA, use one merge per noncycle vertex and $\ell-1$ swaps per cycle
of length $\ell$. To see this, let $f:\{1,\ldots,n\}\to\{1,\ldots,n\}$
be the deterministic state map of the underlying deterministic automaton, so
$\mM\ve_i=\ve_{f(i)}$. In the graph with arrows
$i\to f(i)$, every component is a cycle with trees feeding into it.
Let $k$ count the vertices on cycles and $c$ count the cycles, including
fixed points. Right-multiplying by $\mM_{d\to s}$ replaces column $d$
by column $s$. Thus, starting from $\mI$, right-multiply by
$\mM_{i\to f(i)}$ for each noncycle vertex $i$, processing vertices in
decreasing distance from their cycle. This replaces column $i$ by column
$f(i)$, which is still the original $\ve_{f(i)}$. The remaining cycle
columns are then arranged by swaps within each cycle, leaving completed
noncycle columns unchanged. The column-state recurrence applies the
resulting factors from right to left.
Since a merge costs two CKDA factors and a swap costs one, the total is
\[
  \underbrace{2(n-k)}_{\text{noncycle merges}}
  +\underbrace{(k-c)}_{\text{cycle swaps}}
  =2n-k-c\leq 2n-2,
\]
because $k\geq c\geq1$. For $n=1$, the transition is the identity and
needs no factors.

\textit{Deterministic transitions: exact-real SVD.}
Choose $0<\rho\leq1/\max(1,\|\mM\|_2)$, where $\|\cdot\|_2$ is the
spectral norm, and write the SVD $\rho\mM=\mU\mSigma\mV^{\!\top}$.
Here $\mU,\mV$ are orthogonal and $\mSigma$ is diagonal with entries in
$[0,1]$. Positive scaling preserves the position of the nonzero entry in
a one-hot DFA state, so that state remains exactly decodable. For GDN,
the two orthogonal factors cost at most $n$ reflections each, and the
diagonal costs another $n$ generalized Householders, giving $3n$ factors
\citep[Prop.~1, item~2]{grazzi2025}.
For CKDA, \Cref{prop:CKDA-factorization} applies to the
contraction $\rho\mM$ and gives at most $\max\{1,2n-2\}$ factors by absorbing
the signed diagonal contraction into an adjacent Householder factor.

\textit{Arbitrary rational matrices.}
Let $\mM\in\mathbb Q^{n\times n}$ be the target matrix and
$\vx\in\mathbb R^n$ the input state. The RWKV-7 entry follows from the
$2n$-factor construction in \citet[Lems.~5--6]{merrill_why_2026}.
In their row-state convention, the
factor product is
$\left(\begin{smallmatrix}\mM&\mM\\0&0\end{smallmatrix}\right)$.
Applying these same factors in reverse execution order to a column state
maps $(\vx,0)$ to $(\mM\vx,0)$, so scratch coordinates remain zero at
block boundaries.

For CKDA, adapt the $8n^2+5n+1$-factor DeltaNet program of
\citet[Lems.~11--12]{merrill_why_2026}, using $n$ main coordinates $\vx$,
$n$ scratch coordinates $\vs$, and one temporary coordinate $t$.
First clear $\vs$ and $t$.
A \emph{unit coordinate addition} (transvection) adds coordinate $s$ to
coordinate $d\ne s$ of this augmented state, leaving all others unchanged.
It uses three factors:
\[
 \mI+\ve_d\ve_s^\top
 =\mathcal H_{1/3}(\ve_s+2\ve_d)
  \mathcal H_{1/2}(\ve_s)\mathcal H_2(\ve_s+\ve_d).
\]
This is the column-state transpose of \citet[Lem.~10]{merrill_why_2026}.
Each scaled addition $s_i\leftarrow s_i+M_{ij}x_j$, for $1\leq i,j\leq n$,
uses $t\leftarrow t+x_j$, $t\leftarrow M_{ij}t$, $s_i\leftarrow s_i+t$,
then $t\leftarrow0$, where $M_{ij}$ is the $(i,j)$ entry of $\mM$.
CKDA absorbs this clear into the diagonal gate of the
next Householder factor, reducing the cost from eight to seven factors.

For a rational scaling coefficient $a>0$, define the coordinate sign flip
$\mS_j=\mI-2\ve_j\ve_j^\top$ and use
$\mathcal H_{1+a}(\ve_j)\mS_j$ to scale coordinate $j$ by $a$.
For $a<0$, use $\mathcal H_{1+|a|}(\ve_j)$, and for $a=0$ use
$\mathcal H_1(\ve_j)$. Each is one admissible CKDA factor with
nonnegative learning rate. The cited GDN scaling
$\mathcal H_{1-a}(\ve_j)$ instead needs a negative learning rate for $a>1$
\citep[Lem.~9]{merrill_why_2026}.

The initial clear of scratch and temporary coordinates is absorbed into
the first Householder factor. At the end, the final temporary clear and
all main-coordinate clears are absorbed into the first of $n$ three-factor
unit additions that copy scratch back to main. The total is therefore
\[
 \underbrace{n^2(3+1+3)}_{\text{scaled additions; clears absorbed}}
 +\underbrace{3n}_{\text{copy back}}
 =7n^2+3n.
\]
The program maps $(\vx,\vs,t)$ to $(\mM\vx,\mM\vx,0)$ for arbitrary
initial scratch and temporary contents.
\end{proof}

The DFA simulation below uses the explicit CKDA construction and the
fixed exact datatype of \Cref{app:tracking-precision}.

\subsection{Three-layer expressivity}\label{app:automata-simulation}
We now combine the clock with the cited buffer and completion readout.
The exact finite-datatype and polynomial-precision conventions below
are those defined in \Cref{app:tracking-precision}.

\begin{theorem}[Three layers suffice; restatement of Theorem~\ref{thm:multilayer}]\label{prop:wfamin}
Three CKDA layers solve every finite group-word problem under the
exact finite-datatype convention above. Allowing $\beta>2$ also gives
every regular language under that convention and every WFA over
$\mathbb Q$ in polynomial precision, using exact arithmetic over the
fixed algebraic number field specified above.
\end{theorem}
\begin{proof}
Use the clock--buffer--accumulator construction of
\citet[Thm.~11]{merrill_why_2026}, summarized in
\Cref{app:automata-mechanism}, replacing its two clock layers by
\Cref{lem:wfamincount}. The
remaining layers stream the permutation, DFA, or rational WFA factors
from \Cref{prop:automata-factor-counts}, with
$m=\max(1,n-1)$, $m=\max(1,2n-2)$, or $m=7n^2+3n$, respectively.
The exact finite-datatype and polynomial-storage guarantees follow from
\Cref{app:tracking-precision}.
\end{proof}

\subsection{Comparison with other transition families}\label{app:table-comparison}
\Cref{tab:theory-summary} compares transition families under the expressive
decoder and precision conventions used in the cited constructions. Its
group rows allow every group element as an input token. In particular,
the one-layer swap-tracking result of \citet[Lem.~2]{peng2025rwkv7} does
not give one-layer tracking of arbitrary $S_5$ inputs. The displayed
depths are upper bounds from constructions; a one-layer obstruction and
a three-layer construction leave the two-layer case undecided.

\paragraph{Spectral restrictions.}
For scalar-gated DeltaProduct, write
$\mA=\alpha\prod_{j=1}^{k}(\mI-\beta_j\vk_j\vk_j^\top)$.
Every vector orthogonal to all keys is an eigenvector with real eigenvalue
$\alpha$. Therefore at most $\min(d,k)$ eigenvalues can be non-real,
giving at most $\lfloor\min(d,k)/2\rfloor$ conjugate pairs.
For RWKV-7 and GDN-2 with nonnegative decay and erase gates, the transition
has the form $\mD-\vu(\vv\odot\vu)^\top$ with $\vv\geq0$.
When $\vv>0$, conjugation by $\Diag(\sqrt{\vv})$ makes this matrix
symmetric; zero entries follow by continuity. Hence its spectrum is real,
and the one-layer impossibility result of
\citet[Thm.~2 and App.~B.2]{grazzi2025}, under their finite-precision
assumptions, applies to every group row containing an element of order
greater than two.
The unrestricted WFA parameterizations are not subject to this
nonnegative-gate assumption.

\paragraph{RWKV-7 and GDN-2.}
The shared RWKV-7/GDN-2 entries are transfers between transition families,
not additional theorems claimed by the GDN-2 paper. For the finite-state
constructions, the table uses unit keys, decay entries in $[0,1]$, and
erase entries in $[0,2]$, including the endpoints. In RWKV-7 notation,
this absorbs the factor $c=2$ into the erase gate. The $c=1$ adaptation
in \citet[Appendix~D.3]{peng2025rwkv7} instead rescales the state and
uses normalization and growing exponent storage; our fixed-datatype
entries refer to the $c=2$ construction.
Both contain DeltaNet
when the decay is one and the erase gate is constant. This transfers the
group constructions of \citet[Thms.~1 and~7]{siems2025}.
For GDN-2, use its update in \citet[Eq.~(10)]{hatamizadeh2026gdn2} with
the erase range extended to $[0,2]$. It also implements the DFA column-copy primitive that sends state $j$ to
state $i$: take $\mD=\mI$, $\vk=(\ve_i-\ve_j)/\sqrt2$, and
erase gate $2\ve_j$, obtaining $\mI+(\ve_i-\ve_j)\ve_j^\top$.
Thus it supports the identity, swap, and copy primitives in the four-layer
regular-language construction of \citet[Thm.~3]{peng2025rwkv7}.
The WFA entry instead uses unrestricted parameters, as in
\citet[Thm.~6]{merrill_why_2026}; GDN-2 inherits their DeltaNet construction
by allowing an unrestricted constant erase gate.

\paragraph{DeltaProduct and automata bounds.}
The DeltaProduct automata entries combine existing constructions.
\citet[Lem.~2(ii)]{siems2025} provide its one-layer clock, which replaces
the two clock layers in the four-layer DeltaNet simulation of
\citet[Thm.~11]{merrill_why_2026}. This gives three layers when $k\geq2$;
when $k=1$, the original four-layer bound applies.
For the one-layer entries, \citet[Lem.~3]{peng2025rwkv7} factor a
$q$-state deterministic transition into $q$ identity, swap, or copy
operations. The DFA primitive here copies a column of the identity, not a coordinate of
a column-state vector. The identity in \Cref{prop:automata-factor-counts}
implements it using three Householder factors, so the existing $4q$
factor allowance remains sufficient. The largest effective unit-key learning
rate in that identity is $3$; in particular, $\beta\in[0,4]$ suffices for
the stated regular-language bounds.
\citet[Lem.~12]{merrill_why_2026} gives
$B_q=8q^2+5q+1$ factors for a rational $q\times q$ matrix with scratch
coordinates and unrestricted, possibly negative, learning rates. A DeltaProduct transition with enough factors can apply
these entire products in one step. Thus these table entries are
consequences of the cited results, not new factorization or simulation
claims. In particular, \citet[Thm.~6]{siems2025} concerns products of
RWKV-7 matrices, so it is not a direct source for the Householder-product
regular-language bound.
When all learning
rates remain in $[0,2]$, the regular-language result of
\citet[Thm.~2]{siems2025} instead gives a finite depth that depends on the
language; the scalar gate supplies resets. All finite-precision positive
entries use the exact finite-datatype convention of \Cref{app:tracking-precision}, while
our WFA clock adaptations use exact arithmetic over a fixed algebraic
number field, with polynomial bit length as explained in \Cref{app:tracking-precision}. The WFA weights
and outputs are rational; this does not require all network parameters to be
rational.

\paragraph{The clock obstruction.}
Under the finite-precision assumptions of
\citet[Thm.~2 and Appendix~B.2]{grazzi2025}, a one-layer recurrence
with only real transition eigenvalues cannot count modulo $q>2$.
This applies to GDN with $\mA=\alpha(\mI-\beta\vk\vk^\top)$,
including $\alpha\in[-1,1]$, and to DeltaNet at $\alpha=1$.
In the cited
GDN construction, a parity layer followed by alternating reflections
supplies the clock. CKDA instead composes the two reflections within
one transition. This saves one layer relative to that construction; it
does not prove that every GDN automata simulation requires four layers.

\section{Efficient Implementation of the Signed Gate}\label{app:gauge}\label{app:kernels}
KDA stores decay magnitudes in log-space. To support signed gates, we absorb cumulative signs into the keys and queries, allowing the reuse of existing kernels in flash-linear-attention~\citep{yang2024fla}. Fusing these sign flips into key and query normalization then avoids separate transformation passes.

Kimi K3's safe sigmoid~\citep{team2026kimi} uses $g=-5\sigma(a)$ and $\alpha=\exp(g)$ for the scaled, biased gate preactivation $a$, ensuring $\alpha\geq\epsilon:=e^{-5}$ (indices suppressed). Our signed modification sets $r=\tanh(a/2)=2\sigma(a)-1$ and $\alpha=s[\epsilon+(1-\epsilon)\abs{r}]$, with $s=+1$ for $r\geq0$ and $s=-1$ otherwise, and computes $g=\log\abs{\alpha}\in[-5,0]$ in FP32 without additional clamping. At $a=0$, $\alpha=\epsilon$ and the signed gate is discontinuous; backpropagation detaches $s$ and uses PyTorch's zero subgradient for $\abs{r}$ at zero.

Write $\Diag(\valpha_i)=\mS_i\mD_i$, where $\mS_i=\Diag(\vsigma_i)$ with $\vsigma_i\in\{\pm1\}^n$ and $\mD_i=\Diag(\abs{\valpha_i})$; the sign of a zero gate entry can be chosen arbitrarily. To cancel the signs accumulated by the state, let $\mP_i:=\mS_1\cdots\mS_i$ with $\mP_0=\mI$. These diagonal sign matrices satisfy $\mP_i^\top=\mP_i=\mP_i^{-1}$.

\begin{lemma}[Sign absorption]\label{lem:gauge}
Let $\mH_i=(\mI-\beta_i\vk_i\vk_i^\top)\mS_i\mD_i\mH_{i-1}+\beta_i\vk_i\vv_i^\top$ with readout $\vo_i=\mH_i^\top\vq_i$. In the transformed coordinates $\tilde\mH_i:=\mP_i\mH_i$, $\tilde\vk_i:=\mP_i\vk_i$ and $\tilde\vq_i:=\mP_i\vq_i$, the same computation becomes
\[
\begin{aligned}
    \tilde\mH_i&=(\mI-\beta_i\tilde\vk_i\tilde\vk_i^\top)\mD_i\tilde\mH_{i-1}+\beta_i\tilde\vk_i\vv_i^\top,\\
    \vo_i&=\tilde\mH_i^\top\tilde\vq_i.
\end{aligned}
\]
Moreover, $\tilde\mH_0=\mH_0$, $\|\tilde\vk_i\|=\|\vk_i\|$, and $\mH_i=\mP_i\tilde\mH_i$.
\end{lemma}

\begin{proof}
Left-multiply the recurrence by $\mP_i$ and substitute $\mH_{i-1}=\mP_{i-1}\tilde\mH_{i-1}$. Since diagonal matrices commute, $\mS_i\mD_i\mP_{i-1}=\mP_i\mD_i$, so the transition becomes
\[
    \mP_i(\mI-\beta_i\vk_i\vk_i^\top)\mP_i\mD_i
    =(\mI-\beta_i(\mP_i\vk_i)(\mP_i\vk_i)^\top)\mD_i
    =(\mI-\beta_i\tilde\vk_i\tilde\vk_i^\top)\mD_i.
\]
The additive term becomes $\beta_i\mP_i\vk_i\vv_i^\top=\beta_i\tilde\vk_i\vv_i^\top$, and the readout satisfies $\mH_i^\top\vq_i=\tilde\mH_i^\top\mP_i\vq_i=\tilde\mH_i^\top\tilde\vq_i$. Orthogonality gives the norm identity and the inverse transformation; $\mP_0=\mI$ gives the initial condition.
\end{proof}

\paragraph{Relation to rotary embeddings.}
Absorbing cumulative transformations into keys and queries is familiar from RoPE and RetNet~\citep{su2024rope,sun2023retentive}, and from input-dependent rotary formulations of SSMs and Gated DeltaNet~\citep{lahoti2026mamba,movahedi2025srope}. Here, diagonal signs commute with arbitrary channel-wise decay, preserving all $n$ independent decay magnitudes. The transformation introduces no parameters or additional recurrent state.

\paragraph{Three implementations of the sign transformation.}
The implementation labeled \emph{gauge} uses compiled PyTorch operations, with no kernel changes. It computes cumulative signs and materializes $\tilde\vq_i=\mP_i\vq_i$ and $\tilde\vk_i=\mP_i\vk_i$ using \texttt{torch.compile}, then calls existing KDA kernels on $\abs{\valpha_i}$. It requires no kernel modification; the benchmark uses a TileLang-backed recurrence for this path. The Triton implementation computes signs by an integer parity scan and jointly normalizes and transforms keys and queries. This fusion uses $(\mP_i\vx)/\|\mP_i\vx\|=\mP_i(\vx/\|\vx\|)$; backward normalization and sign application are fused into the final query/key gradient writes. The TileLang-backed hybrid combines these Triton components with TileLang kernels for the WY backward computation; it is not an all-TileLang implementation. The algebraic recurrence is identical in all three cases. The final state is recovered as $\mH_i=\mP_i\tilde\mH_i$, and its incoming gradient is transformed by the same signs. Our models use equal key and value head counts; grouped value attention with value-head-specific signs requires expanding keys and queries.

\paragraph{Recurrence-kernel benchmark.}
\Cref{fig:kernel-throughput_recurrence} measures forward and backward on an H100 with BF16 inputs, 16 heads, and $d_k=d_v=128$, holding the number of logical tokens per step at $32{,}768$. Projections, optimizer updates, and final-state output are excluded. Each implementation is wrapped in \texttt{torch.compile}; kernel compilation and autotuning precede 20 warm-up iterations per shape and 50 timed iterations per round. Bars report mean throughput over four rounds. The normal and signed kernel paths both use normalization fusion, and DeltaProduct$_2$ uses its tuned configuration without a forget gate at the same head dimensions; its two updates per token are not counted as extra logical tokens. The timed recurrence processes two separately projected value inputs per token, one for each delta-rule update, so the throughput includes the recurrence work associated with this additional value input relative to CKDA. Computing the projections is outside the timed region. Triton and hybrid measurements use separate H100 allocations with the same default dot-precision policy, without explicitly setting \texttt{TRITON\_F32\_DEFAULT}. The PyTorch-gauge series comes from an earlier allocation without normalization fusion and is therefore an implementation reference rather than a fully matched fusion ablation. Note that we do not benchmark FlashKDA~\citep{flashkda2026} because it contains only inference kernels, while we benchmark the forward and backward pass.

\section{How the key activation affects learning group-word problems}\label{app:silu}

The key path of DeltaNet~\citep{yang2024deltanet}, inherited by KDA, is $\vx\to\mW_k\to\mathrm{SiLU}\to\ell_2$. Since $\mathrm{SiLU}\ge-0.2785$, normalization turns that floor into an angular prior: a unit key with a negative component of size $c$ requires $\norm{\mathrm{SiLU}(\vz)}\le0.2785/c$. Every unit-key direction remains reachable, so this is a possible optimization bias rather than a representational restriction. At $\beta=2$, $S_3$ needs no negative component at all: $\vk\in\{(\tfrac{\sqrt3}{2},0,\tfrac12),(\tfrac12,0,\tfrac{\sqrt3}{2}),\ve_3\}$ realizes all five non-identity elements, the two of order $3$ reusing the first two keys. The cube realization of $S_4$ in Appendix~\ref{app:cube} uses mixed-sign keys; this does not establish that every higher-dimensional realization or decoded tracker must do so.

\citet{dubinin2026} report that BD-LRU solves $S_4$ with much higher sample efficiency compared to DeltaProduct. But DeltaProduct shares the same key path as DeltaNet, while BD-LRU's entrywise gating has none to constrain. We hypothesize that the SiLU key activation contributes to optimization differences. The ablation in~\Cref{fig:group_word_problems_baselines} is group-dependent: removing SiLU improves the displayed $A_5$ result, whereas the displayed $S_4$ result is stronger with SiLU. It therefore does not establish SiLU as the cause of the cited $S_4$ gap.

\section{Experimental Details}\label{app:exp}
\subsection{Initialization}\label{app:initialization}
\begin{figure}[htbp]
    \centering
    \includegraphics[width=0.6\linewidth]{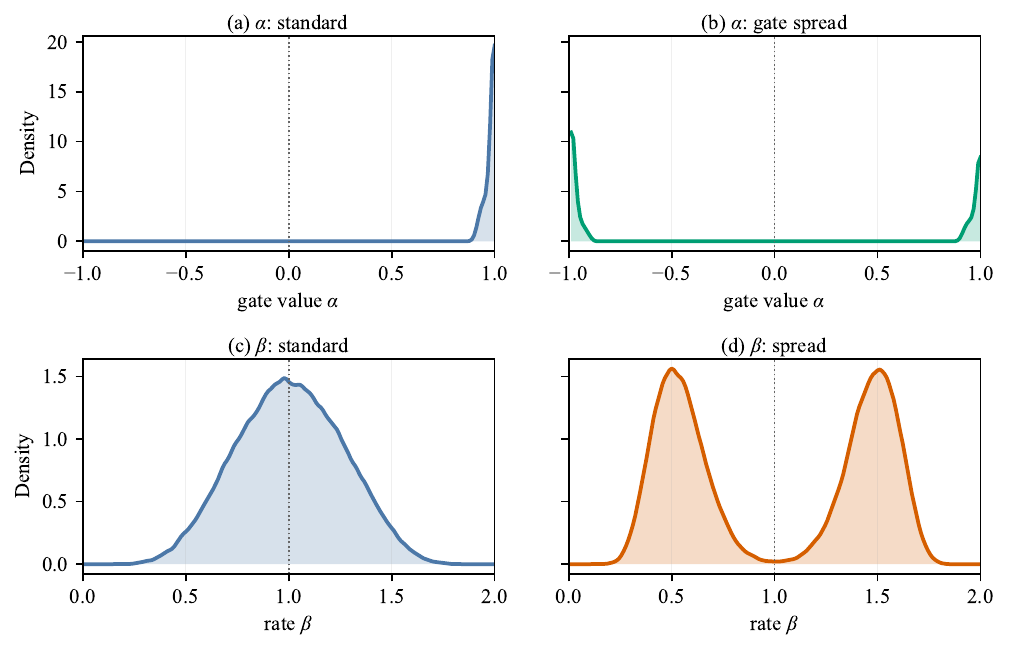}
    \vspace{-5mm}
    \caption{Standard and spread initializations for the signed extended-range layer.}
    \label{fig:initialization}
\end{figure}
The gate starts from a log-uniform decay parameter $d_t$ and applies the corresponding inverse parametrization, so that the initial magnitude is $|\alpha|=\exp(-d_t)$. For a signed gate, standard initialization uses $\alpha=+\exp(-d_t)$, whereas gate spread independently flips the sign of approximately half of the channels, $\alpha=s\exp(-d_t)$ with $s\in\{-1,+1\}$. The rate uses $\beta=c\,\sigma(b)$, where $\sigma$ is the sigmoid, $b$ is the rate preactivation, $c=1$ for the standard range, and $c=2$ for the extended range. Beta spread rescales the projection and adds opposite biases to different heads, $\beta=2\sigma(\gamma z\pm b_0)$, where $z$ is the unshifted projection output, producing two modes near $0.5$ and $1.5$. The resulting empirical distributions are shown in~\Cref{fig:initialization}.

\subsection{State-Tracking}\label{app:exp_state_tracking}
Standard models use one layer with 12 heads of dimension 16, trained for 60k steps with batch size 1024 and Muon~\citep{jordan2024muon} at learning rate $5\times10^{-3}$. We use the length curriculum $4,6,8,16,32$ and report the best of three seeds. The curriculum aids learning, consistent with~\citet{beck2024xlstm,siems2026learning}, and Muon improves it further. For a group $G$, scaled accuracy $(a-1/|G|)/(1-1/|G|)$ maps chance to 0 and perfection to 1; lengths above 32 test extrapolation.  

For the theory-initialized $A_5$ experiment, we use $A_5\cong2I/{\pm1}$ to assign each element a unit-quaternion lift $q_g\in\mathbb{R}^4$ following Appendix~\ref{app:A5const}. Two heads are initialized with $k_g=q_g$, $\valpha=(-1,1,1,1)$, and $\beta\approx2$, yielding $T_g\approx(I-2q_gq_g^\top)\operatorname{diag}(-1,1,1,1)$. One head suffices theoretically; two improve empirical robustness. The model has four heads of dimension $4$, hidden dimension $32$, and a standard MLP readout. All parameters, including initialized embeddings and KDA projections, remain trainable. We train for $3$k steps on all 60 elements using next-state cross-entropy and lengths $4,6,7,8,10,12,16,32$, without auxiliary representation loss, restricted-generator curriculum, or maxout readout.

The robustness limitation discussed in \Cref{sec:exp} does not preclude exact infinite-horizon tracking in exact arithmetic~\citep{chung2026rethinking}. Complementarily, \citet{dankowiakowski2026metric} show that linear-recurrence SSMs cannot robustly recognize all star-free languages, including FLIP-FLOP, whereas nonlinear xLSTM can. State-dependent nonlinear RNNs can implement error-correcting dynamics unavailable to affine recurrences~\citep{beck2024xlstm,poppel_flashrnn_2025,danieli2026pararnn,mishra2026m}.

\subsection{Periodic waveform continuation}
\label{app:waveform_continuation_details}
\paragraph{Data and task.} We synthesize a fixed two-bar groove at \(124\) BPM with \(32\) sixteenth-note frames, render stereo audio at \(4096\) Hz, convert it to mono, and resample each frame to \(64\) waveform values. The resulting periodic sequence \(\mathbf{Y}\in\mathbb{R}^{32\times64}\) is normalized independently along each waveform coordinate. Each training example is circularly shifted by a phase \(s\sim\operatorname{Uniform}\{0,\ldots,31\}\), with targets \(\mathbf{y}_t=\mathbf{Y}_{(s+t)\bmod32}\) and inputs
\[
    \mathbf{x}_t =
    \begin{cases}
        \mathbf{y}_t, & t<8,\\
        \mathbf{0}, & t\geq8.
    \end{cases}
\]
The model thus observes eight frames (one half-bar, approximately \(0.97\) seconds); inputs and targets have shape \(B\times L\times64\), and mean-squared error is computed only after the cue (\(t\geq8\)). Using shifts of one fixed groove isolates phase inference and periodic extrapolation rather than general-purpose audio generation.

\paragraph{Architectures.} All models use one sequence-mixing layer with hidden size \(128\), a bias-free \(64\rightarrow127\) input projection followed by an appended constant coordinate, and a shared \(128\rightarrow512\rightarrow64\) GELU readout. KDA models use eight heads with key and value dimensions \(16\), no short convolution, and no post-recurrent SiLU activation. We compare all four combinations of \(\alpha\in[0,1]\) or \([-1,1]\) and \(\beta\in[0,1]\) or \([0,2]\), using the corresponding spread initialization for extended ranges and default initialization for standard ranges. KDA models have approximately \(182\)k parameters; the GRU baseline has one \(128\)-dimensional recurrent layer and \(206\)k parameters. The causal Transformer has one eight-head pre-norm self-attention block, causal masking, sinusoidal positional encodings, and a \(128\)-dimensional feed-forward block, totaling \(207\)k parameters.

\paragraph{Training.} All models are trained from scratch for \(2500\) updates with batch size \(64\), seed \(0\), and the sequence-length curriculum \(12,16,24,40,72,136\). The curriculum occupies the first \(40\%\) of training, reaching length \(136\) at update \(835\) and retaining it thereafter. Muon optimizes two-dimensional matrices inside the sequence layer with learning rate \(0.02\), momentum \(0.95\), Nesterov momentum, and five Newton--Schulz iterations; AdamW optimizes the input projection, readout, biases, and other parameters with learning rate \(0.003\). Both learning rates use \(250\) warm-up updates followed by cosine decay to \(10\%\) of their initial values. Weight decay is \(10^{-12}\), and the global gradient norm is clipped to \(1\).

\paragraph{Evaluation.} After undoing training normalization, we average error over all \(32\) cue phases, continuation steps, and waveform coordinates at sequence lengths \(16,24,32,40,56,72,88,104,120,136,152,168,184,200,216,232,248,264\). Waveform signal-to-noise ratio is \(\operatorname{SNR}=10\log_{10}\bigl(\mathbb{E}[\mathbf{y}^{2}]/\mathbb{E}[(\hat{\mathbf{y}}-\mathbf{y})^{2}]\bigr)\), with both expectations taken over the continuation region. Qualitative waveforms in \Cref{fig:waveform_continuation} show phase \(0\) over steps \(0\)--\(24\) and the extrapolation window \(232\)--\(264\).

\subsection{Language modeling.}
\paragraph{Training.}
For the 1.3B parameter language modeling experiments, we follow the recipe of \citet{yang2025gdn, hatamizadeh2026gdn2}, training on 100B tokens of FineWeb-edu \citep{lozhkov2024fineweb-edu}, using the Llama-2 tokenizer~\citep{touvron_llama_2023}, a global batch size of $0.5M$ tokens, a linear warmup of $1B$ tokens to peak learning rate $4e-4$, followed by a cosine decay to $10\%$ of the peak learning rate all optimized with the AdamW  optimizer~\citep{loshchilov_decoupled_2019}. For our CKDA variants we use the \textit{standard} init with no SiLU on keys.
Note that our hybrid variants are using full attention at a 3:1 recurrent-to-attention ratio compared to the 1:1 recurrent + sliding window comparison.

For additional ablations on the impact of $\valpha$ and $\beta$ ranges as well as the impact of SiLU activation for keys, we train on Nemotron-CC~\citep{su2025nemotron} and a smaller version of FineWeb-edu. Also here, see Table~\ref{tab:lm_res}, CKDA outperforms an attention baseline as well as DeltaProduct~\citep{siems2025}, GDN~\citep{yang2025gdn} and vanilla KDA~\citep{kimi2025} on downstream evaluations. The small-model experiments use the GPT-2 tokenizer, with the model vocabulary padded to $65{,}536=2^{16}$ entries.

\begin{table}[h!]
\caption{Configurations used in the small-scale Nemotron-CC language modeling experiments. The intermediate size (the latent dimension of the SwiGLU MLP block) is reduced for GDN, DeltaProduct, and CKDA. We use AdamW with a global batch size of 64 and learning rate $5\times10^{-4}$, following a warmup--stable--decay (WSD) schedule: 2000 warmup steps and cooldown over the last $20\%$ of the token budget. Weight decay is set to $0.1$ and $(\beta_1, \beta_2)=(0.9, 0.95)$.}
\label{tab:arch_tab}\label{tab:model-configs}
\centering
\resizebox{\linewidth}{!}{%
\begin{tabular}{lccccc}
\toprule
\textbf{Config} & \textbf{Dense Attention} & \textbf{GDN} & \textbf{DeltaProduct} & \textbf{KDA} & \textbf{CKDA} \\
\midrule
Hidden size                 & 768   & 768   & 768   & 768   & 768   \\
Num.\ hidden layers         & 28    & 28    & 28    & 28    & 28    \\
Num.\ attention heads       & 12    & 12    & 12    & 12    & 12    \\
Intermediate size           & 3072  & 2816  & 2816  & 2816  & 2816  \\
Hidden activation           & SwiGLU & SwiGLU & SwiGLU & SwiGLU & SwiGLU \\
RoPE $\theta$                & 10{,}000 & N/A & N/A & N/A & N/A \\
Max.\ position embeddings   & 2048  & 2048  & 2048  & 2048  & 2048  \\
Vocabulary size             & $65{,}536$ & $65{,}536$ & $65{,}536$ & $65{,}536$ & $65{,}536$ \\
\midrule
Total parameters (millions)            & 341.55    & 342.33    & 342.07    & 344.87    & 344.87    \\
\bottomrule
\end{tabular}%
}
\end{table}

\begin{table}[h!]
\caption{Architectures of the scaling ladder, with the 1.3B/100BT FineWeb-Edu replication beside it. Every arm is parameter-matched to its column's target: the intermediate size (the latent dimension of the SwiGLU MLP) absorbs the difference between each mixer's own parameter budget, so the MLP width differs between arms and the totals are very close. The ladder trains on Nemotron-CC under AdamW with $(\beta_1,\beta_2)=(0.9,0.95)$ and weight decay $0.1$; multiple global batch sizes / learning rates are possible for different total token budget. The FineWeb-Edu column follows the published Gated DeltaNet recipe instead, which is why its vocabulary, head dimension, tying and schedule differ. }
\label{tab:ladder-configs}
\centering
\resizebox{\linewidth}{!}{%
\begin{tabular}{lccccccc}
\toprule
\textbf{Config} & \textbf{47M} & \textbf{124M} & \textbf{302M} & \textbf{588M} & \textbf{983M} & \textbf{1.7B} & \textbf{1.3B (FineWeb-Edu)} \\
\midrule
Hidden size & 384 & 576 & 896 & 1280 & 1536 & 2048 & 2048 \\
Num.\ hidden layers & 12 & 18 & 20 & 20 & 24 & 24 & 24 \\
Num.\ attention heads & 6 & 9 & 14 & 20 & 24 & 32 & 16 \\
Head dimension & 64 & 64 & 64 & 64 & 64 & 64 & 128 \\
Vocabulary size & 50{,}304 & 50{,}304 & 50{,}304 & 50{,}304 & 50{,}304 & 50{,}304 & 32{,}000 \\
Tied embeddings & yes & yes & yes & yes & yes & yes & no \\
Training corpus & Nemotron-CC & Nemotron-CC & Nemotron-CC & Nemotron-CC & Nemotron-CC & Nemotron-CC & FineWeb-Edu \\
LR schedule & WSD & WSD & WSD & WSD & WSD & WSD & cosine \\
Peak learning rate & 0.002 & 0.001 & 0.001 & 0.0005 / 0.001 & 0.0005 & 0.0005 & 0.0004 \\
Global batch size (sequences) & 32 / 64 & 64 & 64 / 128 & 64 / 128 / 256 & 64 / 128 / 256 / 512 & 64 / 128 / 256 / 512 & 128 \\
Warmup steps & 2000 & 2000 & 2000 & 2000 & 2000 & 2000 & 1908 \\
Cooldown fraction of budget & 20\% & 20\% & 20\% & 20\% & 20\% & 20\% & -- \\
\addlinespace
\multicolumn{8}{l}{\textit{Intermediate size}} \\
\quad Transformer++ & 1536 & 2304 & 3584 & 5120 & 6144 & 8192 & 5632 \\
\quad KDA (safe sigmoid) & 1472 & 2240 & 3520 & 4992 & 6016 & 8064 & 5440 \\
\quad CKDA & 1472 & 2240 & 3520 & 4992 & 6016 & 8064 & 5440 \\
\quad KDA + attn 3:1 & 1408 & 2176 & 3456 & 4928 & 5952 & 7936 & 5312 \\
\quad CKDA + attn 3:1 & 1408 & 2176 & 3456 & 4928 & 5952 & 7936 & 5312 \\
\addlinespace
\multicolumn{8}{l}{\textit{Total parameters (millions)}} \\
\quad Transformer++ & 47.64 & 124.55 & 302.01 & 588.73 & 983.31 & 1713.74 & 1364.30 \\
\quad KDA (safe sigmoid) & 48.03 & 125.45 & 303.66 & 586.32 & 980.00 & 1709.71 & 1362.63 \\
\quad CKDA & 48.03 & 125.45 & 303.66 & 586.32 & 980.00 & 1709.71 & 1362.63 \\
\quad KDA + attn 3:1 & 47.27 & 124.14 & 302.96 & 587.75 & 984.36 & 1712.29 & 1362.26 \\
\quad CKDA + attn 3:1 & 47.27 & 124.14 & 302.96 & 587.75 & 984.36 & 1712.29 & 1362.26 \\
\midrule
\multicolumn{8}{l}{\textit{Constant across every column}} \\
Hidden activation & \multicolumn{7}{c}{SwiGLU} \\
Context length & \multicolumn{7}{c}{4{,}096} \\
\bottomrule
\end{tabular}%
}
\end{table}

\paragraph{Language-model evaluation.}

\begin{table}[t]
\caption{Language modeling results. The best-performing configuration for each evaluation is highlighted in \textbf{bold}, and the second best is \underline{underlined}; markings consider only the Nemotron-CC rows. Nemotron-CC runs use 15B training tokens and FineWeb runs use 45B tokens; both blocks average the same ten accuracy tasks. Validation perplexity is measured on each block's own training distribution, and differences in training data and token budget prevent the cross-block comparison from isolating architecture effects. Abbreviations are expanded in Table~\ref{tab:lm_eval_names}. }
\centering
\resizebox{\textwidth}{!}{%
\begin{tabular}{l|c|cc|cccccccccc|c}
\toprule
\multirow{2}{*}{\textbf{Architecture}} & \multicolumn{3}{c|}{\textbf{Perplexity ($\downarrow$)}} & \multicolumn{11}{c}{\textbf{Accuracy ($\uparrow$)}} \\
\cmidrule(lr){2-4} \cmidrule(lr){5-15}
& \textbf{Val.} & \textbf{Wiki.} & \textbf{Lamb.} & \textbf{Hella.} & \textbf{W.grad} & \textbf{W.grande} & \textbf{Lamb.} & \textbf{COPA} & \textbf{PIQA} & \textbf{ARC-e.} & \textbf{ARC-c.} & \textbf{BBQA-Wiki} & \textbf{OBQA} & \textbf{Avg.} \\
\midrule
Dense Attention & 17.99 & \textbf{24.15} & 24.54 & 42.36 & 63.37 & \textbf{53.75} & \textbf{39.92} & \underline{70.00} & 68.44 & 59.81 & 29.61 & 48.79 & 33.00 & 50.90 \\
\midrule
DeltaProduct $k=2$ & 18.87 & 28.07 & 31.90 & 39.50 & 62.64 & 52.09 & 32.00 & 65.00 & 68.17 & 56.61 & 27.39 & 46.79 & 34.20 & 48.44 \\
DeltaProduct $k=2$, $\beta\in[0,2]$ & 19.04 & 28.38 & 30.43 & 39.40 & 61.54 & 50.67 & 33.20 & 64.00 & 66.65 & 57.45 & 27.47 & 46.85 & 33.60 & 48.08 \\
\midrule
GDN & 18.29 & 27.18 & 25.28 & 42.30 & 66.67 & 51.14 & 36.13 & 65.00 & 68.28 & 57.87 & 28.41 & 47.47 & 34.60 & 49.79 \\
GDN $\beta\in[0,2]$ & 18.10 & 27.22 & 25.21 & 42.15 & 66.30 & 51.54 & 36.00 & 63.00 & 67.79 & 58.12 & 29.10 & 47.40 & 35.00 & 49.64 \\
\midrule
KDA $\valpha\in[0,1] \ \beta\in[0,1]$ & \textbf{17.45} & 25.29 & \underline{19.67} & 43.87 & 66.30 & 52.72 & 39.71 & 67.00 & 68.55 & 59.81 & 30.03 & \textbf{51.21} & 34.00 & 51.32 \\
KDA $\valpha\in[0,1] \ \beta\in[0,2]$ & 17.55 & 25.51 & 20.53 & 43.79 & 63.74 & 52.57 & 39.10 & 66.00 & 68.34 & 59.05 & 30.12 & 49.48 & 33.60 & 50.58 \\
\midrule
KDA $\valpha\in[-1,1] \ \beta\in[0,2]$  (\textit{with ablations})  & & & & & & & & & & & & & &  \\

\hspace{16pt}{standard} & 17.74 & 25.86 & 22.50 & {43.22} & \textbf{69.96} & 51.38 & 38.13 & 62.00 & \textbf{69.75} & {60.40} & 28.75 & \underline{51.20} & \textbf{36.20} & {51.10} \\

\hspace{16pt}{spread} & \underline{17.49} & {25.34} & {20.21} & \textbf{44.04} & \underline{68.86} & \underline{53.67} & {39.67} & {68.00} & \underline{69.53} & \underline{61.32} & \textbf{31.23} & {51.10} & 35.60 & \textbf{52.30} \\
\hspace{16pt}{spread, no SiLU on keys} & \underline{17.49} & 25.30 & \textbf{19.64} & {43.65} & 65.93 & 53.12 & \underline{39.74} & \textbf{72.00} & 68.55 & {61.07} & 30.12 & 50.48 & 34.00 & \underline{51.87} \\
\hspace{16pt}{spread, $\beta\in[0,1]$} & \textbf{17.45} & \underline{25.26} &  20.17 & \underline{43.91} & 64.10 & 52.41 & 39.28 & 67.00 & 68.82 & \textbf{61.91} & \underline{30.29} & 50.34 & \underline{35.80} & 51.39 \\
\midrule
\multicolumn{15}{l}{\textit{FineWeb, 45B tokens}} \\
KDA $\valpha\in[0,1] \ \beta\in[0,2]$ & 17.02 & 23.46 & 13.66 & 47.08 & 64.84 & 54.22 & 45.29 & 68.00 & 70.84 & 52.82 & 26.02 & 54.32 & 32.60 & 51.60 \\ %
KDA $\valpha\in[0,1] \ \beta\in[0,2]$, no SiLU on keys & 17.03 & 23.43 & 13.67 & 47.19 & 68.13 & 55.09 & 45.57 & 67.00 & 69.80 & 52.02 & 26.11 & 55.79 & 31.80 & 51.85 \\ %
CKDA $\valpha\in[-1,1]\ \beta\in[0,2]$ (\textit{with ablations}) & & & & & & & & & & & & & & \\
\hspace{16pt}{standard} & 17.00 & 23.29 & 14.18 & 46.76 & 68.86 & 54.22 & 44.42 & 66.00 & 70.08 & 52.57 & 27.30 & 56.12 & 31.80 & 51.81 \\ %
\hspace{16pt}{standard, no SiLU on keys} & 17.01 & 23.20 & 14.21 & 46.72 & 70.70 & 56.04 & 44.11 & 65.00 & 70.73 & 52.31 & 25.26 & 56.71 & 31.40 & 51.90 \\ %
\hspace{16pt}{spread} & 17.10 & 23.92 & 13.83 & 46.98 & 67.77 & 54.70 & 44.75 & 68.00 & 70.13 & 52.31 & 26.02 & 53.49 & 31.80 & 51.60 \\ %
\hspace{16pt}{spread, no SiLU on keys} & 17.12 & 23.76 & 14.34 & 46.40 & 67.77 & 52.33 & 44.58 & 71.00 & 70.84 & 53.07 & 27.39 & 54.22 & 31.80 & 51.94 \\ %
\hspace{16pt}{spread, $\beta$ spread init} & 17.16 & 23.80 & 14.33 & 46.70 & 68.86 & 54.62 & 44.23 & 70.00 & 70.95 & 51.85 & 26.88 & 54.99 & 33.00 & 52.21 \\ %
\bottomrule
\end{tabular}%
}
\vspace{-3mm}
\label{tab:lm_res}
\end{table}

\begin{table}[t]
\caption{RULER needle retrieval at 1.3B parameters / 100B FineWeb-Edu tokens, 4K training sequences; accuracy (\%) over 500 samples per cell. Our hybrid rows are evaluated only at or below 4,096 tokens, the context their attention layers trained at. It remains to be investigated why KDA without extended gates is much worse here. }
  \label{tab:fwedu-1p3b-niah}
  \centering
  \resizebox{\linewidth}{!}{%
  \small
  \setlength{\tabcolsep}{3.5pt}
  \begin{tabular}{l cccc cccc ccc ccc}
    \toprule
    Model & \multicolumn{4}{c}{S-NIAH-1} & \multicolumn{4}{c}{S-NIAH-2} & \multicolumn{3}{c}{S-NIAH-3} & \multicolumn{3}{c}{MK-NIAH-1} \\
    \cmidrule(lr){2-5} \cmidrule(lr){6-9} \cmidrule(lr){10-12} \cmidrule(lr){13-15}
     & 1K & 2K & 4K & 8K & 1K & 2K & 4K & 8K & 1K & 2K & 4K & 1K & 2K & 4K \\
    \midrule
    KDA, bounded gate (ours) & 32.2 & 65.8 & 94.0 & 68.8 & 73.6 & 100.0 & 97.6 & 21.0 & 7.4 & 46.6 & 9.8 & 48.6 & 26.2 & 26.0 \\
    CKDA (ours) & 100.0 & 100.0 & 100.0 & 87.8 & 100.0 & 99.8 & 88.6 & 33.2 & 94.4 & 88.2 & 51.0 & 56.0 & 46.6 & 37.0 \\
    \midrule
    \multicolumn{15}{l}{\footnotesize\itshape ours, hybrid - 3:1 full NoPE gated attention} \\
    KDA + attn 3:1 (ours) & 100.0 & 100.0 & 100.0 & 100.0 & 99.8 & 99.2 & 99.2 & 70.2 & 99.4 & 98.0 & 95.8 & 73.8 & 60.2 & 63.8 \\
    CKDA + attn 3:1 (ours) & 99.8 & 96.8 & 81.2 & 53.4 & 93.2 & 98.6 & 83.2 & 62.0 & 97.6 & 91.0 & 86.8 & 95.0 & 94.4 & 94.4 \\
    \midrule
    \multicolumn{15}{l}{\footnotesize\itshape \citet{hatamizadeh2026gdn2}, recurrent} \\
    Mamba-2 & 100.0 & 100.0 & 97.0 & 55.8 & 99.6 & 99.6 & 62.6 & 21.0 & 59.2 & 38.6 & 14.4 & 29.0 & 21.2 & 21.4 \\
    Gated DeltaNet & 99.8 & 100.0 & 100.0 & 97.6 & 100.0 & 100.0 & 87.2 & 32.0 & 89.8 & 54.2 & 60.6 & 58.0 & 37.0 & 27.8 \\
    KDA & 100.0 & 100.0 & 99.2 & 70.6 & 100.0 & 100.0 & 89.0 & 30.6 & 77.4 & 63.2 & 26.2 & 54.0 & 44.2 & 28.0 \\
    Mamba-3 (SISO) & 100.0 & 99.0 & 63.4 & 27.8 & 99.8 & 99.0 & 59.4 & 25.2 & 60.2 & 35.6 & 12.2 & 44.8 & 27.4 & 20.2 \\
    Mamba-3 (MIMO) & 100.0 & 99.8 & 93.0 & 35.6 & 99.8 & 98.8 & 64.2 & 27.2 & 89.2 & 72.4 & 29.2 & 49.4 & 19.2 & 18.0 \\
    Gated DeltaNet-2 & 100.0 & 100.0 & 100.0 & 97.8 & 100.0 & 100.0 & 93.0 & 39.2 & 92.0 & 89.8 & 31.8 & 72.6 & 51.4 & 37.8 \\
    \bottomrule
  \end{tabular}
  }
\end{table}

We use the Language Model Evaluation Harness~\citep{eval-harness}.
The evaluation tasks include WikiText~\citep{merity2016pointer},
LAMBADA (OpenAI version)~\citep{paperno2016lambada},
HellaSwag~\citep{zellers2019hellaswag},
the Winograd Schema Challenge~\citep{levesque2012winograd},
WinoGrande~\citep{sakaguchi2021winogrande},
COPA~\citep{roemmele2011copa}, PIQA~\citep{bisk2020piqa}, 
ARC-Easy, ARC-Challenge~\citep{clark2018think}, OpenBookQA~\citep{mihaylov2018openbookqa}, SWDE~\citep{lockard2019openceres}, FDA~\citep{arora2023language}, and SQUADv2~\citep{rajpurkar2018know}.
Table~\ref{tab:lm_eval_names} expands the benchmark abbreviations.
In Table~\ref{tab:fwedu-1p3b-niah} we also report RULER results. We use the standard methodology and prompts, but we found that slight variations of the prompt (e.g. adding a ":" to "The answer is") lead to largely different numbers and bad S-NIAH numbers can to a large degree be attributed to a lack of instruction following (no answer is given, the prompt is repeated). These should therefore be treated with caution.

\begin{table}[h!]
\caption{Full names of abbreviated evaluation benchmarks.}
\label{tab:lm_eval_names}
\centering
\begin{tabular}{l|l}
\toprule
\textbf{Abbreviation} & \textbf{Full Name} \\
\midrule
Wiki. & WikiText \\
Lamb. & LAMBADA OpenAI \\
Hella. & HellaSwag \\
W.grad & Winograd Schema Challenge \\
W.grande & WinoGrande \\
COPA & Choice of Plausible Alternatives \\
PIQA & Physical Interaction QA \\
ARC-e. & AI2 Reasoning Challenge (Easy) \\
ARC-c. & AI2 Reasoning Challenge (Challenge) \\
BBQA & BigBench QA \\
OBQA & OpenBook QA \\
Avg. & Average across accuracy tasks \\
\bottomrule
\end{tabular}
\end{table}

\section{Scaling Law results}\label{app:scaling-results}
In addition to language modeling performance at a fixed scale, we also test for scaling behavior for varying model size and training data size~\citep{kaplan_scaling_2020, hoffmann_training_2022}. We follow the recipe of~\citet{ajroldi_deriving_2026} on varying the model and data size of training with a linear warmup-stable-decay (WSD) schedule with AdamW~\citep{loshchilov_decoupled_2019}, using model sizes of $47M$, $124M$, $302M$, $588M$, $983M$ and $1.71B$ parameters and data scales of $6B$, $12B$, $20B$, $30B$, $50B$ tokens (leaving out the largest data scales here). To save compute, we rely on their found optimal batch size and learning rates - assuming invariance of these towards a change of sequence mixing backbone or hybrid model architecture. We follow the scaling model of
\citet{videau_skaling_2026,busbridge_distillation_2025} in the parametrization of the loss for a parametric fit.

\subsection{Scaling-law methodology}
\label{sec:scaling-methodology}

\paragraph{The grid.}
Each architecture is trained on a shared $(N, D)$ grid of six model sizes
(47M--1.7B non-embedding parameters, $N$ from $4.73 \times 10^{7}$ to $1.71
\times 10^{9}$) crossed with five token budgets ($D \in \{6, 12, 20, 30,
50\}$\,BT), on the high-quality subset of Nemotron-CC tokenized with the
GPT-NeoX-20B tokenizer at a sequence length of 4096. So that a difference
between arms is attributable to the mixer and not to parameter count, every
non-attention arm is parameter-matched to the attention baseline at its rung by
adjusting $d_{\mathrm{ffn}}$ alone; depth, $d_{\mathrm{model}}$, head count and
head dimension are held fixed across arms. Batch size $b^\star$ and peak
learning rate $\eta^\star$ are taken per cell from the grids of
\citet{ajroldi_deriving_2026}, so that the ladder inherits a tuned, published
schedule rather than one of our own choosing. We depart from that reference in
exactly two respects, both applied uniformly across the grid: $\beta_2 = 0.95$
at every rung, where the reference switches from $0.99$ below 300M and so would
place a hyperparameter change inside our $N$ axis; and no bias terms anywhere
outside the (C)KDA layers.

\paragraph{Endpoints.}
Every cell is trained under a warmup--stable--decay schedule, and the loss we
fit is the validation loss (over 0.838\,B tokens) of the \emph{annealed} endpoint. Within a rung, cells sharing $b^\star$ and $\eta^\star$ are run as a single constant-rate trunk with
a cooldown branched from it at $0.8D$ and annealed linearly over the remaining
$20\%$, rather than as independent runs to save compute. 

\paragraph{Functional forms.}
We fit the non-separable Skaling form of
\citet{videau_skaling_2026,busbridge_distillation_2025}, the one
\citet{ajroldi_deriving_2026} find extrapolates reliably on this corpus:
\begin{align}
  L_{\mathrm{Skaling}}(N, D)    &= E + \bigl(A N^{-\alpha} + B D^{-\beta}\bigr)^{k}.
\end{align}

\paragraph{Three treatments of the irreducible loss.}
Every fit below gives each architecture its own $A$, $\alpha$, $B$, $\beta$,
$k$ and chained-cell offset, and the three differ in how the irreducible loss $E$ is modeled. Table~\ref{tab:scaling-holdout-isolated}
fits each arm in complete isolation, so $E$ is free and separate per arm;
Table~\ref{tab:scaling-holdout-pinned} pins $E$ to the value of
\citet{ajroldi_deriving_2026}, as they have a more extensive, compute heavy grid. Since the irreducible loss should be independent of the model - assuming each can perfectly fit at infinite scales - this is a valid assumption.

\paragraph{Estimation.}
Parameters are estimated by trust-region least squares on the residuals in
nats, under a Huber loss with scale $0.01$\,nats. The scale is set at the
measured run-to-run reproducibility of the ladder, where two bit-identical
configurations differ by up to $0.002$\,nats: cells agreeing to within noise
then enter quadratically, while a single genuine outlier cannot dominate the
fit. Because both forms are strongly multi-modal in $(A, \alpha, B, \beta,
k)$, each fit is run from $600$ randomized starts, log-uniform over the
prefactors and uniform over the exponents, and the lowest-cost optimum is
retained.

\paragraph{Uncertainty.}
Intervals are obtained by bootstrap rather than from a covariance matrix. We resample the endpoints with replacement, stratified within each arm so that
every replicate has the same design as the data and no arm can lose the cells
that identify it, and refit. Each replicate is warm-started from the full-data
optimum and given a small number of additional random restarts, which keeps the
cost tractable while guarding against a resample whose optimum genuinely lies
elsewhere. We report the $2.5$th and $97.5$th percentiles of the replicates as
a subscript on each estimate. These intervals describe sampling variability of
this ladder under the fitted model; they do not carry the systematic
uncertainty in $b^\star$ and $\eta^\star$, which are held fixed at the values from~\citet{ajroldi_deriving_2026}.

\pagebreak
\subsection{Scaling Law Results}
All trained model architectures can be well fitted to the chosen parametric form~\citep{busbridge_distillation_2025,videau_skaling_2026}, see Tables~\ref{tab:scaling-holdout-isolated} and~\ref{tab:scaling-holdout-pinned}. A held-out point on the largest scale can be predicted with low error for each model, see Table~\ref{tab:scaling-holdout}. Re-using the irreducible loss found in~\citet{ajroldi_deriving_2026} (which is smaller than our fitted values) leads to a slight shift of fitted parameters, namely the data scaling exponent and coupling exponent are smaller. This means that larger scale experiments could lead to a non-negligible change still in the fitted parameters. While the validation loss slightly favors KDA over CKDA, the downstream evaluations show a mixed picture here, see Table~\ref{tab:downstream-avg9}. Given the confidence bounds, CKDA and KDA perform on par in vanilla language modeling.
All measured validation losses are shown in Table~\ref{tab:scaling-data}.

The derived compute-optimal scaling for the model architectures is showing in Figure~\ref{fig:scaling_vs_compute}.
The observed decline of the advantage of the recurrent / hybrid models over the Transformer going to large over-training regimes at small scales (see \ref{fig:scaling-advantage}) is not clearly visible in the scaling exponents, though there are differences. 
A detailed exploration of this remains to be investigated in future work. 

\begin{table}[ht]
	\centering
	\small
	\caption{Fitted parameters with 1.7B/50BT held out, all parameters fit per architecture in an isolated way. RMSE is in-sample, over that arm's remaining cells; the held-out error is in Table~\ref{tab:scaling-holdout}. Note that our data is missing the regime of about $10\times$ more compute (100B and 300B tokens) - leading to a slightly different offset compared to~\citet{ajroldi_deriving_2026}.}
	\label{tab:scaling-holdout-isolated}
	\resizebox{\linewidth}{!}{%
		\begin{tabular}{lrrrrrrr}
			\toprule
			Model Arch. & $E$ & $A$ & $\alpha$ & $B$ & $\beta$ & $k$ & RMSE \\
			\midrule
			Transformer++ & $1.1628_{[-0.1500,\,1.3895]}$ & $2.76\times10^{3}_{[6.61\times10^{2},\,1.10\times10^{5}]}$ & $0.401_{[0.326,\,0.453]}$ & $5.23\times10^{3}_{[4.97\times10^{2},\,5.78\times10^{4}]}$ & $0.382_{[0.234,\,0.475]}$ & $0.500_{[0.263,\,0.621]}$ & $0.00697$ \\
			Transformer++ (QK-norm) & $1.0869_{[0.6970,\,1.2665]}$ & $5.57\times10^{3}_{[2.06\times10^{3},\,1.57\times10^{4}]}$ & $0.429_{[0.385,\,0.461]}$ & $7.61\times10^{3}_{[1.47\times10^{3},\,5.81\times10^{4}]}$ & $0.391_{[0.297,\,0.486]}$ & $0.450_{[0.362,\,0.524]}$ & $0.00432$ \\
			KDA & $1.2091_{[1.0517,\,1.3015]}$ & $2.85\times10^{4}_{[2.06\times10^{4},\,4.96\times10^{4}]}$ & $0.528_{[0.515,\,0.545]}$ & $1.35\times10^{4}_{[6.31\times10^{3},\,2.92\times10^{4}]}$ & $0.425_{[0.380,\,0.465]}$ & $0.436_{[0.395,\,0.461]}$ & $0.00206$ \\
			CKDA & $1.2111_{[0.9258,\,1.3227]}$ & $2.57\times10^{4}_{[1.47\times10^{4},\,6.13\times10^{4}]}$ & $0.523_{[0.498,\,0.547]}$ & $1.33\times10^{4}_{[3.31\times10^{3},\,3.56\times10^{4}]}$ & $0.425_{[0.352,\,0.476]}$ & $0.440_{[0.373,\,0.484]}$ & $0.00269$ \\
			KDA + attn 3:1 & $1.1765_{[0.8513,\,1.3069]}$ & $2.49\times10^{4}_{[1.32\times10^{4},\,7.21\times10^{4}]}$ & $0.520_{[0.490,\,0.541]}$ & $1.24\times10^{4}_{[4.07\times10^{3},\,3.34\times10^{4}]}$ & $0.419_{[0.356,\,0.474]}$ & $0.434_{[0.361,\,0.480]}$ & $0.00262$ \\
			CKDA + attn 3:1 & $1.2473_{[1.0476,\,1.3609]}$ & $2.18\times10^{4}_{[1.34\times10^{4},\,5.07\times10^{4}]}$ & $0.521_{[0.502,\,0.544]}$ & $1.12\times10^{4}_{[5.22\times10^{3},\,2.63\times10^{4}]}$ & $0.421_{[0.376,\,0.463]}$ & $0.457_{[0.384,\,0.492]}$ & $0.00212$ \\
            \midrule
    \citet{ajroldi_deriving_2026} & $0.9638_{\pm0.0870}$ & $1.96\times10^{4}_{\pm1.40\times10^{4}}$ & $0.483_{\pm0.030}$ & $4.54\times10^{3}_{\pm2.90\times10^{3}}$ & $0.349_{\pm0.022}$ & $0.393_{\pm0.033}$ & $0.00473$ \\
			\bottomrule
		\end{tabular}%
	}
\end{table}

\begin{table}[ht]
  \centering
  \small
  \caption{$E$ pinned to $0.9638$, the value of \citet{ajroldi_deriving_2026}, fitted on the same corpus and held-out split over a ladder reaching 1.7B parameters and 300B tokens. }
  \label{tab:scaling-holdout-pinned}
  \resizebox{\linewidth}{!}{%
  \begin{tabular}{lrrrrrrr}
    \toprule
    Model Arch. & E$^*$ & $A$ & $\alpha$ & $B$ & $\beta$ & $k$ & RMSE \\
    \midrule
    Transformer++ & *& $3.99\times10^{3}_{[7.51\times10^{2},\,1.43\times10^{4}]}$ & $0.403_{[0.325,\,0.458]}$ & $3.17\times10^{3}_{[5.47\times10^{2},\,2.13\times10^{4}]}$ & $0.342_{[0.272,\,0.423]}$ & $0.454_{[0.384,\,0.562]}$ & $0.00697$ \\
    Transformer++ (QK-norm) & *& $7.28\times10^{3}_{[2.95\times10^{3},\,1.27\times10^{4}]}$ & $0.431_{[0.390,\,0.457]}$ & $6.38\times10^{3}_{[1.67\times10^{3},\,2.43\times10^{4}]}$ & $0.372_{[0.317,\,0.429]}$ & $0.423_{[0.384,\,0.476]}$ & $0.00439$ \\
    KDA & *& $5.07\times10^{4}_{[3.24\times10^{4},\,7.53\times10^{4}]}$ & $0.533_{[0.512,\,0.552]}$ & $7.46\times10^{3}_{[3.42\times10^{3},\,1.44\times10^{4}]}$ & $0.373_{[0.341,\,0.401]}$ & $0.383_{[0.364,\,0.406]}$ & $0.00244$ \\
    CKDA & *& $4.53\times10^{4}_{[2.35\times10^{4},\,7.21\times10^{4}]}$ & $0.527_{[0.499,\,0.547]}$ & $6.77\times10^{3}_{[2.56\times10^{3},\,1.56\times10^{4}]}$ & $0.369_{[0.326,\,0.404]}$ & $0.387_{[0.362,\,0.417]}$ & $0.00302$ \\
    KDA + attn 3:1 & * & $4.06\times10^{4}_{[2.16\times10^{4},\,5.91\times10^{4}]}$ & $0.523_{[0.497,\,0.540]}$ & $7.33\times10^{3}_{[3.03\times10^{3},\,1.86\times10^{4}]}$ & $0.374_{[0.337,\,0.413]}$ & $0.388_{[0.366,\,0.416]}$ & $0.00286$ \\
    CKDA + attn 3:1 & * & $4.10\times10^{4}_{[2.64\times10^{4},\,6.65\times10^{4}]}$ & $0.526_{[0.507,\,0.548]}$ & $5.11\times10^{3}_{[2.31\times10^{3},\,1.22\times10^{4}]}$ & $0.358_{[0.325,\,0.393]}$ & $0.394_{[0.367,\,0.419]}$ & $0.00263$ \\
    \midrule
    \citet{ajroldi_deriving_2026} & $0.9638_{\pm0.0870}^*$ & $1.96\times10^{4}_{\pm1.40\times10^{4}}$ & $0.483_{\pm0.030}$ & $4.54\times10^{3}_{\pm2.90\times10^{3}}$ & $0.349_{\pm0.022}$ & $0.393_{\pm0.033}$ & $0.00473$ \\
    \bottomrule
  \end{tabular}%
  }
\end{table}

\begin{table}[ht]
  \centering
  \small
  \caption{Held-out error at 1.7B/50BT against the parametric loss fit, at the largest model/data regime we are training here. }
  \label{tab:scaling-holdout}
  \resizebox{0.8\linewidth}{!}{%
  \begin{tabular}{lrrrrrr}
    \toprule
    Arm & Measured & $E$ per model  & $E$ pinned to \citet{ajroldi_deriving_2026}  \\
    \midrule
    Transformer++ & $2.1630$ & $+0.0007_{[-0.0081,\,+0.0089]}$ &  $-0.0027_{[-0.0059,\,+0.0027]}$ \\
    Transformer++ (QK-norm) & $2.1374$ & $+0.0006_{[-0.0053,\,+0.0053]}$ &  $-0.0018_{[-0.0048,\,+0.0024]}$  \\
    KDA & $2.1046$ & $-0.0007_{[-0.0047,\,+0.0022]}$ & $-0.0046_{[-0.0070,\,-0.0016]}$ \\
    CKDA & $2.1061$ & $+0.0007_{[-0.0046,\,+0.0065]}$ &  $-0.0029_{[-0.0059,\,+0.0032]}$  \\
    KDA + attn 3:1 & $2.0971$ & $-0.0031_{[-0.0078,\,+0.0006]}$ & $-0.0063_{[-0.0088,\,-0.0033]}$  \\
    CKDA + attn 3:1 & $2.1002$ & $-0.0016_{[-0.0058,\,+0.0018]}$  & $-0.0059_{[-0.0086,\,-0.0027]}$  \\
    \bottomrule
  \end{tabular}%
  }
\end{table}

\begin{table}[t]
  \centering
  \small
  \caption{Validation loss at every scaling experiment point. Best per row in bold. The final column is the best cell of the hyperparameter sweep of \citet{ajroldi_deriving_2026} at the same $(N, D)$, a dash indicates that no result was reported”—the cells contain dashes.}
  \label{tab:scaling-data}
  \resizebox{\linewidth}{!}{%
  \begin{tabular}{lrrrrrrr}
    \toprule
    Cell & Transformer++ & Transformer++ (QK-norm) & KDA & CKDA & KDA + attn 3:1 & CKDA + attn 3:1 & \citet{ajroldi_deriving_2026} \\
    \midrule
    47M/6BT & 2.9844 & 2.9463 & 2.9247 & 2.9267 & \textbf{2.9126} & 2.9167 & 2.9272 \\
    47M/12BT & 2.9109 & 2.8833 & 2.8664 & 2.8671 & \textbf{2.8522} & 2.8560 & 2.8747 \\
    47M/20BT & 2.8782 & 2.8504 & 2.8403 & 2.8397 & \textbf{2.8250} & 2.8257 & 2.8417 \\
    47M/30BT & 2.8578 & 2.8335 & 2.8215 & 2.8220 & \textbf{2.8063} & 2.8075 & 2.8198 \\
    47M/50BT & 2.8271 & 2.8045 & 2.7974 & 2.7962 & \textbf{2.7809} & 2.7814 & 2.7974 \\
    \midrule
    124M/6BT & 2.7471 & 2.7188 & 2.6869 & 2.6880 & \textbf{2.6718} & 2.6719 & 2.7129 \\
    124M/12BT & 2.6799 & 2.6541 & 2.6212 & 2.6220 & \textbf{2.6072} & 2.6083 & 2.6497 \\
    124M/20BT & 2.6389 & 2.6124 & 2.5837 & 2.5850 & \textbf{2.5705} & 2.5708 & 2.6048 \\
    124M/30BT & 2.6074 & 2.5851 & 2.5576 & 2.5581 & \textbf{2.5442} & 2.5462 & 2.5736 \\
    124M/50BT & 2.5682 & 2.5488 & 2.5286 & 2.5288 & \textbf{2.5146} & 2.5154 & 2.5409 \\
    \midrule
    302M/6BT & 2.6126 & 2.5789 & 2.5243 & 2.5258 & \textbf{2.5146} & 2.5174 & 2.5723 \\
    302M/12BT & 2.5382 & 2.5044 & 2.4505 & 2.4550 & 2.4446 & \textbf{2.4437} & 2.4939 \\
    302M/20BT & 2.4942 & 2.4560 & 2.4077 & 2.4089 & \textbf{2.3956} & 2.3967 & 2.4480 \\
    302M/30BT & 2.4598 & 2.4287 & 2.3768 & 2.3784 & \textbf{2.3651} & 2.3667 & 2.4158 \\
    302M/50BT & 2.4160 & 2.3882 & 2.3441 & 2.3453 & \textbf{2.3323} & 2.3335 & 2.3778 \\
    \midrule
    588M/6BT & 2.5209 & 2.4892 & 2.4496 & 2.4532 & 2.4430 & \textbf{2.4424} & 2.4798 \\
    588M/12BT & 2.4346 & 2.4051 & 2.3626 & 2.3673 & 2.3556 & \textbf{2.3555} & 2.3988 \\
    588M/20BT & 2.3805 & 2.3534 & 2.3143 & 2.3191 & 2.3051 & \textbf{2.3049} & 2.3476 \\
    588M/30BT & 2.3436 & 2.3172 & 2.2775 & 2.2815 & \textbf{2.2676} & 2.2683 & 2.3112 \\
    588M/50BT & 2.3032 & 2.2770 & 2.2386 & 2.2417 & \textbf{2.2284} & 2.2295 & 2.2696 \\
    \midrule
    983M/6BT & 2.4551 & 2.4282 & 2.3905 & 2.3963 & \textbf{2.3810} & 2.3850 & 2.4250 \\
    983M/12BT & 2.3682 & 2.3403 & 2.3009 & 2.3059 & \textbf{2.2920} & 2.2959 & 2.3363 \\
    983M/20BT & 2.3114 & 2.2844 & 2.2484 & 2.2532 & \textbf{2.2369} & 2.2405 & 2.2815 \\
    983M/30BT & 2.2721 & 2.2464 & 2.2098 & 2.2141 & \textbf{2.1981} & 2.2017 & 2.2430 \\
    983M/50BT & 2.2298 & 2.2070 & 2.1686 & 2.1722 & \textbf{2.1571} & 2.1608 & 2.2003 \\
    \midrule
    1.7B/6BT & 2.4001 & 2.3725 & 2.3359 & 2.3402 & \textbf{2.3284} & 2.3299 & -- \\
    1.7B/12BT & 2.3100 & 2.2829 & 2.2496 & 2.2507 & \textbf{2.2399} & 2.2433 & -- \\
    1.7B/20BT & 2.2488 & 2.2218 & 2.1882 & 2.1898 & \textbf{2.1795} & 2.1828 & -- \\
    1.7B/30BT & 2.2077 & 2.1812 & 2.1477 & 2.1494 & \textbf{2.1395} & 2.1429 & 2.1802 \\
    1.7B/50BT & 2.1630 & 2.1374 & 2.1046 & 2.1061 & \textbf{2.0971} & 2.1002 & 2.1353 \\
    \bottomrule
  \end{tabular}%
  }
\end{table}

\begin{table}[t]
  \centering
  \small
  \caption{Downstream accuracy, averaged over the nine accuracy tasks of~\Cref{tab:fwedu-1p3b-gdn2} at every scaling experiment cell. Below 302M most of the suite is at chance -- HellaSwag near 27, ARC-c near 23, OpenBookQA near 25 -- so the average at the lower rungs is largely majority-class behavior on BoolQ and PIQA. }
  \label{tab:downstream-avg9}
  \resizebox{\linewidth}{!}{%
  \begin{tabular}{lrrrrrr}
    \toprule
    Cell & Transformer++ & Transformer++ (QK-norm) & KDA & CKDA & KDA + attn 3:1 & CKDA + attn 3:1 \\
    \midrule
    47M/6BT & 37.97 & 37.75 & 39.08 & 38.49 & 37.70 & \textbf{39.12} \\
    47M/12BT & 38.14 & 37.62 & 38.04 & 38.43 & 38.11 & \textbf{38.60} \\
    47M/20BT & 38.59 & 38.22 & 39.06 & 37.84 & 37.72 & \textbf{39.42} \\
    47M/30BT & 39.13 & 37.49 & \textbf{40.16} & 39.36 & 39.22 & 39.46 \\
    47M/50BT & 38.64 & 38.94 & \textbf{39.42} & 38.96 & 38.96 & 38.49 \\
    \midrule
    124M/6BT & 40.47 & 40.06 & 41.22 & \textbf{42.25} & 41.06 & 40.69 \\
    124M/12BT & 40.57 & 41.74 & \textbf{42.41} & 41.29 & 41.96 & 41.68 \\
    124M/20BT & 42.41 & 42.16 & 43.41 & \textbf{43.67} & 43.11 & 42.88 \\
    124M/30BT & 42.01 & 41.52 & 43.53 & 43.64 & \textbf{43.68} & 43.48 \\
    124M/50BT & 42.98 & 44.01 & 44.02 & \textbf{44.23} & 43.59 & 43.51 \\
    \midrule
    302M/6BT & 42.17 & 42.39 & \textbf{44.76} & 44.56 & 44.50 & 44.22 \\
    302M/12BT & 44.10 & 44.15 & 45.89 & 45.90 & \textbf{46.77} & 45.53 \\
    302M/20BT & 44.50 & 45.24 & 46.12 & \textbf{47.19} & 47.01 & 47.02 \\
    302M/30BT & 45.31 & 45.19 & 47.73 & 47.99 & \textbf{48.20} & 47.97 \\
    302M/50BT & 46.21 & 46.21 & 47.69 & 48.25 & 48.82 & \textbf{49.14} \\
    \midrule
    588M/6BT & 44.01 & 44.15 & 45.75 & 45.35 & \textbf{45.96} & 45.46 \\
    588M/12BT & 45.84 & 46.67 & 47.82 & 47.67 & \textbf{48.53} & 47.87 \\
    588M/20BT & 47.16 & 46.79 & 48.57 & 48.66 & 48.75 & \textbf{49.54} \\
    588M/30BT & 47.53 & 47.98 & 50.18 & 49.45 & 49.50 & \textbf{50.79} \\
    588M/50BT & 48.45 & 49.83 & \textbf{51.43} & 50.44 & 51.24 & 51.30 \\
    \midrule
    983M/6BT & 45.22 & 46.19 & 46.80 & 47.14 & \textbf{47.78} & 46.86 \\
    983M/12BT & 47.99 & 48.36 & 49.29 & 49.48 & \textbf{50.40} & 49.33 \\
    983M/20BT & 48.90 & 49.12 & 50.51 & 51.22 & \textbf{51.49} & 50.95 \\
    983M/30BT & 50.16 & 50.42 & 51.51 & 51.95 & \textbf{52.60} & 52.32 \\
    983M/50BT & 51.80 & 51.57 & 53.13 & 54.10 & 54.10 & \textbf{54.25} \\
    \midrule
    1.7B/6BT & 46.66 & 47.61 & 48.25 & \textbf{48.84} & 48.35 & 48.29 \\
    1.7B/12BT & 49.02 & 49.59 & 51.10 & \textbf{51.34} & 51.18 & 50.95 \\
    1.7B/20BT & 50.13 & 51.75 & 52.80 & \textbf{53.38} & 52.60 & 52.53 \\
    1.7B/30BT & 51.56 & 53.29 & \textbf{55.13} & 54.59 & 54.91 & 54.35 \\
    1.7B/50BT & 53.08 & 55.10 & 56.87 & \textbf{57.21} & 56.47 & 55.82 \\
    \bottomrule
  \end{tabular}%
  }
\end{table}

\clearpage
\section{Additional Experimental Results}\label{appsec:exp_res}

\begin{figure}
    \centering
    \includegraphics[width=1.0\linewidth]{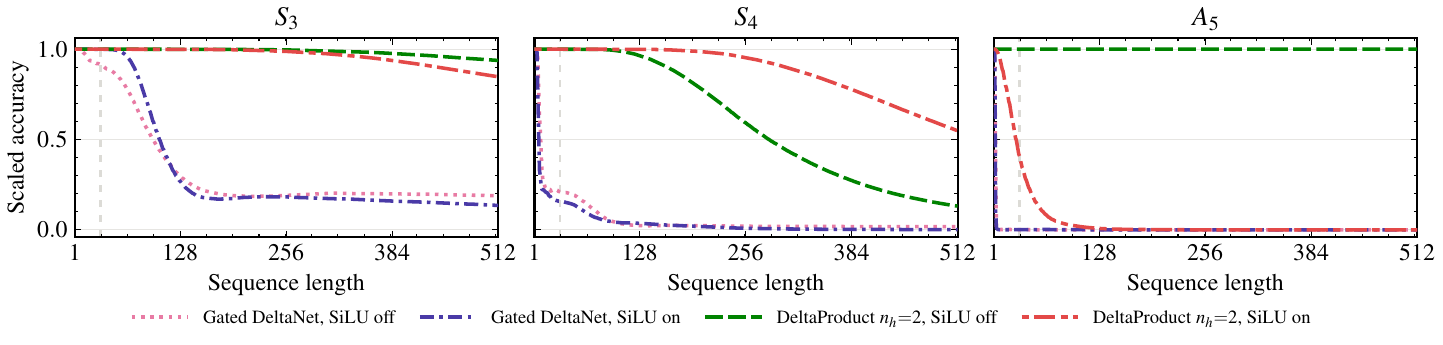}
    \vspace{-6mm}
    \caption{Baseline results for~\Cref{fig:group_word_problems}}
    \label{fig:group_word_problems_baselines}
\end{figure}

\begin{figure}
    \centering
    \includegraphics[width=0.8\linewidth]{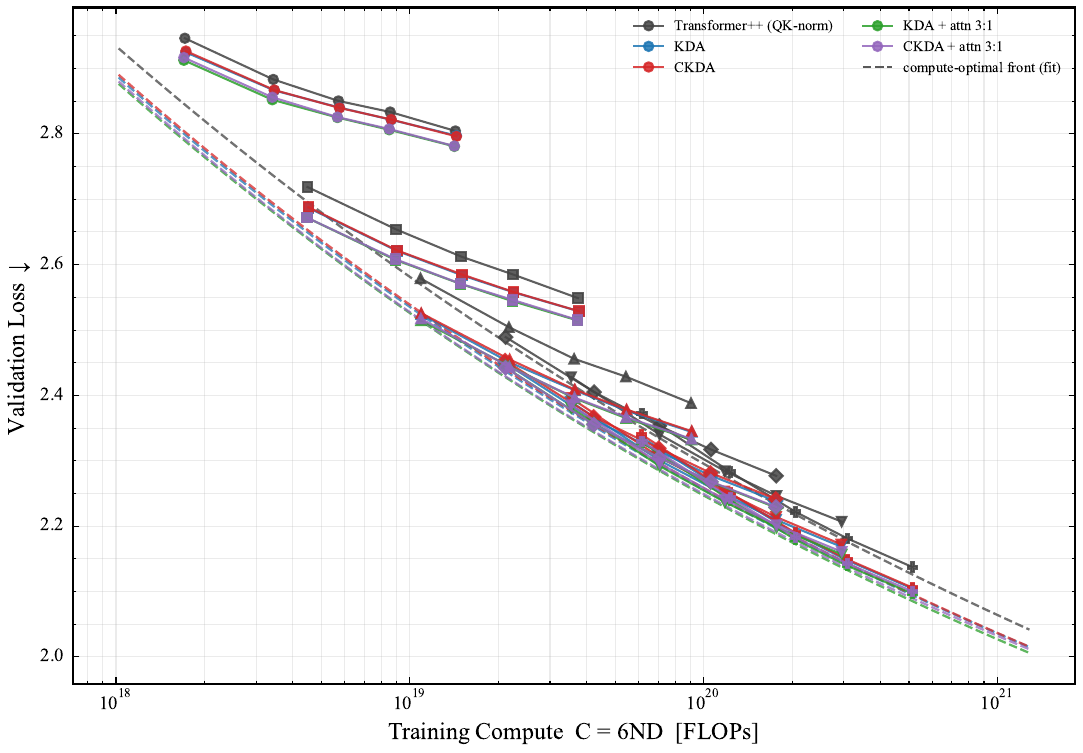}
    \vspace{-4mm}
    \caption{Validation loss vs. Compute for the language model scaling experiments on Nemotron-CC.}
    \label{fig:scaling_vs_compute}
\end{figure}

\begin{figure}
    \centering
    \includegraphics[width=0.45\linewidth]{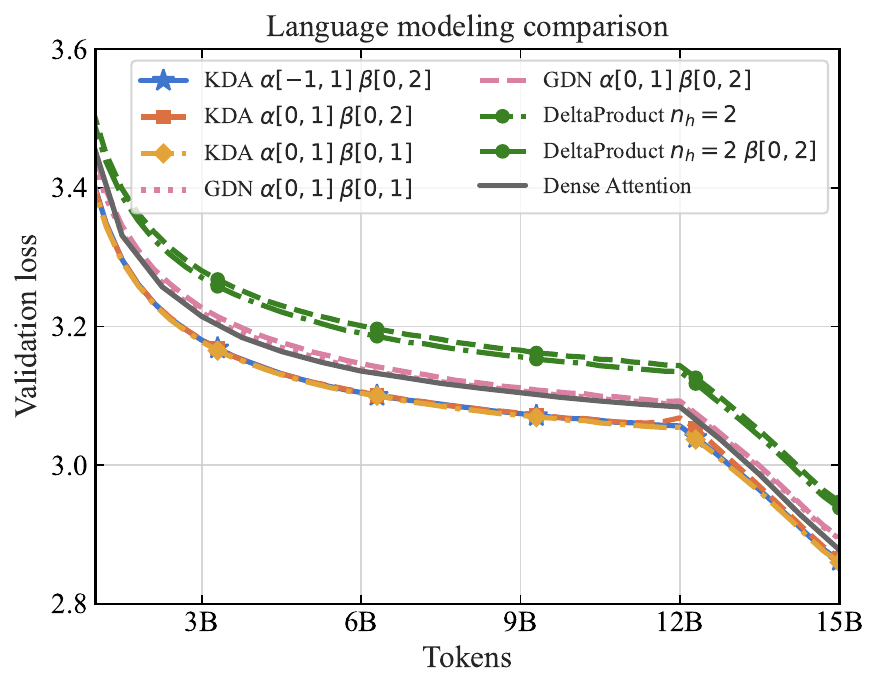}
    \includegraphics[width=0.45\linewidth]{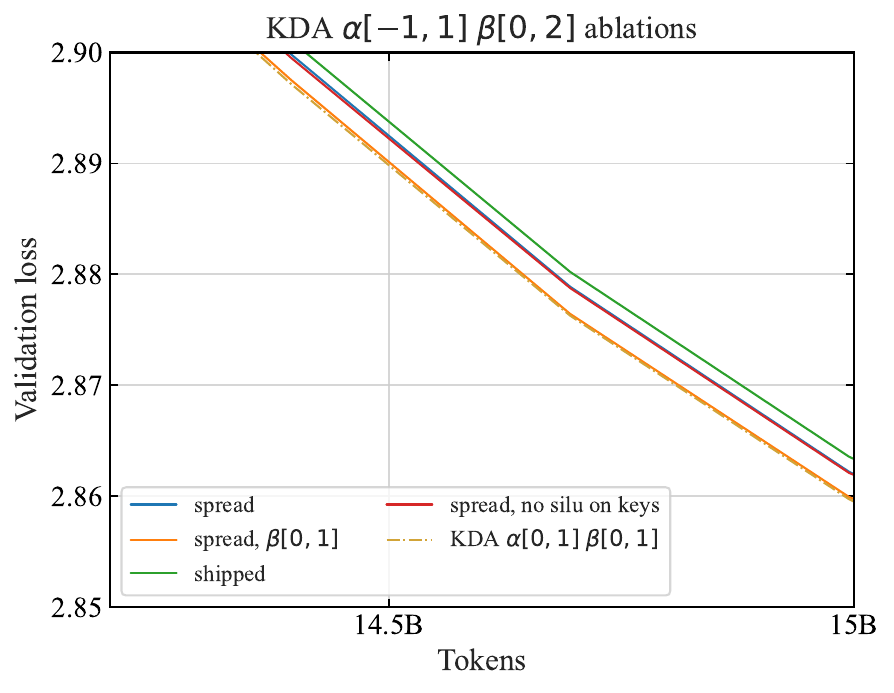}
    \caption{Language modeling results: validation loss versus tokens. (Left) Comparison across baselines. (Right) Comparison across ablations shown for the end of the cooldown phase, as loss values are quite close together.}
    \label{fig:val_curves}
\end{figure}

\begin{figure}
    \centering
    \includegraphics[width=0.8\linewidth]{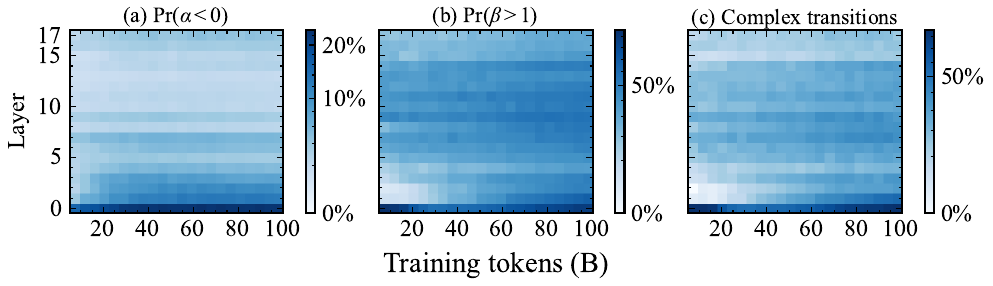}
    \caption{Hybrid CKDA 1.3B, see~\Cref{fig:signed_kda_interpretability_llm}.}
    \label{fig:signed_kda_interpretability_llm_hybrid}
\end{figure}
\begin{figure}
    \centering
    \includegraphics[width=0.8\linewidth]{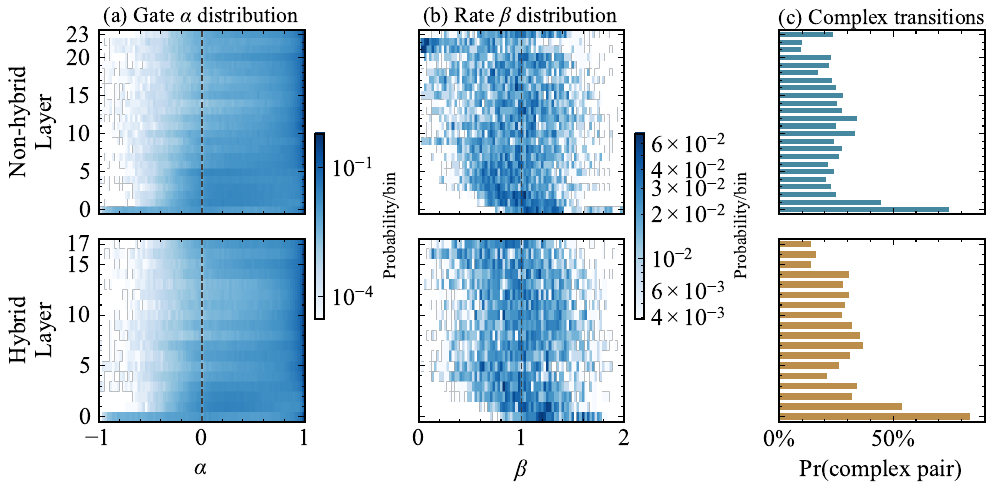}
    \caption{Final layer profiles of gate and $\beta$ for the CKDA 1.3B non-hybrid and hybrid models after training on 100B tokens.}
    \label{fig:final_layer_stats_llm}
\end{figure}

\end{document}